\documentclass{article}

\PassOptionsToPackage{numbers, compress}{natbib}

\usepackage[main, final]{neurips_2026}

\usepackage[utf8]{inputenc} \usepackage[T1]{fontenc}    \usepackage{hyperref}       \usepackage{url}            \usepackage{booktabs}       \usepackage{amsfonts}       \usepackage{nicefrac}       \usepackage{microtype}      \usepackage{xcolor}         \usepackage{float}
\usepackage{microtype}
\usepackage{graphicx}
\usepackage{subcaption}
\usepackage[dvipsnames,table]{xcolor}
\usepackage{researchpack}
\usepackage{multicol}
\usepackage{multirow}
\usepackage[ruled,vlined]{algorithm2e}

\usepackage{siunitx}
\usepackage{hyperref}
\usepackage{tikz}
\newcommand{\LN}{\textsc{LN}}

\newcommand{\VQ}{\textsc{VQ}}
\newcommand{\TS}{\textsc{TS}}
\newcommand{\DC}{\textsc{DC}}
\newcommand{\IR}{\textsc{IR}}
\newcommand{\PS}{\textsc{PS}}
\newcommand{\SM}{\textsc{SM}}
\newcommand{\PC}{\textsc{PC}}
\newcommand{\KC}{\textsc{KC}}
\newcommand{\NC}{\textsc{NC}}
\newcommand{\VQDC}{\textsc{VQ-DC}}
\newcommand{\VQNC}{\textsc{VQ-NC}}
\newcommand{\tissue}{\texttt{tissue}}
\newcommand{\tissueM}{\texttt{TissueMNIST}}
\newcommand{\cifarTen}{\texttt{Cifar10}}
\newcommand{\cifar}{\texttt{Cifar100}}

 \newcommand{\ice}{\mathrm{frz}}

\usepackage[capitalize,noabbrev]{cleveref}

\usepackage[textsize=tiny]{todonotes}

\title{Divide et Calibra: \\Multiclass Local Calibration via Vector Quantization}

\author{
    Cesare Barbera\thanks{Corresponding author: \texttt{cesare.barbera@phd.unipi.it}; \texttt{cesare.barbera@unitn.it}.} \\
    DI, University of Pisa \\ 
    DISI, University of Trento    
    \And
    Lorenzo Perini \\
    Meta Platforms, Inc.
    \And
    Giovanni De Toni\thanks{Work done while at Fondazione Bruno Kessler.} \\
    ETH AI Center \& ETH Zurich
    \And
    Andrea Passerini \\
    DISI, University of Trento
    \And
    Andrea Pugnana \\
    DISI, University of Trento
}

\begin{document}

\maketitle
\begin{abstract}
Accurate and well-calibrated Machine Learning (ML) models are mandatory in high-stakes settings, yet effective multiclass calibration remains challenging: global approaches assume calibration errors are homogeneous across the latent space, while local methods often rely on representation-space dimensionality reduction, discarding latent directions relevant to miscalibration.
To address these issues, we propose a compositional approach to multiclass calibration, where region-specific calibration maps are constructed from shared factors. We instantiate this idea via Vector Quantization (VQ), which induces a structured partition of the representation space, and an indexed parameterization of Dirichlet concentrations that enables parameter sharing across regions.
Our approach learns heterogeneous calibration maps that generalize well even to sparse regions of the latent space. Experiments on benchmark datasets show significant improvements in local calibration while maintaining competitive global calibration and predictive performance.
\end{abstract}

\section{Introduction}

In high-stakes settings such as healthcare and finance, Machine Learning (ML) models must be not only accurate but also well-calibrated~\citep{DBLP:journals/corr/abs-2308-01222, DBLP:journals/cmpb/SambyalNKB23}. In these settings, the reliability of the \textit{entire predicted probability vector} is critical: clinicians, judges, or automated policies act on relative likelihoods between multiple competing outcomes, rather than on the most likely class alone.

Consider an ML model that predicts cancer stage from cell images. Its probabilities can become systematically distorted for cells exhibiting unusual combinations of shape, texture and staining patterns because few similar examples were observed during training. This makes the model unreliable exactly in those cases where calibration is highly sought for trustworthy decision-making~\citep{DBLP:conf/aaai/RuggieriP25}.

Such systematic distortions arising in low-density regions of the representation space are known as \textit{proximity bias}~\citep{DBLP:conf/nips/XiongDKW0XH23}.
To address proximity bias, recent work has moved towards \emph{local calibration}, which requires probability estimates to be reliable not only on average, but within specific regions of the input or latent space~\citep{DBLP:conf/uai/LuoBBZWXSESP22}. However, local multiclass calibration has a critical statistical bottleneck: \textit{data sparsity}. As the dimensionality of the representation space increases, data points become more isolated, making it harder to reliably estimate local corrections. Conversely, global calibration methods~\citep{DBLP:conf/nips/KullPKFSF19,DBLP:journals/corr/abs-2511-03685} avoid data sparsity by applying a single correction map to all samples, but fail to correct \textit{proximity bias} in low-density regions. This describes a key trade-off: global methods are stable but biased, while local methods are flexible but high-variance. 

To address this challenge,  we reinterpret local calibration as a compositional learning problem, where region-specific calibration functions are constructed from reusable latent components, rather than estimated independently for each region. \Cref{fig:pipeline} showcases our proposed methodology: first, we consider a discretization of the representation space via vector quantization (VQ). VQ induces a Voronoi tessellation of the embedding space, mapping each continuous latent vector to a discrete sequence of elements in a shared \textit{codebook} (\ie a finite set of prototype vectors, or ``\textit{codewords}''). This replaces a continuous representation with a structured composition of discrete components, drawn from frequently reused codewords. Intuitively, even if a particular index sequence corresponds to a rare region, its codewords are shared across many samples (assuming good codebook utilization), providing better statistical support than direct conditioning on rare continuous features.
Second, we apply this same discretization principle also to calibration parameters, \ie we learn calibration maps for rare or isolated instances through combinations of well-estimated, frequently reused calibration factors.
This yields an implicit form of density regularization: the model is locally adaptive, yet avoids learning fully instance-specific calibration maps in sparse regions.

\textbf{Our Contributions.} To sum up, our contributions are:

\begin{itemize}
    \item[$(i)$] We propose a \emph{compositional formulation of local multiclass calibration}, where calibration maps are defined over a Voronoi tessellation of the embedding space induced by Vector Quantization, enabling region-aware modeling of miscalibration (\cref{sec:quantization}).

    \item[$(ii)$] To make this formulation tractable, we introduce an \textit{indexed parameterization trick} that constructs region-specific calibration maps by reusing shared codeword-dependent factors. This yields a parameter-efficient model that scales to exponentially many regions, enables compositional generalization, and whose estimation error is governed by codeword frequency rather than by the number of samples in each region~(\cref{sec:local_dir_cal,sec:stability}).

    \item [$(iii)$] We empirically validate the approach showing large and significant improvements over existing baselines, particularly in low-support regions (\cref{sec:experimental-evaluation}).
\end{itemize}

\begin{figure*}[t]
\begin{subfigure}[t]{.680\textwidth}
        \centering
        \includegraphics[width=\linewidth]{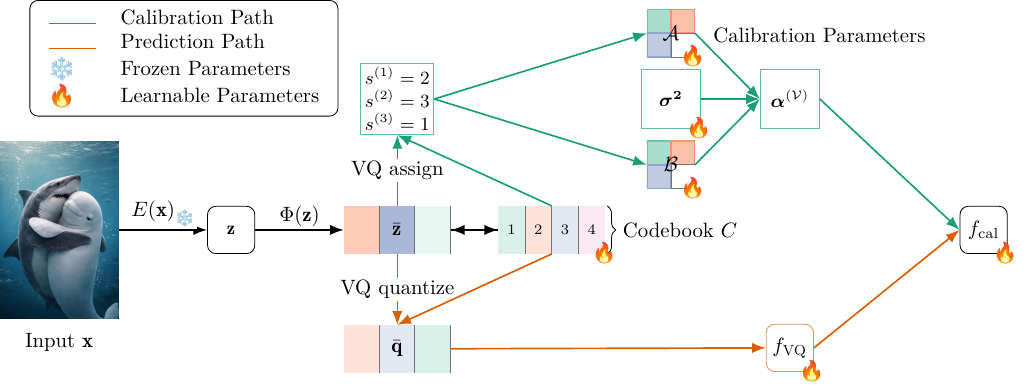}
    \caption{Toy example pipeline when $w=3$ and $|\mathcal{C}|=4$.}
    \label{fig:pipeline}
\end{subfigure}\hfill
\begin{subfigure}[t]{.3\textwidth}
        \centering
        \includegraphics[width=\linewidth]{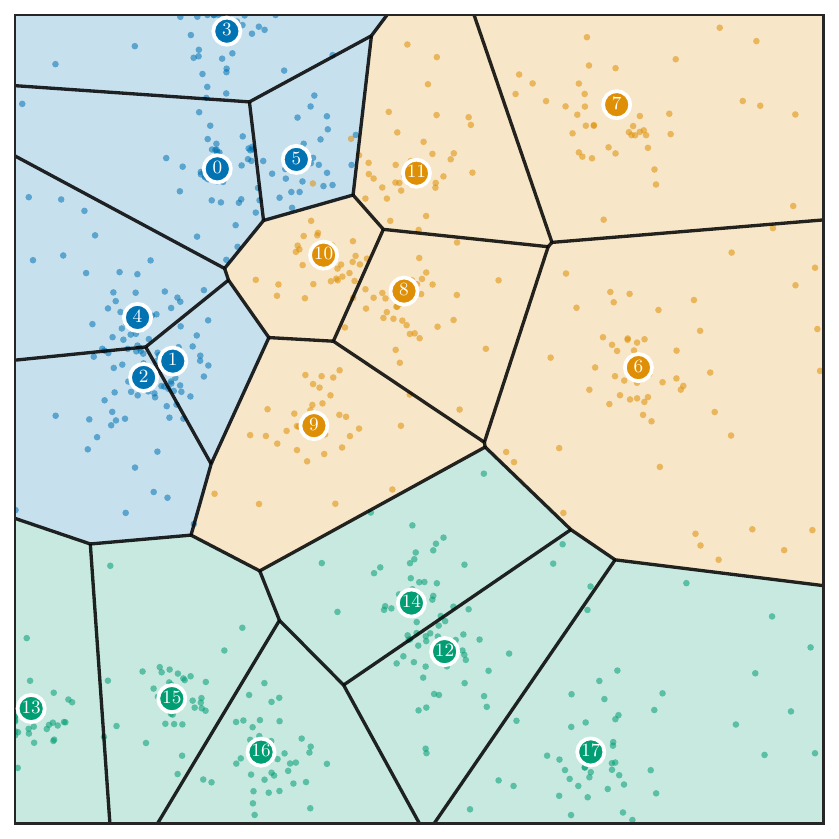}
\caption{Voronoi Tessellation Example.} \label{fig:wvoronoi}
\end{subfigure}
\caption{We depict our pipeline (\cref{fig:pipeline}) and Voronoi assignment procedure (\cref{fig:wvoronoi}).}

\end{figure*}

\section{Background}

\textbf{Preliminaries.}
We consider a multi-class classification task, where $\mcX\subseteq\mathbb{R}^m$ is the feature space and $\mcY = \{0,\ldots, |\mathcal{Y}|-1\}$ is a finite target space with $|\mathcal{Y}|$ distinct labels. Let us assume we have access to a given dataset \( \mathcal D = \{(\mbx_i, y_i)\}_{i=1}^n \) of input-output pairs drawn from an unknown joint distribution \( \mathcal{P} \) over \( \mathcal{X} \times \mathcal{Y} \). Each input \( \mbx_i \in \mcX \) is a feature vector of $m$ dimensions, and each label \( y_i \in \mcY \) has a corresponding one-hot encoded vector $\mathbf{y}_i$ indicating the correct class among the \( |\mathcal{Y}| \) possible classes.
We consider a probabilistic classifier \( f \colon \mathcal{X} \rightarrow \Delta^{|\mathcal{Y}|} \) mapping inputs to the \((|\mathcal{Y}|-1)\)-dimensional probability simplex, with \( f(\mbx) = \mathbf{\hat{p}} \) and \( \mathbf{\hat{p}}_{k} = f_k(\mbx) \) the predicted probability of class \( k \).

\textbf{Calibration notions and metrics.}
\label{sec:notions_metrics}
The strictest notion is \textit{strong calibration}~\citep{DBLP:conf/aistats/VaicenaviciusWA19}, which requires the target class conditional distribution on any prediction of the classifier to match that prediction:
\[
\mathbb{P}(\mathbf{y}_k =1 \mid \mathbf{\hat{p}}) =  \mathbf{\hat{p}}_k \quad \quad \forall \; k \in \{0, \ldots, |\mcY|-1\}
\]
While strong calibration is defined as a population-level expectation, recent work has shifted attention to \textit{local calibration} \citep{DBLP:conf/uai/LuoBBZWXSESP22, valk2023assuming, DBLP:journals/corr/Barbera25}, which requires calibration to hold within localized regions of the input or representation space. A model is then \textit{locally calibrated} if its predictions match consistent kernel estimates of the conditional label distribution.

We consider both \textit{global} and \textit{local} calibration metrics to assess calibration.
Expected Calibration Error (ECE)~\citep{DBLP:conf/aaai/NaeiniCH15} measures the discrepancy between predicted confidence and empirical accuracy across bins; in multiclass settings it extends to Class-wise ECE \citep{DBLP:conf/nips/KullPKFSF19} (per-class average) and Multidimensional ECE (MECE), the latter impractical due to combinatorial simplex-binning. Expected Cumulative Calibration Error (ECCE)~\citep{DBLP:journals/jmlr/IbarraGTTX22} aggregates errors cumulatively, providing a more robust assessment. We summarize these errors by averaging the absolute cumulative residuals across bins.

For local calibration metrics, we consider the multiclass extension of
\textit{Local Calibration Error} (LCE) \cite{DBLP:conf/uai/LuoBBZWXSESP22} proposed by~\citet{DBLP:journals/corr/Barbera25}. 
This metric captures discrepancies between predicted probabilities and ground-truth
labels for neighboring points of an anchor $\mbx_i$, weighted by a kernel function.
We also consider the \textit{Maximum Local Calibration Error} (MLCE), which measures the worst-case local calibration error over the dataset.
See \cref{sec:metrics} for a detailed description of all metrics.

\section{Methodology}
\label{sec:methodology}
Given a pre-trained classifier $f$ and a calibration set $\mathcal D_{\mathrm{cal}}$, our goal is to construct a region-aware calibration function $f_{\mathrm{cal}}$ that maps the model output $\hat{\mbp}$ to a calibrated posterior $\hat{\mbp}_{\mathrm{cal}}$.
Unlike global methods that apply a unique map to the model scores for all instances, learning \textit{local} calibration functions requires access to a (frozen) representation space and replacing the original classification head, as in \citet{kirichenko2022last} or other local calibration approaches \citep{DBLP:conf/iclr/LinT023, DBLP:journals/corr/Barbera25}.

Note that obtaining a heterogeneous calibrator in high-dimensional settings is challenging for three main reasons. First, defining \textit{``local regions''} in high-dimensional embedding spaces is non-trivial. Standard distance-based methods suffer from the curse of dimensionality \citep{beyer1999nearest, wasserman2006all, hastie2009elements, polianskii2022voronoi}, leading to isolated neighborhoods and unreliable density estimation. Second, even if we define discrete regions, the number of such regions typically grows \textit{exponentially} with the dimensionality. Learning a distinct, unstructured calibration map for every region is computationally intractable and memory-heavy. Third, models that act locally are often prone to overfitting. Thus, we need to ensure that a local calibration map does not come at the cost of statistical instability, especially for rare instances.

In the following, we address these challenges. In \cref{sec:quantization} (\textit{Divide}), we solve the first problem by partitioning the space using vector quantization, mapping each instance to a structured region defined by shared codewords. In \cref{sec:local_dir_cal} (\textit{Calibra}), we first introduce a compositional Dirichlet calibration model based on shared codeword-dependent factors that assigns region-specific calibration maps while sharing parameters across regions.
Then, we show how parameter sharing controls the estimation error of the calibration maps, even for rare or isolated instances.

\subsection{\textit{Divide:} Voronoi Tessellation of the Latent Space} \label{sec:quantization}
The encoder representation space induced by modern neural networks is typically \textit{high-dimensional} and \textit{structured}, yet calibration behavior varies substantially across different regions of this space\citep{DBLP:conf/nips/XiongDKW0XH23}.
To reason about such locality in a principled way, we adopt a geometric perspective: partitioning the space into a collection of discrete regions, where the model's predictions are approximately similar.
We construct this partition by independently quantizing contiguous sub-vectors of the latent representation using a shared codebook.

Let \(\mathbf x\) be mapped by the encoder (\textit{e.g.,} the penultimate layer of a Neural Network) to a feature vector \(\mbz = E(\mbx) \in \mathbb{R}^{m'}\).
Let us define a segmentation map as $\Phi: \mathbb{R}^{m'} \rightarrow (\mathbb{R}^d)^w$, where $d,w \in \mathbb{N}$. Such a map partitions \(\mbz\) into \(w\) contiguous segments (\textit{``slots''}) of size $d$, so that we have $m'=d\cdot w$ and:
\[
\mathbf{\overline{z}} = \bigl[\mbz^{(1)}, \mbz^{(2)}, \dots, \mbz^{(w)}\bigr],
\qquad \mbz^{(i)} \in \mathbb{R}^{d}
\]

We define a \textit{codebook} as  \(\mathcal C = \{ \mbc_1, \dots, \mbc_{|\mathcal C|} \} \subset \mathbb{R}^{d}\) where $\mbc_i \in \mathbb{R}^d$ is a \textit{codeword}. A codebook is a finite set of representative vectors (codewords) in $\mathbb{R}^{d}$ that serve as building blocks for the quantization procedure.
Given the segmented vector $\overline{\mbz} = \Phi(\mbz)$, we assign each $i$-th slot \(\mbz^{(i)}\) independently to a codebook entry, following a nearest-neighbor rule (\eg the $\ell_2$ norm), via the indices:
\[
s(i) = \argmin_{j \in \{1, \dots, |\mathcal C|\}} \| \mbz^{(i)} - \mbc_j \|_2
\]
We denote the full vector of quantization indices as \(\mathbf{s}=(\, s(1), \ldots,s(w) \,)\) and the fully quantized latent representation as
\(
\mathbf{\overline{q}}_{\mathbf{s}} = [\, \mbc_{s(1)}, \dots,\, \mbc_{s(w)} \,] \in \mathbb{R}^{w\times d}
\). We define the corresponding flattened representation (obtained by concatenating the $w$ codewords) as $\mbq_{\mathbf{s}} \in \mathbb{R}^{m'}$.
The set of all possible flattened quantized configurations is
\(
\mathcal{Q} = \big\{ \mbq_{\mathbf{s}} \, : \, \mathbf{s} \in  \{1, \dots,  |\mathcal C|\}^w\big\}\subset \mathbb R^{m'}.
\)
Even if we quantize locally at the level of individual slots, product quantization partitions the full latent space into a combinatorially large number of regions~\citep{DBLP:journals/pami/JegouDS11}, as the following proposition illustrates:
\begin{proposition}
A global minimizer of the Euclidean distance in the flattened space, \[
\mbq^\star = \argmin_{\mbq \in \mathcal{Q}} \| \mbz - \mbq \|_2
\]
can be obtained by assigning each slot independently to its nearest codebook vector, namely:
\[
s^\star(i) = \argmin_{j \in \{1, \dots, |\mathcal C|\}} \| \mbz^{(i)} - \mbc_j \|_2 \quad \forall i,
\qquad
\mbq^\star = \mbq_{\mathbf{s}^\star}
\] 
\label{prop:global-minimizer-euclidean-space}
\end{proposition}
\begin{proof}
We provide the proof in \cref{sec:prop1_proof}.
\end{proof}

Each element \(\mbq \in \mathcal{Q}\) defines a \emph{Voronoi cell} centroid in the flattened feature space, namely:
\(
\mathcal{V}(\mbq)
= \bigl\{\, \mbz \in \mathbb{R}^{m'} : \| \mbz - \mbq \|_2 \le \| \mbz - \mbq' \|_2
\ \quad \forall \; \mbq' \in \mathcal{Q} \,\bigr\}
\), with potential ties handled by random assignment.
These Voronoi cells form a partition of \(\mathbb{R}^{m'}\) (see \cref{fig:wvoronoi}), as
\[
\mathbb{R}^{m'} = \bigcup_{\mbq \in \mathcal{Q}} \mathcal{V}(\mbq),
\qquad
\mathcal{V}(\mbq) \cap \mathcal{V}(\mbq') = \emptyset
\ \text{for } \mbq \ne \mbq'
\] 
Note that $(i)$ the standard per-slot nearest-neighbor quantization is equivalent to assigning
\(\mbz = E(\mbx)\)
to the centroid \(\mbq^\star \in \mathcal{Q}\) whose Voronoi region contains it; $(ii)$ the global Voronoi tessellation in \(\mathbb{R}^{m'}\) factorizes into $|\mathcal Q|=|\mathcal C|^w$ Voronoi cells.
Thus, we can interpret the latent space as a tessellation of Voronoi cells, each associated with a distinct quantized representation.
Notably, this partition
provides the structural foundation to build locally adaptive yet parameter-efficient calibration maps.

\subsection{\textit{Calibra:} Learning Local Calibration Maps}
\label{sec:local_dir_cal}

In this subsection, we formalize our calibration approach from a probabilistic perspective.
Following \citet{DBLP:conf/nips/KullPKFSF19}, we motivate the functional form of the calibration map with a generative argument. We model the distribution of predicted probabilities conditioned on the true label and Voronoi region, and derive the corresponding calibrated posterior over labels via Bayes' rule. The resulting posterior admits a tractable log-linear form presented in~\cref{prop:log_linear}.
\begin{proposition}
    \label{prop:log_linear}
    Assume that, conditioned on a cell $\mathcal{V}$ and the true label $y=j$, the predicted probability vector $\hat{\mbp}$ follows a Dirichlet distribution, i.e.,
    \(\mathbf{\hat{p}} \mid (y=j,\,\mcV) \ \sim\ \operatorname{Dir}\bigl(\boldsymbol\alpha^{(j, \mcV)}\bigr).\)
     The calibration bias and weight vector for \(j\) and \(\mcV\) are:
\[
\boldsymbol\beta_{j, \mcV} = \log \pi_{j\mid \mcV} - \log \operatorname{Beta}\bigl(\boldsymbol\alpha^{(j, \mcV)}\bigr)
\quad\quad
\boldsymbol\omega_{j, \mcV}=\boldsymbol\alpha^{(j, \mcV)}- \mathbf{1}
\]
and the logarithm of posterior probabilities admits the form:
\begin{equation}
    \log \mathbb P(y=j \mid \mathbf{\hat p}, \mcV) = \boldsymbol\beta_{j, \mcV} + \boldsymbol\omega_{j, \mcV}^{\top}\log\mathbf{\hat p} \;+\; \mathrm{const}
    \label{eqn:log-linear-form}
\end{equation}
\end{proposition}
\begin{proof}
We provide the proof in \cref{sec:proof_prop2}.
\end{proof}
Here, $\operatorname{Beta}(\cdot)$ denotes the multivariate Beta function associated with the Dirichlet distribution, $\pi_{j\mid \mcV}$ a possibly cell-dependent prior probability and $\mathbf 1\in \mathbb R^{|\calY|}$ is the all ones vector.
This results in a linear calibration model where the per-cell parameter is determined by Dirichlet concentrations.
Note that, in our pipeline, the scores $\hat\mbp$ fed to the calibrator are produced by the quantization-aware head from $\overline{\mbq}_{\mathbf s}$ (\cref{sec:algo}) and are therefore constant within each cell. Hence, the Dirichlet assumption cannot hold literally: similarly to \citet{DBLP:conf/nips/KullPKFSF19}, who fit the Dirichlet calibration map discriminatively, we use \cref{prop:log_linear} to motivate a parametric family rather than as a generative model. 

The formulation in \cref{prop:log_linear} shows that learning a calibration map requires specifying the Dirichlet concentration parameters $\boldsymbol{\alpha}^{(\mcV)}$ for each Voronoi cell $\mcV$. However, the number of such cells grows exponentially with the number of quantization slots, making an unconstrained per-cell parameterization computationally and statistically infeasible.
To address this challenge, we seek a parameterization of $\boldsymbol{\alpha}^{(\mcV)}$ that $(i)$ captures class-dependent redistribution of confidence mass, while $(ii)$ enabling parameter sharing across regions. We achieve this by constructing region-specific calibration maps compositionally from shared codeword-dependent factors, yielding a parameter-efficient representation that generalizes across exponentially many cells.

\textbf{Indexed Parameterization Trick.}
Instead of learning independent calibration maps per region, we construct them compositionally from shared codeword-dependent factors, allowing information to be reused across exponentially many regions. To enable compositional generalization across regions, we introduce an \textit{indexed parameterization trick}: discrete quantization indices are reused as selectors of calibration parameters. This induces a compositional parameterization that shares statistical strength across exponentially many regions while using only a linear number of parameters. We thereby obtain a parameter-efficient way to associate each region with a pair of \textit{sender}-\textit{receiver} calibration maps.

Let $\mathcal{C}=\{\mbc_1,\dots,\mbc_{|\mathcal C|}\}\subset\mathbb{R}^d$ be the
\emph{embedding codebook}, and
$\mathcal{A}=\{\mba_1,\dots,\mba_{|\mathcal C|}\}\subset\mathbb{R}^{|\mcY|}$,
$\mathcal{B}=\{\mathbf{b}_1,\dots,\mathbf b_{|\mathcal C|}\}\subset\mathbb{R}^{|\mcY|}$
the corresponding \emph{receiver} and \emph{sender} calibration codebooks,
respectively. All three codebooks share the same discrete indexing set $\{1,\dots,|\mathcal C|\}$.
We also introduce slot dependent scaling $\boldsymbol\sigma^2 = (\sigma_1^2,\ldots,\sigma_w^2)$.
Each sub-vector $\mbz^{(i)}$ is assigned independently to its nearest embedding centroid, yielding an index sequence $\mathbf{s}=(s(1),\dots,s(w))$, which uniquely identifies a region $\mcV$.
The same index sequence is reused to select calibration parameters from $\mathcal{A}$ and $\mathcal{B}$. This induces local \textit{sender} and \textit{receiver} weight matrices of shape $w\times|\mcY|$:
\[
\mathbf A^{(\mcV)}=
[\mba_{s(1)}, \dots, \mba_{s(w)}]^{\!\top} \quad \mathbf B^{(\mcV)}=
[\mathbf b_{s(1)}, \dots, \mathbf b_{s(w)}]^{\!\top}
\]
The matrices $\mathbf A^{(\mcV)}$ and $\mathbf B^{(\mcV)}$ define the Dirichlet concentration
parameters through the bilinear form:
\begin{equation}
\label{eq:bilinear_alpha_quant}
\boldsymbol\alpha^{(\mcV)}\;=\;\phi\big(\mathbf A^{(\mcV)\!\top}\mathrm{diag}(\boldsymbol\sigma^2)\,\mathbf B^{(\mcV)}\big)+\mathbf I
\end{equation}
where $\phi$ is a positive, monotone link function applied
element-wise to ensure positivity and $\boldsymbol\alpha^{(\mcV)}$ is a $|\mcY|\times|\mcY|$ matrix whose $j$-th column is the concentration vector $\boldsymbol\alpha^{(j,\mcV)}$ of \cref{prop:log_linear}, \ie $\alpha^{(\mcV)}_{ij}$ encodes the evidence toward predicted class $i$ when the true class is $j$.
Finally, the identity shift guarantees $\alpha^{(\mcV)}_{jj}>1$ and allows the map to be initialized close to the identity (\cref{sec:experimental_details}). We provide an illustrative example of the procedure in \cref{fig:pract_examp} in the Appendix.

This construction admits a natural interpretation in terms of class-dependent evidence redistribution. Each codeword contributes to a pair of vectors: a \emph{sender} $\mathbf b_k$, describing how confidence mass is emitted when a given class is true, and a \emph{receiver} $\mathbf a_k$, describing how this mass is accumulated across predicted classes. The bilinear interaction combines these contributions across slots, capturing how evidence is transferred between classes in a region-specific manner. In this sense, the model represents miscalibration as structured, class-dependent redistribution of confidence mass.
Moreover, the diagonal weighting $\mathrm{diag}(\boldsymbol{\sigma}^2)$ introduces slot-dependent scaling, allowing different positions in the index sequence to contribute unequally.
When the entries of $\boldsymbol\sigma^2$ are distinct, this breaks permutation invariance over the sequence $\mathbf s$, allowing different index configurations to produce distinct calibration maps.
Conversely, for $\boldsymbol\sigma^2=\mathbf 1$, the map then depends on $\mathbf s$ only through its codeword counts.

This construction changes the number of parameters from $(|\mcY|^2+|\mcY|)\times|\mathcal C|^w$, when per-cell concentrations and priors are learned, to $2\times |\mathcal C| \times |\mcY| + |\mcY| + w$.
In practice, when $|\mathcal C|,  w > 1$, our proposal substantially decreases the number of parameters. Finally, parameter sharing is key to statistical stability, as it leverages recurring structure across regions.
As we show in the following, it keeps the estimation error of the calibration parameters under control even in low-density regions.

\textbf{Statistical Stability.}
\label{sec:stability}
The key insight of our approach is the following: while individual Voronoi cells may be sparsely populated, the \textit{codewords} that compose them are frequently reused. By tying calibration parameters to codewords rather than cells, we gain statistical stability. We support the benefits of this parameter sharing through two \emph{local} statistical results and a corollary, conditional on the optimization trajectory entering a well-behaved neighborhood of a stationary point of the loss. 
Let $\theta$ collect the calibration parameters (the codebooks $\mathcal A,\mathcal B$, the global prior, and $\boldsymbol\sigma^2$ when learned) and $N=|\mathcal D_{\mathrm{cal}}|$ denote the number of calibration samples. Our analysis concerns the calibration stage, conditional on the codebook and the quantization-aware head (\cref{ap:assumptions}).
First, let us assume the following regularity conditions hold, i.e., (A1) \textit{smoothness}, (A2) \textit{uniform convergence}, (A3) \textit{gradient concentration}, (A4) \textit{local curvature}, and (A5) \textit{approximate stationarity}\footnote{We provide a detailed description of all our assumptions for \cref{thm:local_consistency}, \cref{thm:convergence_rates} and \cref{cor:cell_est_err} in \cref{ap:assumptions}.}. Under such assumptions, we show that, despite the compositional structure, the estimation procedure is locally statistically consistent, as formally stated in the following Theorem:

\begin{theorem}[Local Rate and Consistency]\label{thm:local_consistency}
Let $\mathcal L_N(\theta)$ be the empirical cross-entropy loss and $\Theta_{\text{local}}$ be a compact convex neighborhood of a population local minimizer $\theta^\star$ (the pseudo-true parameter).
Let $\hat\theta_N \in \Theta_{\text{local}}$ be an approximate stationary point satisfying $\|\nabla \mathcal L_N(\hat\theta_N)\|_2 \le \varepsilon_N$, where $\varepsilon_N \xrightarrow{p} 0$. Under the assumptions (A1)-(A5), the estimator converges in probability to the population minimizer:
\[
\|\hat\theta_N-\theta^\star\|_2
=
O_p\left(\frac1{\sqrt{N}}\right) + O_p(\varepsilon_N)
\]
In particular, if \(\varepsilon_N=o_p(1)\), then \(\hat\theta_N \xrightarrow{p}\theta^\star\).
\end{theorem}
\begin{proof}
We provide the proof in \cref{sec:proof_thm1}. \end{proof}
While consistency ensures correctness in the limit, it does not describe behavior under data sparsity.
We introduce additional assumptions - (A6) weak \textit{cross-block coupling}, (A7) \textit{bounded conditional influence between codewords}, (A8) \textit{standard moment and concentration conditions} and (A9) \textit{Local Hessian Lipschitzness} -
to capture the mechanism by which parameter sharing induced by VQ can stabilize optimization empirically. 
If these assumptions hold, the leading term of the error of each codebook parameter scales with the number of occurrences of the corresponding codeword, not with the number of calibration samples falling in a given cell $\mathcal V$, as shown in the following Theorem:

\begin{theorem}[Frequency-Weighted Convergence of Codeword Parameters]
\label{thm:convergence_rates}
Assume (A1)-(A9), and condition on the global scale
$\boldsymbol{\sigma}^2$, the codebook $\mathcal C$ and the
assignments induced by quantization.
Let $\boldsymbol{\pi}_k$ denote a codebook parameter block
$\mathbf a_k \in \mathcal A$ or $\mathbf b_k \in \mathcal B$, and let
\(
p_k = \sum_{i=1}^{w} \mathbb{P}(s(i)=k)
\)
denote the expected number of occurrences of codeword $k$ per sample.
For any active codeword with $p_k>0$,
\[
\|\boldsymbol{\hat{\pi}}_k-\boldsymbol{\pi}_k^\star\|_2
=
O_p\!\left(\frac{1}{\sqrt{N p_k}}\right)
+
O_p\!\left(\frac{1}{\sqrt N}\right)
+
O_p\!\left(\frac{\varepsilon_N}{p_k}\right)
\]
and, since $N_k/N \xrightarrow{p} p_k$, the leading codeword-specific statistical term scales as $\nicefrac{1}{\sqrt{N_k}}$.
\end{theorem}
\begin{proof}
We provide the proof in \cref{sec:proof_thm2}. 
\end{proof}
Note that \cref{thm:convergence_rates} applies to each codeword individually and does not imply convergence over the entire codebook. Moreover, when the optimization error is negligible relative to the statistical term, i.e., \(\varepsilon_N=o_p(\nicefrac{\sqrt{N_k}}{N})\), estimation depends on the actual number of occurrences of the corresponding codeword. The main benefit of parameter sharing is precisely that a codeword can be accurately estimated from all of its occurrences across the calibration set, independently of how frequently any particular combination of codewords is observed. Consequently, a Voronoi cell may be rare, or even absent from the calibration sample, but still associated with a well-estimated calibration map, provided that its constituent codewords are sufficiently prevalent. We now formalize this cell-level property by bounding the estimation error of the corresponding Dirichlet concentrations.

\begin{corollary}[Cell-Level Estimation Error]
\label{cor:cell_est_err}
Under the setting of \cref{thm:convergence_rates}, further assume that the link function $\phi$ is positive, monotone and $1$-Lipschitz, \eg the softplus. Since the codebook is fixed and finite, for every fixed cell $\mcV$ with index sequence $\mathbf s$ of active codewords, the Frobenius norm of the estimation error in the local Dirichlet concentrations satisfies:
\[
\big\|
\hat{\boldsymbol\alpha}^{(\mcV)}
-
\boldsymbol\alpha^{(\mcV)\star}
\big\|_F
=
O_p\!\left(
\sum_{i=1}^w
\sigma_i^2
\left[
\frac{1}{\sqrt{N p_{s(i)}}}
+
\frac{1}{\sqrt N}
+
\frac{\varepsilon_N}{p_{s(i)}}
\right]
\right)
\]
\end{corollary}
\begin{proof}
We provide the proof in \cref{app:proof_cor_1}.
\end{proof}
Hence, the estimation error of a cell depends on the estimation errors of its constituent codewords, rather than on the number of calibration samples falling in that cell. In particular, a cell may contain few or even no calibration samples while still admitting a well-estimated calibration map, provided that its constituent codewords are active and
sufficiently well estimated. Because $\|\mba_k^\star\|_2$ and $\|\mathbf b_k^\star\|_2$
are uniformly bounded over active codewords, combining the above bound with \cref{thm:convergence_rates} yields a cell-level error controlled by the frequencies of the codewords appearing in the index sequence $\mathbf s$.

\textbf{Learning procedure.}
In practice, we propose a two-stage procedure (shown in \cref{alg:appendix_calibration} in the Appendix) to decouple representation learning from calibration: 
\begin{enumerate}
    \item \textit{Quantization-Aware Representation Learning:} We freeze the backbone encoder and train the codebook $\mathcal{C}$ and a quantization-aware classification head. Codebooks are updated via Exponential Moving Average (EMA) to capture the density of the latent space~\citep{van2017neural}.
    \item \textit{Region-Aware Calibration:} We freeze the codebook and quantized classifier, then optimize the calibration parameters
$(\mathcal{A}, \mathcal{B})$ (and, optionally, $\boldsymbol\sigma^2$, which we keep fixed in our experiments) by minimizing the negative log-likelihood of the calibrated posterior.
\end{enumerate}

Both stages are trained on the calibration set $\mathcal D_{\mathrm{cal}}$ (\cref{alg:appendix_calibration}).
The output is a codebook \(\mathcal C\) defining the tessellation of the latent space and a calibrated predictor \(f_{\mathrm{cal}}\).

\section{Experimental Evaluation}
\label{sec:experimental-evaluation}

In our experiments, we address the following questions:

\begin{itemize}
    \item \textbf{Q1:} Does our method improve over existing baselines on \textit{local calibration} metrics?
    \item \textbf{Q2:} How does our method perform on low-support regions, where calibration is hardest?
    \item \textbf{Q3:} Is our method competitive with existing baselines on \textit{global calibration} metrics?
\end{itemize}
The code\footnote{Code developed by authors from University of Trento.} can be found at ~\url{https://github.com/Cesbar99/Divide-et-Calibra}.

\subsection{Experimental Settings}

\textbf{Datasets.}
We evaluate the research questions over the same image multiclass datasets used by \citet{DBLP:journals/corr/Barbera25}, \ie classic benchmark data \cifarTen{}, \cifar{}~\citep{krizhevsky2009learning}, and real-world medical data \tissueM{} from the \texttt{MedMNIST} collection~\citep{DBLP:conf/isbi/YangSN21,yang2023medmnist,ljosa2012annotated}. 
We also evaluate our approach on tabular data, using the \texttt{Weather} dataset~\citep{DBLP:conf/nips/MalininBGGGCNPP21}; for space reasons, tabular results are reported in \cref{app:tabular}, where \VQ{} attains the best $LCE$ and is competitive on both $MLCE$ and global calibration metrics. 

\textbf{Methods and Architectures.}
We evaluate our approach (\VQ{}) against both global and local calibrators. We consider publicly available calibration baselines, including Temperature Scaling (\TS{})~\citep{DBLP:conf/icml/GuoPSW17}, Isotonic Regression (\IR{})~\citep{DBLP:conf/kdd/ZadroznyE02}, and Platt Scaling (\PS{})~\citep{platt1999probabilistic}. In addition, we compare with multiclass parametric models, \ie Dirichlet Calibration (\DC{})~\citep{DBLP:conf/nips/KullPKFSF19} and Structured Matrix Scaling (\SM{})~\citep{DBLP:journals/corr/abs-2511-03685}, and local approaches such as KCal (\KC{})~\citep{DBLP:conf/iclr/LinT023}, Local Nets (\LN{})~\citep{DBLP:journals/corr/Barbera25} and ProCal (\PC{})~\citep{DBLP:conf/nips/XiongDKW0XH23}. Finally, we include the base model when no calibration is performed (\NC{}). 

For image data, we report results using \texttt{ResNet} architectures in the main paper, for a direct comparison with the \citet{DBLP:journals/corr/Barbera25} benchmark.
Results are largely consistent across backbones, with smaller gains for \texttt{ViT} on \cifarTen{} and \cifar{}. We provide values for \texttt{ConvNext}~\citep{liu2022convnet} and \texttt{ViT}~\citep{dosovitskiy2020vit} architectures in \cref{sec:convnext_results} and \cref{sec:vit_results}, respectively.

\textbf{Implementation details.}
For all experiments, we set the number of slots $w=64$ and codebook size $|\mathcal{C}|=64$ for all VQ methods.
Still, our proposal is robust to this choice of hyperparameters, showing good performance across a wide range of values.
We provide results when varying $w$ and $|\mathcal{C}|$ in \cref{sec:codewordsandslots}. Moreover, we fix $\mathrm{diag}(\boldsymbol{\sigma}^2) = \mathbf I$ to avoid over-fitting due to limited calibration sets. Indeed, over-fitting is a well-known issue in post-hoc calibration, where simpler models and stronger regularization are often preferred \cite{DBLP:journals/corr/abs-2511-03685}. This choice yields a regularized model that preserves the compositional structure induced by codeword reuse. 
We instead report a variant of our method with a regularized, learnable $\mathrm{diag}(\boldsymbol{\sigma}^2)$ in~\cref{sec:reg_sigma}. This variant achieves performance comparable to the fixed-scaling formulation presented here, with modest improvements on some metrics.
Similarly, we use a global learned prior $\pi_{j\mid\mcV}=\pi_j$, so that locality enters the calibration bias $\boldsymbol \beta_{j,\calV}$ only through $\log \operatorname{Beta}(\boldsymbol\alpha^{(j,\mcV)})$.
We detail all other hyperparameters in~\cref{sec:experimental_details}. 

\textbf{Metrics.}
We evaluate local calibration using both $LCE$ and $MLCE$ (see \cref{sec:notions_metrics}) as the main metrics of interest.
We assess global calibration using $ECCE$ and the Negative Log-Likelihood ($NLL$), which is a proper scoring rule and captures both calibration and predictive performance.
Moreover, we report in \cref{app:extraRes} also $ECE$ and accuracy $(ACC)$ for all datasets and baselines. 

\textbf{Experimental Setup.}
For image datasets, we consider the \citet{DBLP:journals/corr/Barbera25} data splits, \ie $(a)$ a training set; $(b)$ a calibration set (with an internal $90\%/10\%$ calibration/validation split); and $(c)$ a held-out test set used exclusively for evaluation. 
Then, we $(i)$ train classifiers on the training set; $(ii)$ train calibration methods on the calibration set; $(iii)$ compute the metrics defined above on the test set. We repeat this procedure using $5$ different seeds and average results.
We provide further implementation details in \cref{sec:experimental_details}.

\textbf{Additional Experiments.}
We include in~\cref{sec:ablate_roles} further ablations for the components of our calibration pipeline (confirming quantization as the primary source of improvement, with consistent gains from our compositional parameterization) and for the calibration set size (showing consistent performance; see \cref{sec:ablate_calsize}). Finally, in~\cref{sec:usage_stats} we report codeword usage statistics, showing good codebook utilizations and lack of unused codewords.
We further disentangle the effect of quantization from classification-head adaptation in~\cref{sec:disentangle}, investigate the role of the quantization partition itself in~\cref{sec:quantization_role}, and evaluate sensitivity of LCE and MLCE to the kernel bandwidth in~\cref{sec:ablating_gamma}. Finally, \cref{sec:time} reports training-time measurements.
\subsection{Experimental Results}

\textbf{Q1: \VQ{} outperforms all baselines in local calibration metrics.} \Cref{fig:Q1} shows the results for local calibration metrics, \ie $LCE$ and $MLCE$. 

Regarding $LCE$, our approach outperforms all existing baselines on all datasets: on \cifarTen{}, \VQ{} reaches $LCE\approx 0.0059\pm 0.0002$, obtaining a $\approx25\%$ drop with respect to \LN{}, the closest competitor ($LCE \approx 0.0079\pm 0.0003$); on \cifar{}, \VQ{} obtains $LCE\approx 0.0017\pm 0.0001$, a $\approx 29\%$ drop compared to \LN{} (the second-best with $LCE \approx 0.0024 \pm 0.0001$); on \tissue{}, \VQ{}'s $LCE$ is $\approx 0.0083\pm 0.0002$, which is a $\approx 60\%$ drop with respect to \LN{} ($LCE\approx 0.0208\pm 0.0007$).
Regarding $MLCE$, we see similar results, with \VQ{} achieving the best results on \cifarTen{} and \cifar{} and ranking second best on \tissue{}.

Overall, results indicate that our approach outperforms competitors on local calibration metrics.

\begin{figure*}[t]
        \includegraphics[width=\linewidth]{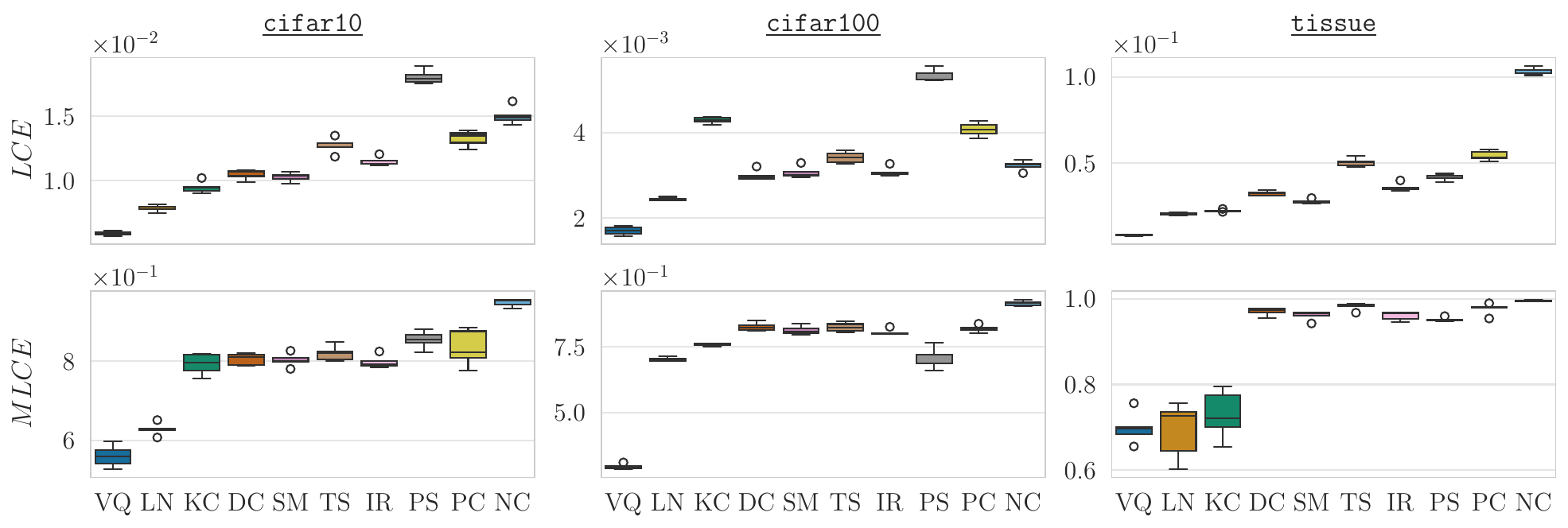}
\caption{Local calibration metrics over five runs (\texttt{ResNet} backbone). The lower the better.}
    \label{fig:Q1}
\end{figure*}

\textbf{Q2: \VQ{} obtains most pronounced gains in low support regions.}
To better understand the behavior of calibration under data sparsity, we analyze local calibration error as a function of the Effective Sample Size ($ESS$) of each point neighborhood. Specifically, we partition samples into quantile-based bins according to their $ESS$, computed from kernel weights in the representation space, and report the average absolute $LCE$ within each bin (\cref{fig:Q2}).

We observe similar trends for all datasets: local calibration error decreases as the effective sample size increases. This behavior is expected, as points with low $ESS$ belong to low-density regions, which are harder to correct.
Importantly, while all methods exhibit this trend, our \VQ{}-based approach consistently achieves lower $LCE$ across the entire $ESS$ spectrum, with the most pronounced gains in low-support regions. This suggests that \VQ{} effectively mitigates the impact of data sparsity.
We attribute this behavior to the parameter sharing induced by vector quantization: although individual Voronoi regions may be sparsely populated, their constituent codewords are reused across many samples.
Interestingly, this empirical observation is consistent with our theoretical analysis (\cref{sec:stability}), where the estimation error of the calibration parameters depends on codeword frequency rather than on the number of samples in each region. 

Overall, these findings highlight that our method reduces the sensitivity of calibration error to data sparsity, enabling reliable calibration even in low-density regions of the representation space.

\begin{figure*}[t]
        \includegraphics[width=\linewidth]{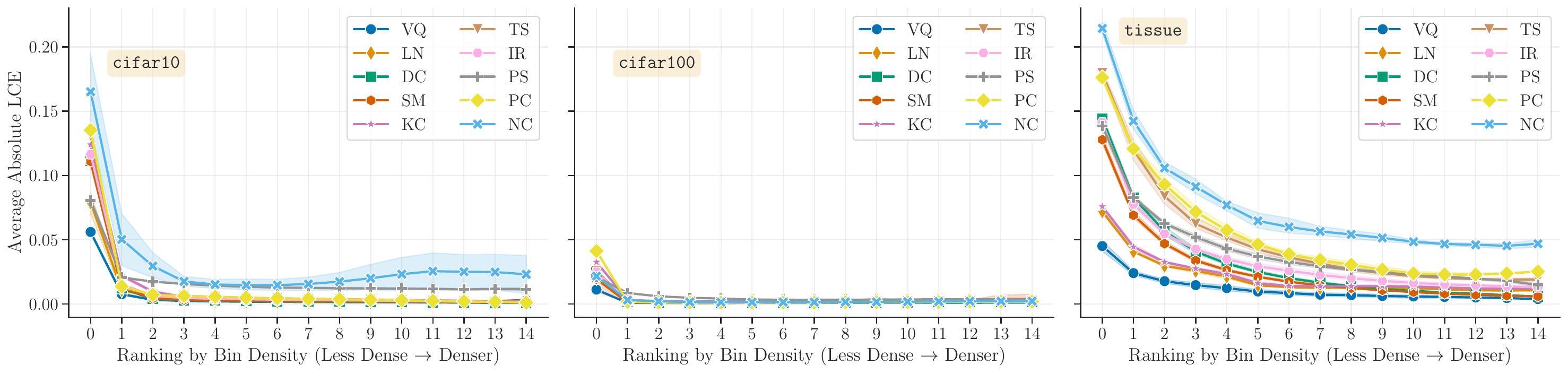}
    \caption{{Local calibration in density-based sub-bins over five runs (\texttt{ResNet} backbone).}}
    \label{fig:Q2}
\end{figure*}

\textbf{Q3: \VQ{} reaches comparable results on other metrics.}
\cref{fig:Q3} reports $ECCE$ and $NLL$ for all methods and datasets ($ECE$ and $ACC$ are reported in \cref{tab:allResults} of the appendix). 

Regarding $ECCE$, our approach improves upon other local approaches, such as \KC{}, \LN{} and \PC{}, on \cifarTen{} and \tissue{}, while performing similarly on \cifar{}.
Similarly, when compared to \SM{}, the strongest global calibration method, \VQ{} achieves comparable results on \cifar{} while closely following on the remaining datasets (see \cref{tab:allResults} for detailed results).
For $NLL$, \VQ{} improves over the base model, while remaining slightly behind the best results in some settings.
As intended, \VQ{} delivers significant improvements on local calibration metrics, with its clearest advantages on sparse regions. On global calibration and predictive metrics, it remains competitive with strong baselines, indicating that local gains are not obtained at the cost of severe global degradation.

\begin{figure*}[t]
        \includegraphics[width=\linewidth]{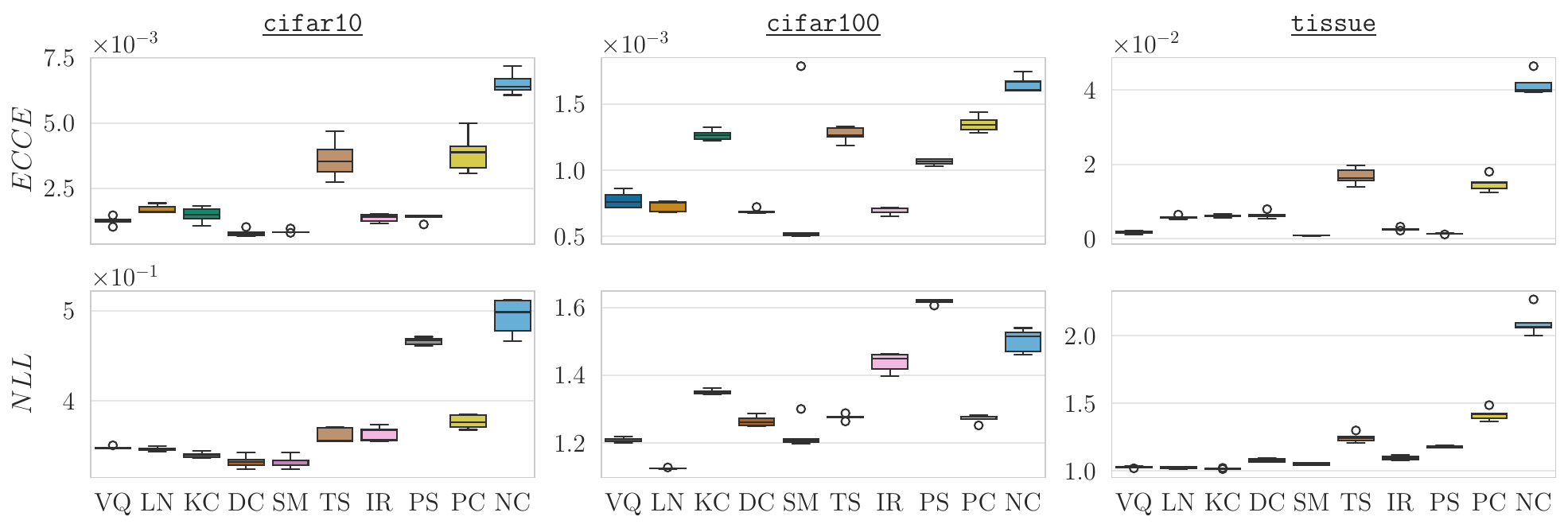}
    \caption{Global calibration metrics over five runs (\texttt{ResNet} backbone). The lower the better.} 
    \label{fig:Q3}
\end{figure*}

\section{Related Work}

\textbf{Global Calibration.} Standard post-hoc calibration involves learning a single, global map to obtain calibrated probabilities. Parametric methods such as \textit{Platt Scaling}~\citep{platt1999probabilistic}, \textit{Temperature Scaling}~\citep{DBLP:conf/icml/GuoPSW17}, and \textit{Beta Calibration}~\citep{kull2017beta} are widely used. These methods generally learn a monotonic function to rescale the logits (or log-probabilities), preserving the rank ordering of classes. Non-parametric approaches, such as \textit{Isotonic Regression}~\citep{DBLP:conf/kdd/ZadroznyE02} and \textit{Histogram Binning}~\citep{DBLP:conf/icml/ZadroznyE01}, offer greater flexibility but require significantly more data to estimate the calibration map, as they rely on having sufficient samples in each bin to compute reliable statistics.
In the multiclass setting, \textit{Dirichlet Calibration} (DC)~\citep{DBLP:conf/nips/KullPKFSF19} generalizes the Beta calibration framework to the probability simplex. Recently, \citet{DBLP:journals/corr/abs-2511-03685} introduced an effective regularization scheme showing promising results in mitigating over-fitting.

\textbf{Local Calibration.}
To address the limitations of global maps, \textit{local calibration} methods make the correction function dependent on the input features or embeddings. Approaches like \textit{Local Temperature Scaling}~\citep{ding2021local}, \textit{KCal}~\citep{DBLP:conf/iclr/LinT023}, \textit{LECE} \citep{valk2023assuming} and \textit{ProCal} \citep{DBLP:conf/nips/XiongDKW0XH23} typically employ kernel density estimation or nearest-neighbor lookups to estimate local error rates. More recently, \textit{Local Nets}~\citep{DBLP:journals/corr/Barbera25} proposed using a secondary neural network to predict calibrated probabilities dynamically.
Our VQ-based approach differs from these techniques by replacing continuous density estimation with Voronoi tessellations.
This allows statistics to be pooled across regions through shared codewords, so that the estimation error of each local calibration map is controlled by codeword frequency (\cref{sec:stability}).

\textbf{Multicalibration.}
Another calibration notion, known as \textit{Multicalibration}, requires the model to be calibrated not just on the entire population, but on virtually any sub-population identifiable by a specific hypothesis class~\citep{hebert2018multicalibration,jung2021moment,colinibaldeschi2026multicalibration,jin2025discretization,perini2025mcgrad, DBLP:journals/corr/abs-2602-06773}.
Thus, the goal of multicalibration approaches is to ensure that ML models do \textit{not} make biased predictions on any of these sub-populations~\citep{globus2023multicalibrated,noarov2023statistical}.
While related, multicalibration differs from our setting in two key aspects. First, it is mostly studied for binary/regression tasks, whereas we address \emph{multiclass} calibration and the complexities of the probability simplex. Second, it identifies semantically meaningful subgroups (\eg \texttt{User Age $>$ 18}) using expressive features in tabular data.

\textbf{Calibration and Conformal Prediction.}
Conformal prediction provides distribution-free prediction sets with finite-sample marginal coverage guarantees under exchangeability~\citep{vovk2005algorithmic}. 
Mondrian conformal prediction extends this idea providing finer guarantees than global marginal validity conditioning on predefined categories or partitions of the data~\citep{DBLP:conf/copa/BostromJL21,DBLP:conf/icml/LaanA25}.
More broadly, recent work has also investigated how post-hoc calibration affects the coverage and efficiency of conformal predictors~\citep{DBLP:conf/icml/DabahT25, DBLP:journals/tmlr/XiHL0W25}. These approaches are complementary to our setting, which instead learns region-specific calibration maps.

\section{Conclusion}

In this work, we tackle the trade-off between stable but coarse global calibration techniques and flexible but high-variance local calibration approaches.
We proposed a compositional approach to local multiclass calibration, instantiated via vector quantization.
Empirical evidence shows improvements over existing baselines in terms of local calibration metrics, while retaining comparable results on global calibration and classification metrics. 

\textbf{Limitations and Future Work.}
\label{sec:limitations}
Our method relies on a frozen representation space, and therefore its success depends on the extent to which that space captures locality relevant to miscalibration. If latent features are poorly aligned with the calibration structure the induced Voronoi partition may be less effective. Interestingly, such partition defines a finite family of groups in latent space, suggesting a promising direction for future work: bridging local calibration and multicalibration by deriving guarantees for input-space groups while operating over representation-defined subpopulations.

Moreover, our theory is local and conditional on regularity assumptions, stable quantization assignments, and convergence to a well-behaved stationary point. It covers the calibration stage only and treats the codebook and the quantization-aware head as fixed. Extending it to joint end-to-end learning via recent VQ differentiable methods~\citep{WArno26} or growing-codebook regimes is left for future work.

We further stress that, while VQ consistently improves local calibration empirically, it does not uniformly outperform strong global baselines on $ECCE$ or $NLL$, and its advantages are clearest in low-support in-distribution regions. This suggests that a promising direction is to design calibration models that explicitly learn when to apply local corrections and when to rely on global adjustments.

Additionally, while our empirical evaluation includes multiple datasets, baselines, and architectures, it remains focused on classification settings. Extending the study to additional modalities, larger-scale tasks, and more diverse forms of distribution shift is an important direction for future work.

In conclusion, better-calibrated probabilities matter because they turn model scores into information that people can act on:
investigating how reliable probability estimates interact with downstream decision-making and uncertainty-aware guarantees is a fundamental direction for future research.

\section*{Broader Impact}
\label{sec:broader}
This work aims to improve the reliability of probabilistic predictions by reducing miscalibration in local regions of the data space, which is particularly relevant when models are used for risk-sensitive decisions (e.g., medical triage, content moderation, or autonomous systems). Better local calibration enables more trustworthy uncertainty estimates, possibly improving downstream decision-making.

\section*{Acknowledgements}

We thank the Anonymous Reviewers and Area Chair for their feedback.
Cesare Barbera, Andrea Pugnana and Andrea Passerini acknowledge the CINECA award under the ISCRA initiative for the availability of high-performance computing resources and support. This work was funded by the European Union. The views and opinions expressed are however those of the author(s) only and do not necessarily reflect those of the European Union, the European Health and Digital Executive Agency (HaDEA) or the European Research Executive Agency. Neither the European Union nor the granting authority can be held responsible for them. Grant Agreement no. 101120763 - TANGO.

\bibliography{bibliography}
\bibliographystyle{unsrtnat}

\clearpage

\onecolumn
\appendix
\crefname{appendix}{Appendix}{Appendices}
\Crefname{appendix}{Appendix}{Appendices}

\crefalias{section}{appendix}
\crefalias{subsection}{appendix}
\crefalias{subsubsection}{appendix}
\section{Proofs}
\label{app:proofs}

\newcommand{\dN}{\Delta_N}
\newcommand{\Htilde}{\nabla^2\mathcal{L}(\tilde{\theta})}
\newcommand{\Hstar}{\nabla^2\mathcal{L}(\theta^{\star})}
\newcommand{\gtilde}{\nabla\mathcal{L}(\tilde \theta)}
\newcommand{\gstar}{\nabla\mathcal{L}(\theta^{\star})}
\newcommand{\ghat}{\nabla\mathcal{L}(\hat{\theta}_N)}

\subsection{Proof of \texorpdfstring{\cref{prop:global-minimizer-euclidean-space}}{Proposition 1}} 
\label{sec:prop1_proof}

\begin{paragraph}{Statement:}
A global minimizer of the Euclidean distance in the flattened space, \[
\mbq^\star = \argmin_{\mbq \in \mathcal{Q}} \| \mbz - \mbq \|_2
\]
can be obtained by assigning each slot independently to its nearest codebook vector, namely:
\[
s^\star(i) = \argmin_{j \in \{1, \dots, |\mathcal C|\}} \| \mbz^{(i)} - \mbc_j \|_2 \quad \forall i,
\qquad
\mbq^\star = \mbq_{\mathbf{s}^\star}
\] 
\end{paragraph}

\begin{proof}
It suffices to show that:
\[
\min_{\mbq \in \mathcal{Q}} \| \mbz - \mbq \|^2 = \sum_{i=1}^{w} \min_{j \in \{1, \dots, |\mathcal C|\}} \| \mbz^{(i)} - \mbc_j \|^2 \notag
\]
Take any index vector \(\mathbf{s}\),
\begin{align}
\| \mbz - \mbq_{\mathbf{s}} \|^2
= \big\| ( \mbz^{(1)} - \mbc_{s(1)}, \dots, \mbz^{(w)} - \mbc_{s(w)} ) \big\|^2 = \sum_{i=1}^{w} \| \mbz^{(i)} - \mbc_{s(i)} \|^2 \notag
\end{align} Because the Euclidean norm is additive over the \(w\) blocks, minimizing over all vectors \(\mathbf s\) separates into \(w\) independent minimizations, each solved by:
\[
s^\star(i) = \arg\min_j \| \mbz^{(i)} - \mbc_j \|^2\notag
\]
The vector $\mathbf{s}^\star$ containing those indices
therefore corresponds to the flattened codeword $\mbq^\star$ that attains the global minimum.
\end{proof}

\subsection{Proof of \texorpdfstring{\cref{prop:log_linear}}{Proposition 2}} \label{sec:proof_prop2}

\begin{paragraph}{Statement:}
Assume that, conditioned on a cell $\mathcal{V}$ and the true label $y=j$, the predicted probability vector $\hat{\mbp}$ follows a Dirichlet distribution, i.e.,
    \(\mathbf{\hat{p}} \mid (y=j,\,\mcV) \ \sim\ \operatorname{Dir}\bigl(\boldsymbol\alpha^{(j, \mcV)}\bigr).\)
     The calibration bias and weight vector for \(j\) and \(\mcV\) are:
\[
\boldsymbol\beta_{j, \mcV} = \log \pi_{j\mid \mcV} - \log \operatorname{Beta}\bigl(\boldsymbol\alpha^{(j, \mcV)}\bigr)
\quad\quad
\boldsymbol\omega_{j, \mcV}=\boldsymbol\alpha^{(j, \mcV)}- \mathbf{1}
\]
and the logarithm of posterior probabilities admits the form:
\begin{equation}
    \log \mathbb P(y=j \mid \mathbf{\hat p}, \mcV) = \boldsymbol\beta_{j, \mcV} + \boldsymbol\omega_{j, \mcV}^{\top}\log\mathbf{\hat p} \;+\; \mathrm{const}
\end{equation}
\end{paragraph}

\begin{proof}
Conditioned on the true label \(y=j\) and the Voronoi cell \(\mathcal V\), we assume the model's output probability vector \(\mathbf{\hat p}\) is distributed according to a Dirichlet:
\[
\mathbf{\hat{p}} \mid (y=j,\,\mcV) \ \sim\ \operatorname{Dir}\bigl(\boldsymbol\alpha^{(j, \mcV)}\bigr),
\]
where $\boldsymbol\alpha^{(j, \mcV)}$ encodes the model's belief structure over classes conditioned on the true label and the cell. Then the Dirichlet density is
\[
\mathbb P\bigl(\mathbf{\hat p} \mid y=j, \mcV\bigr)
= \frac{1}{\operatorname{Beta}\bigl(\boldsymbol\alpha^{(j, \mcV)}\bigr)}
\prod_{i=1}^{|\mcY|} \mathbf{\hat p}_{i}^{\boldsymbol\alpha^{(j, \mcV)}_{i}-1},
\]
where \(\operatorname{Beta}(\cdot)\) is the multivariate Beta function.
Assuming a (possibly cell-dependent) prior \(\pi_{j\mid \mcV} = \mathbb P(y=j\mid \mcV)\), the posterior over labels given the observed \(\mathbf{\hat p}\) and cell \(\mcV\) is
\begin{align}
\mathbb P(y=j \mid \mathbf{\hat p}, \mcV)
& \propto \pi_{j\mid \mcV}\; p\bigl(\mathbf{\hat p} \mid y=j, \mcV\bigr) \nonumber\\
&= \pi_{j\mid \mcV}\; \frac{1}{\operatorname{Beta}\bigl(\boldsymbol\alpha^{(j, \mcV)}\bigr)}
\prod_{i=1}^{|\mcY|} \mathbf{\hat p}_{i}^{\boldsymbol\alpha^{(j,\mcV)}_{i}-1}\label{eqn:posterior-raw}
\end{align}

\textbf{Log-linear calibration form.}
Taking logs in \eqref{eqn:posterior-raw} gives
\begin{align}\label{eq:log-post}
\log p(y=j \mid \mathbf{\hat p}, \mcV) &= \log \pi_{j\mid \mcV} - \log \operatorname{Beta}\bigl(\boldsymbol\alpha^{(j, \mcV)}\bigr)
+ \sum_{i=1}^{|\mcY|} \bigl(\boldsymbol\alpha^{(j, \mcV)}_{i}-1\bigr)\log \mathbf{\hat p}_{i} + \mathrm{const},
\end{align}
where the additive constant enforces normalization across \(j\).
The calibration bias and weight vector for \(j\) and \(\mcV\) are
\[
\boldsymbol\beta_{j, \mcV} = \log \pi_{j\mid \mcV} - \log \operatorname{Beta}\bigl(\boldsymbol\alpha^{(j, \mcV)}\bigr)
\quad
\boldsymbol\omega_{j, \mcV}=\boldsymbol\alpha^{(j, \mcV)}- \mathbf{1}
\]
where \(\mathbf{1}\in\mathbb{R}^{|\mcY|}\) is the all-ones vector. Then, \eqref{eq:log-post} simplifies:
\begin{equation*}
    \log \mathbb P(y=j \mid \mathbf{\hat p}, \mcV) = \boldsymbol\beta_{j, \mcV} + \boldsymbol\omega_{j, \mcV}^{\top}\log\mathbf{\hat p} \;+\; \mathrm{const}
\end{equation*} 

\end{proof}

\subsection{Assumptions for \texorpdfstring{\cref{thm:local_consistency} and \cref{thm:convergence_rates}}{Theorems 1 and 2}} 
\label{ap:assumptions}

In the following we discuss the assumptions underlying our statistical results. 
Assumptions (A1)-(A5) are standard local regularity conditions for M-estimators and Z-estimators, see \eg \citep{van2000asymptotic, white1982maximum}. They ensure well-defined population gradients and Hessians, local identifiability of the statistical target, uniform convergence of the empirical objective and curvature in a neighborhood of the solution, and stability under approximate stationarity of the optimization procedure. 

Assumptions (A6)-(A7) impose a weak blockwise coupling condition through Schur dominance of the population Hessian. Intuitively, this rules out pathological over-parameterization in which changes in one block can be arbitrarily compensated by others, and is realistic in compositional models where codewords contribute additively and appear with heterogeneous frequencies. 

Assumption (A8) collects mild moment and dependence conditions needed to control empirical gradients and Hessians via standard concentration and law-of-large-numbers arguments. 

Finally, Assumption (A9) imposes a local Lipschitz continuity condition on the empirical Hessian. 
This is a standard smoothness requirement ensuring that curvature does not vary too abruptly within the neighborhood \(\Theta_{\rm local}\). 
Such conditions are mild and are typically satisfied when the loss is sufficiently smooth in a local neighborhood of the optimum.

Throughout, we work in a local fixed-dimension regime (with $|\mathcal C|, w$ treated as fixed and independent of $N$), and all results are local in nature, characterizing statistical behavior once optimization reaches the basin of a well-specified solution. \Cref{thm:convergence_rates} additionally conditions on the global scale parameters $\boldsymbol\sigma^2$ being locally stable within the same basin of attraction as the pseudo-true parameter $\theta^\star$ to isolate the statistical behavior of the compositional codebook parameters. Extending the analysis beyond this assumption is out of the scope of this work. 

We stress that the bilinear calibration parameterization is invariant to reciprocal
rescaling of the sender and receiver factors. In particular, for each
active codeword $k$,
\[
(\mathbf a_k,\mathbf b_k)
\mapsto
(t_k\mathbf a_k,\mathbf b_k/t_k),
\qquad t_k\neq 0,
\]
leaves the induced calibration map unchanged.
Our statistical analysis is therefore understood on a locally
identifiable parameterization obtained by fixing this scale ambiguity (equivalently, by imposing a local normalization constraint on one factor of each active codeword). 
All gradients and Hessians below are taken with respect to the resulting free local coordinates.
In the experiments, we instead resolve this degeneracy operationally through the small weight-decay penalty used during optimization. The regularization used in training is not part of the statistical
objective analyzed below.

More precisely, we assume there exists a neighborhood $\mathcal N(\theta^\star)$ and a convex compact set
$\Theta_{\mathrm{local}}\subset \mathcal N(\theta^\star)$ with $\theta^\star \in \Theta_{\text{local}}$ such that the following holds:

\begin{enumerate}[label=(A\arabic*)]

\item \textbf{Smoothness and exchangeability of expectation.} The loss $\mathcal{L}(z;\theta) $ is twice continuously differentiable in a neighborhood of $\theta^\star$, with derivatives dominated by integrable envelopes so that differentiation may be interchanged with expectation for both the gradient and Hessian.

\item \textbf{Uniform convergence.}
$\sup_{\theta\in\Theta_{\mathrm{local}}}|\mathcal L_N(\theta)-\mathcal L(\theta)|\xrightarrow{p}0$, \(
\sup_{\theta\in\Theta_{\mathrm{local}}}
\big\|
\nabla\mathcal L_N(\theta)-\nabla\mathcal L(\theta)
\big\|_2
\xrightarrow{p}0
\) and
\(
\sup_{\theta\in\Theta_{\mathrm{local}}}
\big\|
\nabla^2\mathcal L_N(\theta)-\nabla^2\mathcal L(\theta)
\big\|_{\text{op}}
\xrightarrow{p}0 
\)

\item \textbf{Uniform stochastic equicontinuity and gradient concentration.}
In addition to uniform convergence of the Hessian, assume
\[
\|\nabla L_N(\theta^\star)-\nabla L(\theta^\star)\|_2 = O_p\left(\frac1{\sqrt{N}}\right)
\]

\item \textbf{Local curvature.}
Under the locally identifiable parameterization above,
$\theta^\star$ is the unique minimizer of $L$ on
$\Theta_{\mathrm{local}}$, and
$\nabla^2L(\theta^\star)$ is positive definite.
In particular, there exists $\lambda>0$ such that
\[
\lambda_{\min}\!\left(\nabla^2L(\theta)\right)
\geq \lambda
\qquad
\forall\theta\in\Theta_{\mathrm{local}}
\]
Hence, for all sufficiently large $N$,
$\nabla^2L_N(\theta)$ is invertible on $\Theta_{\mathrm{local}}$ with
\[
\sup_{\theta\in\Theta_{\mathrm{local}}}
\|\nabla^2L_N(\theta)^{-1}\|_{\mathrm{op}}
=
O_p(1)
\]

\item \textbf{Approximate stationarity.}
The optimization output satisfies
$\|\nabla \mathcal L_N(\hat\theta_N)\|_2 \le \varepsilon_N$ with
$\varepsilon_N\to 0$ in probability, and $\hat\theta_N\in\Theta_{\mathrm{local}}$ w.p.\ $\to 1$.

\item \textbf{Blockwise Schur dominance (weak coupling).}
For each codeword block $\boldsymbol{\pi}_k$, let $S_k$ denote the population Schur complement
\[
S_k
=
\nabla^2_{\boldsymbol{\pi}_k \boldsymbol{\pi}_k}\mathcal L(\theta^\star)
-
\nabla^2_{\boldsymbol{\pi}_k,-k}\mathcal L(\theta^\star)
\big[\nabla^2_{-k,-k}\mathcal L(\theta^\star)\big]^{-1}
\nabla^2_{-k,\boldsymbol{\pi}_k}\mathcal L(\theta^\star)
\]
Assume there exists $c\in(0,1]$ such that for all $k$,
\[
S_k \succeq c\,\nabla^2_{\boldsymbol{\pi}_k \boldsymbol{\pi}_k}\mathcal L(\theta^\star)
\]
It guarantees that the effective curvature along $\boldsymbol{\pi}_k$ remains non-degenerate even after accounting for interactions with other parameters.
\item \textbf{Bounded conditional cross-block influence.}
Let $\bar H_{k,-k}
=
\mathbb{E}\!\left[
H^{(k)}_{n,i,-k}(\theta^\star)
\;\middle|\;
s_n(i)=k
\right]$ be the expected Hessian cross-terms conditional on codeword appearance. We assume there exists $M \in \mathbb{R}_{>0}$ such that:
\[
\sup_k \left\| \bar H_{k,-k} \big[\nabla^2_{-k,-k}\mathcal L(\theta^\star)\big]^{-1} \right\|_{\text{op}} < M < \infty
\] This ensures that while codewords may interact, the strength of cross-block curvature contributed by each occurrence of a codeword is uniformly bounded relative to the global curvature of the remaining parameters.
\item \textbf{Moment and dependence conditions.}
For each block $k$, the per-occurrence gradient contributions
$\{g^{(k)}_{n,i}(\theta^\star) : (n,i)\in\mathcal I_k\}$ are mean-zero and satisfy
$\sup_{k}\mathbb E\|g^{(k)}_{n,i}(\theta^\star)\|_2^2<\infty$.

Similarly, let $H^{(k)}_{n,i}(\theta^\star)$ denote the per-occurrence Hessian
contribution to the $(\boldsymbol{\pi}_k,\boldsymbol{\pi}_k)$ block and let $H^{(k)}_{n,i,-k}(\theta^\star)$ denote
the corresponding per-occurrence contribution to the $(\boldsymbol{\pi}_k,\theta_{-k})$ cross-block.
Assume the uniform moment bounds
\(
\sup_{k}\mathbb E\|H^{(k)}_{n,i}(\theta^\star)\|_{\mathrm{op}}<\infty\) and \(
\sup_{k}\mathbb E\|H^{(k)}_{n,i,-k}(\theta^\star)\|_{\mathrm{op}}<\infty.
\)
Conditional on the codeword assignment, with the quantization rule held fixed during calibration and in the asymptotic analysis, the empirical averages of these per-occurrence contributions concentrate around their conditional expectations (\ie a conditional LLN with $\sqrt{N_k}$-rate concentration holds).

\item \textbf{Local Hessian Lipschitzness.}
There exists \(L_H<\infty\) such that, with probability tending to one,
\[
\sup_{\theta,\theta'\in\Theta_{\rm local}}
\frac{
\|\nabla^2 L_N(\theta)-\nabla^2 L_N(\theta')\|_{\rm op}
}{
\|\theta-\theta'\|_2
}
\leq L_H 
\]

\end{enumerate}

\subsection{Proof of \texorpdfstring{\cref{thm:local_consistency}}{Theorem 1}}\label{sec:proof_thm1}

\begin{paragraph}{Statement:}
Let $\mathcal L_N(\theta)$ be the empirical cross-entropy loss and $\Theta_{\text{local}}$ be a compact convex neighborhood of a population local minimizer $\theta^\star$ (the pseudo-true parameter).
Let $\hat\theta_N \in \Theta_{\text{local}}$ be an approximate stationary point satisfying $\|\nabla \mathcal L_N(\hat\theta_N)\|_2 \le \varepsilon_N$, where $\varepsilon_N \xrightarrow{p} 0$. Under the assumptions (A1)-(A5), the estimator converges in probability to the population minimizer:
\[
\|\hat\theta_N-\theta^\star\|_2
=
O_p\left(\frac1{\sqrt{N}}\right) + O_p(\varepsilon_N)
\]
In particular, if \(\varepsilon_N=o_p(1)\), then \(\hat\theta_N \xrightarrow{p}\theta^\star\).
\end{paragraph}

\begin{proof}

In the following proof, $\|\cdot\|_2$ denotes the Euclidean norm for vectors, and $\|\mathbf A\|_{\mathrm{op}} = \sup_{\|x\|_2=1} \|\mathbf A x\|_2$ denotes the induced operator (spectral) norm for matrices. Derivatives are taken with respect to the free
local coordinates of the identifiable parameterization introduced in \cref{ap:assumptions}.
Let
\[
\Delta_N = \hat\theta_N-\theta^\star
\]
By the fundamental theorem of calculus,
\[
\nabla L_N(\hat\theta_N)-\nabla L_N(\theta^\star)
=
\bar H_N \Delta_N,
\]
where
\[
\bar H_N
=
\int_0^1
\nabla^2 L_N(\theta^\star+t\Delta_N)\,dt
\]
Hence,
\[
\Delta_N
=
\bar H_N^{-1}
\left(
\nabla L_N(\hat\theta_N)-\nabla L_N(\theta^\star)
\right),
\]
so that
\[
\Delta_N
=
-
\bar H_N^{-1}\nabla L_N(\theta^\star)
+
\bar H_N^{-1}\nabla L_N(\hat\theta_N)
\]
Since $\theta^\star$ is an interior population minimizer,
$\nabla L(\theta^\star)=0$. Hence, by (A3),
\[
\|\nabla L_N(\theta^\star)\|_2
=
\|\nabla L_N(\theta^\star)-\nabla L(\theta^\star)\|_2
=
O_p\left(\frac1{\sqrt{N}}\right)
\]
By (A4) and uniform Hessian convergence (A2), all Hessians along the
segment joining $\theta^\star$ and $\hat\theta_N$ are uniformly positive
definite with probability tending to one. Hence,
\[
\|\bar H_N^{-1}\|_{\mathrm{op}} = O_p(1)
\]
Combining:
\[
\|\Delta_N\|_2 \le O_p(1)\cdot O_p\left(\frac1{\sqrt{N}}\right) + O_p(1)\cdot\varepsilon_N = O_p\left(\frac1{\sqrt{N}}\right)+O_p(\varepsilon_N)
\]
if \(\varepsilon_N=o_p(1)\), this  proves $\hat \theta_N \to_p \theta^\star$.

To interpret $\theta^\star$, we write the population risk as the expected negative log-likelihood:
\[
\mathcal L(\theta)=\mathbb E_{P_{\mathrm{true}}}\!\left[-\log p_\theta(\mby\mid \mbx)\right]
\]
Then, for each $\mbx$, the conditional cross-entropy decomposes as
\[
\mathbb E\!\left[-\log p_\theta(\mby\mid \mbx)\right]
=
H\!\left(P_{\mathrm{true}}(\cdot\mid \mbx)\right)
+
D_{\mathrm{KL}}\!\left(P_{\mathrm{true}}(\cdot\mid \mbx)\,\|\,p_\theta(\cdot\mid \mbx)\right),
\]
where the entropy term $H$ does not depend on $\theta$. Taking expectation over $\mbx$ yields
\[
\mathcal L(\theta)
= \text{const} + \mathbb E_X\!\left[
D_{\mathrm{KL}}\!\left(P_{\mathrm{true}}(\cdot\mid \mbx)\,\|\,p_\theta(\cdot\mid \mbx)\right)
\right]
\]
Hence $\theta^\star$ minimizes the (conditional) Kullback-Leibler divergence within
$\Theta_{\mathrm{local}}$, \ie it is the pseudo-true parameter \citep{white1982maximum}.

\end{proof}

\subsection{Proof of \texorpdfstring{\cref{thm:convergence_rates}}{Theorem 2}}\label{sec:proof_thm2}
\textbf{Statement:} Assume (A1)-(A9), and condition on the global scale
$\boldsymbol{\sigma}^2$, the codebook $\mathcal C$ and the
assignments induced by quantization.
Let $\boldsymbol{\pi}_k$ denote a codebook parameter block
$\mathbf a_k \in \mathcal A$ or $\mathbf b_k \in \mathcal B$, and let
\(
p_k = \sum_{i=1}^{w} \mathbb{P}(s(i)=k)
\)
denote the expected number of occurrences of codeword $k$ per sample. For any active codeword with $p_k>0$,
\[
\|\boldsymbol{\hat{\pi}}_k-\boldsymbol{\pi}_k^\star\|_2
=
O_p\!\left(\frac{1}{\sqrt{N p_k}}\right)
+
O_p\!\left(\frac{1}{\sqrt N}\right)
+
O_p\!\left(\frac{\varepsilon_N}{p_k}\right)
\]
Since $N_k/N \xrightarrow{p} p_k$, the leading codeword-specific statistical term scales as $\nicefrac1{\sqrt{N_k}}$.

\begin{proof}
In the following proof, $\|\cdot\|_2$ denotes the Euclidean norm for vectors, and $\|\mathbf A\|_{\mathrm{op}} = \sup_{\|x\|_2=1} \|\mathbf A x\|_2$ denotes the induced operator (spectral) norm for matrices. All parameter blocks, gradients, and Hessians are expressed in the free local coordinates of the identifiable parametrization.

\textbf{Proof strategy}. We bound the error of a fixed codeword block \(\Delta_k=\hat\pi_k-\pi_k^\star\). We first linearize the empirical first-order condition around \(\theta^\star\), and then isolate \(\Delta_k\) through the Schur complement of the empirical Hessian. This expresses \(\Delta_k\) as the sum of three contributions: a direct block-gradient term, a cross-block leakage term, and an optimization residual term. We then bound each contribution separately. In particular, codeword \(k\) appears \(N_k\) times, which yields a block-gradient of order \(O_p(\sqrt{N_k}/N)\), while the inverse effective curvature scales as \(O_p(1/p_k)\). Their product therefore gives the codeword-specific rate \(O_p(1/\sqrt{Np_k})\). The cross-block contribution is shown to be \(O_p(1/\sqrt{N})\), while approximate stationarity contributes \(O_p(\varepsilon_N/p_k)\), up to lower-order terms. Combining these three bounds gives the stated result. More precisely, these are the steps required:

\paragraph{Step 1: Linearization around $\theta^\star$.}
Let
\[
\Delta_N = \hat\theta_N-\theta^\star,
\qquad
H_N^\star = \nabla^2L_N(\theta^\star)
\]
and define
\[
\bar H_N
=
\int_0^1
\nabla^2L_N(\theta^\star+t\Delta_N)\,dt
\]
By the fundamental theorem of calculus,
\[
\nabla L_N(\hat\theta_N)-\nabla L_N(\theta^\star)
=
\bar H_N\Delta_N
\]
Adding and subtracting $H_N^\star\Delta_N$ yields
\[
H_N^\star\Delta_N
=
-\nabla L_N(\theta^\star)+r_N
\]
where
\[
r_N
=
\nabla L_N(\hat\theta_N)
+
\left(H_N^\star-\bar H_N\right)\Delta_N
\]

By (A5),
\[
\|\nabla L_N(\hat\theta_N)\|_2\leq\varepsilon_N
\]
Moreover, by the triangle inequality for the operator norm and then Assumption~(A9),
\begin{align}
\|H_N^\star-\bar H_N\|_{\mathrm{op}}
&=
\left\|
\int_0^1
\left[
\nabla^2L_N(\theta^\star)
-
\nabla^2L_N(\theta^\star+t\Delta_N)
\right]dt
\right\|_{\mathrm{op}}\\
&\leq
\int_0^1
\left\|
\nabla^2L_N(\theta^\star)
-
\nabla^2L_N(\theta^\star+t\Delta_N)
\right\|_{\mathrm{op}}
dt\\
&\leq
\int_0^1
L_H
\left\|
\theta^\star-(\theta^\star+t\Delta_N)
\right\|_2
dt\\
&=
\int_0^1
L_H\,t\|\Delta_N\|_2\,dt\\
&=
\frac{L_H}{2}\|\Delta_N\|_2
\end{align}
Therefore,
\[
\|(H_N^\star-\bar H_N)\Delta_N\|_2
\leq
\frac{L_H}{2}\|\Delta_N\|_2^2
\]
By Theorem 1,
\[
\|\Delta_N\|_2
=
O_p\left(\frac1{\sqrt{N}}\right)+O_p(\varepsilon_N)
\]
and hence
\[
\|r_N\|_2
\leq
\varepsilon_N
+
O_p\!\left(
\left(\frac1{\sqrt{N}}+\varepsilon_N\right)^2
\right)
\label{eq:reisudal_bound}
\]

\textbf{Step 2}: Codeword-frequency scaling. Consider the linearized system at $\theta^\star$ with optimization residual 

\begin{equation}
\label{eq:newton_system_residual}
\nabla^2 \mathcal L_N(\theta^\star)\,\Delta_N
=
-\nabla \mathcal L_N(\theta^\star) + r_N,
\end{equation}
and recall that 
\begin{equation*}
\Delta_N = \hat\theta_N - \theta^\star,
\qquad
\|r_N\|_2\le \varepsilon_N + O_p\!\left(\left(\frac1{\sqrt{N}}+\varepsilon_N\right)^2\right)
\end{equation*}
To simplify notation, in this paragraph block Hessians $H$ and gradients $g$ are evaluated at $\theta^\star$ and, together with residuals, are empirical unless stated otherwise. Outside of this paragraph these quantities will instead feature the $N$ subscript to signal empirical.
Now partition $\theta^\star=(\boldsymbol{\pi}_k,\theta_{-k})$ and write the gradient, residual, and Hessian
in block form:
\[
\Delta_N =
\begin{pmatrix}
\Delta_k\\
\Delta_{-k}
\end{pmatrix},
\; \;
\nabla \mathcal L_N(\theta^\star)=
\begin{pmatrix}
g_k\\
g_{-k}
\end{pmatrix},
\; \;
r_N=
\begin{pmatrix}
r_k\\
r_{-k}
\end{pmatrix},
\; \;
\nabla^2 \mathcal L_N(\theta^\star)=
\begin{pmatrix}
H_{kk} & H_{k,-k}\\
H_{-k,k} & H_{-k,-k}
\end{pmatrix}
\]
where $g_k=\nabla_{\boldsymbol{\pi}_k}\mathcal L_N(\theta^\star)$ and
$H_{kk}=\nabla^2_{\boldsymbol{\pi}_k \boldsymbol{\pi}_k}\mathcal L_N(\theta^\star)$.
Equation~\eqref{eq:newton_system_residual} is equivalent to
\begin{align}
\label{eq:block_eq1_res}
H_{kk}\Delta_k + H_{k,-k}\Delta_{-k} &= -g_k + r_k,\\
\label{eq:block_eq2_res}
H_{-k,k}\Delta_k + H_{-k,-k}\Delta_{-k} &= -g_{-k} + r_{-k}
\end{align}
Under (A4) and (A2), for large $N$ the principal submatrix $H_{-k,-k}$ is also invertible $\text{w.p.} \to 1$ and we solve~\eqref{eq:block_eq2_res} for $\Delta_{-k}$:
\begin{equation}
\label{eq:delta_minus_k_res}
\Delta_{-k}
=
- H_{-k,-k}^{-1} g_{-k}
+ H_{-k,-k}^{-1} r_{-k}
- H_{-k,-k}^{-1} H_{-k,k}\Delta_k 
\end{equation}
Substituting~\eqref{eq:delta_minus_k_res} into~\eqref{eq:block_eq1_res} yields
\[
\Big(
H_{kk} - H_{k,-k}H_{-k,-k}^{-1}H_{-k,k}
\Big)\Delta_k
=
- g_k
+ H_{k,-k}H_{-k,-k}^{-1} g_{-k}
+ r_k
- H_{k,-k}H_{-k,-k}^{-1} r_{-k}
\]
Define the empirical Schur complement and leakage term
\[
S_{k,N}=
H_{kk} - H_{k,-k}H_{-k,-k}^{-1}H_{-k,k},
\qquad
R_{k,N}=
H_{k,-k}H_{-k,-k}^{-1} g_{-k}
\]
Then
\begin{equation}
\label{eq:schur_solution_residual}
\Delta_k
=
- S_{k,N}^{-1}\Big(g_k - R_{k,N}\Big)
+
S_{k,N}^{-1}\Big(r_k - H_{k,-k}H_{-k,-k}^{-1}r_{-k}\Big)
\end{equation}
Here, for large $N$, the empirical Schur complement $S_{k,N}$ is invertible
with probability tending to one, since $S_{k,N}\to_p S_k$ by (A2) and
$S_k$ is positive definite by (A6).

\textbf{Step 3: Express in terms of the Schur complement.} 
Consider again the linearized system induced by the Hessian at $\theta^\star$:
\begin{equation}
\label{eq:start_bound}
\nabla^2\mathcal L_N(\theta^\star)\,\Delta_N
=
-\nabla\mathcal L_N(\theta^\star) + r_N,
\qquad \|r_N\|_2\le \varepsilon_N + O_p\!\left((N^{-1/2}+\varepsilon_N)^2\right)
\end{equation}
Taking the $\boldsymbol{\pi}_k$-block and solving via the Schur complement gives \eqref{eq:schur_solution_residual}:
\[
\boldsymbol{\pi}_k - \boldsymbol{\pi}_k^\star
=
- S_{k,N}^{-1}\Big(\nabla_{\boldsymbol{\pi}_k}\mathcal L_N(\theta^\star) - R_{k,N}\Big)
+ S_{k,N}^{-1}\Big(r_k - H_{k,-k}H_{-k,-k}^{-1}r_{-k}\Big)
\]
we take norms and apply the triangle inequality:
\begin{align}
\label{eq:vk_norm_bound}
\|\boldsymbol{\pi}_k - \boldsymbol{\pi}_k^\star\|_2
&\le
\|S_{k,N}^{-1}\|_2\,
\Big(
\|\nabla_{\boldsymbol{\pi}_k}\mathcal L_N(\theta^\star)\|_2
+
\|R_{k,N}\|_2
\Big)
+
\|S_{k,N}^{-1}\|_2\,\big\|r_k - H_{k,-k}H_{-k,-k}^{-1}r_{-k}\big\|_2
\end{align} 

\textbf{Step 4: Gradient scaling.}
Let $\mathcal I_k = \{(n,i) : s_n(i)=k\}$ denote the set of sample-slot pairs in
which codeword $k$ appears, with $|\mathcal I_k| = N_k$, where $N$ is the number of samples and $w$ the number of slots. The empirical gradient block with respect to $\boldsymbol{\pi}_k$ admits the decomposition
\begin{equation}
\label{eq:vk_grad}
\nabla_{\boldsymbol{\pi}_k} \mathcal L_N(\theta^\star)
=
\frac{1}{N}
\sum_{n=1}^N \sum_{i=1}^w \mathbf 1\{s_n(i)=k\}
g^{(k)}_{n,i}(\theta^\star)
\end{equation}
By (A8) and a standard Chebyshev bound for empirical averages, \begin{equation}
\label{eq:vk_grad_rate1}
\|\nabla_{\boldsymbol{\pi}_k} \mathcal L_N(\theta^\star)\|_2
=
O_p\!\left(\frac{\sqrt{N_k}}{N}\right)
\end{equation}

\textbf{Step 5: Hessian scaling.}
Define the pooled conditional Hessian contribution
\[
\bar H_k
=
\frac{1}{p_k}
\sum_{i=1}^w
\mathbb P(s(i)=k)\,
\mathbb E\!\left[
H^{(k)}_{n,i}(\theta^\star)
\mid s(i)=k
\right]
\]
where \(p_k=\sum_{i=1}^w \mathbb P(s(i)=k)\).
This definition allows the expected Hessian contribution of codeword \(k\)
to vary across slots.
The corresponding Hessian block satisfies
\begin{equation}
\label{eq:vk_hess}
\nabla^2_{\boldsymbol{\pi}_k \boldsymbol{\pi}_k} \mathcal L_N(\theta^\star)
=
\frac{1}{N}
\sum_{n=1}^N \sum_{i=1}^w \mathbf 1\{s_n(i)=k\}
H^{(k)}_{n,i}(\theta^\star)
\end{equation}
By Assumption (A8) and the conditional law of large numbers (as $N\to\infty$), the empirical Hessian becomes:
\begin{equation}
\label{eq:vk_hess_rate}
\nabla^2_{\boldsymbol{\pi}_k \boldsymbol{\pi}_k} \mathcal L_N(\theta^\star)
=
\frac{N_k}{N}\Big(\bar H_k + o_p(1)\Big)
\end{equation}
Since
\(
\frac{N_k}{N}\xrightarrow{p}p_k,
\)
for any active codeword $k$ with $p_k>0$ we have
\[
\nabla^2_{\boldsymbol{\pi}_k\boldsymbol{\pi}_k}L_N(\theta^\star)
=
p_k(\bar H_k+o_p(1))
\]
The corresponding population Hessian block is therefore
\[
\nabla^2_{\boldsymbol{\pi}_k\boldsymbol{\pi}_k}L(\theta^\star)
=
p_k\bar H_k
\]
Since $\nabla^2L(\theta^\star)$ is positive definite by Assumption (A4),
$\bar H_k\succ0$ for every active codeword. Hence
\[
\left[\nabla^2_{\boldsymbol{\pi}_k\boldsymbol{\pi}_k}L_N(\theta^\star)\right]^{-1}
=
\frac{1}{p_k}\bar H_k^{-1}
+
o_p\!\left(\frac{1}{p_k}\right)
\]
and therefore
\[
\left\|
\left[\nabla^2_{\boldsymbol{\pi}_k\boldsymbol{\pi}_k}L_N(\theta^\star)\right]^{-1}
\right\|_{\mathrm{op}}
=
O_p\!\left(\frac{1}{p_k}\right)
\]

\textbf{Step 6: Controlling the Schur inverse.}
By block-wise identifiability the (population) Schur complement
\[
S_k
=
\nabla^2_{\boldsymbol{\pi}_k \boldsymbol{\pi}_k}\mathcal L(\theta^\star)
-
\nabla^2_{\boldsymbol{\pi}_k,-k}\mathcal L(\theta^\star)\,
\big[\nabla^2_{-k,-k}\mathcal L(\theta^\star)\big]^{-1}\,
\nabla^2_{-k,\boldsymbol{\pi}_k}\mathcal L(\theta^\star)
\]
satisfies
\begin{equation}
\label{eq:decouple_domination}
S_k \succeq c\,H_{kk},
\qquad
H_{kk}=\nabla^2_{\boldsymbol{\pi}_k \boldsymbol{\pi}_k}\mathcal L(\theta^\star).
\end{equation}
Then $S_k$ is positive definite and:
\begin{equation}
\label{eq:schur_inv_bound_pop}
\|S_k^{-1}\|_2 \le \frac{1}{c}\,\|H_{kk}^{-1}\|_2
\end{equation}

(A2) implies $S_{k,N}\to_p S_k$ and $H_{kk,N}\to_p H_{kk}$ in operator norm with probability tending to one. (A4) implies $H_{-k,-k}$ is invertible, hence $H_{kk,N}$ also is, in probability. Therefore, for $N$ large enough, the inequality holds, \begin{equation}
\label{eq:schur_inv_vs_diag}
\|S_{k,N}^{-1}\|_2
\le
\frac{2}{c}\,\|H_{kk,N}^{-1}\|_2,
\qquad\text{and hence}\qquad
\|S_{k,N}^{-1}\|_2 = O_p\!\left(\|H_{kk,N}^{-1}\|_2\right)
\end{equation}

Using the result of Step 5 and Assumption (A6),
\[
\|S^{-1}_{k,N}\|_{\mathrm{op}}
=
O_p\!\left(\frac{1}{p_k}\right)
\]

\textbf{Step 7: Controlling the gradient and leakage terms.}
By~\eqref{eq:vk_grad_rate1}, the gradient contribution satisfies
\begin{equation}
\label{eq:vk_grad_rate_explicit}
\|\nabla_{\boldsymbol{\pi}_k}\mathcal L_N(\theta^\star)\|_2
=
O_p\!\left(\frac{\sqrt{N_k}}{N}\right)
\end{equation}
Since $N_k/N\xrightarrow{p}p_k$,
\[
\|\nabla_{\boldsymbol{\pi}_k}L_N(\theta^\star)\|_2
=
O_p\!\left(\sqrt{\frac{p_k}{N}}\right)
\]
Combining this with
$\|S^{-1}_{k,N}\|_{\mathrm{op}}=O_p(p_k^{-1})$
gives
\[
\|S^{-1}_{k,N}\nabla_{\boldsymbol{\pi}_k}L_N(\theta^\star)\|_2
=
O_p\!\left(\frac{1}{\sqrt{Np_k}}\right)
\]

Since $N_k/N\xrightarrow{p}p_k$,
\[
\|H_{k,-k,N}\|_{\mathrm{op}}=O_p(p_k) \label{eq:term_one}
\]
By (A4),
\[
\|H^{-1}_{-k,-k,N}\|_{\mathrm{op}}=O_p(1) \label{eq:term_two}
\]
and, as before,
\[
\|\nabla_{\theta_{-k}}L_N(\theta^\star)\|_2
=
O_p\left(\frac1{\sqrt{N}}\right)
\]
Hence
\[
\|R_{k,N}\|_2
=
O_p\!\left(\frac{p_k}{\sqrt N}\right)
\]
Therefore,
\[
\|S^{-1}_{k,N}R_{k,N}\|_2
=
O_p\!\left(\frac{1}{\sqrt N}\right)
\]

\textbf{Step 8: Controlling the residuals.}
From Step 1 (\cref{eq:reisudal_bound}),
\[
\|r_N\|_2
\le
\varepsilon_N
+
O_p\!\left(\left(\frac1{\sqrt{N}}+\varepsilon_N\right)^2\right)
\]
Moreover by \cref{eq:term_one} and \cref{eq:term_two},
\begin{align}
\|H_{k,-k,N}H^{-1}_{-k,-k,N}\|_{\mathrm{op}}
&\le
\|H_{k,-k,N}\|_{\mathrm{op}}
\|H^{-1}_{-k,-k,N}\|_{\mathrm{op}}
=
O_p(p_k)O_p(1)
=
O_p(p_k)
\end{align}
Then
\begin{align}
\Big\|
r_{k,N}
-
H_{k,-k,N}H^{-1}_{-k,-k,N}r_{-k,N}
\Big\|_2
&\le
\|r_{k,N}\|_2+
\|H_{k,-k,N}H^{-1}_{-k,-k,N}\|_{\mathrm{op}}
\|r_{-k,N}\|_2\\
&\le
\|r_N\|_2
+
O_p(p_k)\|r_N\|_2
\end{align}
Since \(k\) and \(p_k>0\) are fixed and $1+O_p(p_k)=O_p(1)$,
\[
\Big\|
r_{k,N}
-
H_{k,-k,N}H^{-1}_{-k,-k,N}r_{-k,N}
\Big\|_2
=
O_p(\|r_N\|_2)
\]
Combining:
\[
\Big\|r_{k,N}
-
H_{k,-k,N}H_{-k,-k,N}^{-1}r_{-k,N}
\Big\|_2
=
O_p(\varepsilon_N)
+
O_p\!\left(\left(\frac1{\sqrt{N}}+\varepsilon_N\right)^2\right)
\]
Using $\|S^{-1}_{k,N}\|_{\mathrm{op}}=O_p(p_k^{-1})$,
\[
\begin{aligned}
&
\left\|
S^{-1}_{k,N}
\left(
r_{k,N}
-
H_{k,-k,N}H^{-1}_{-k,-k,N}r_{-k,N}
\right)
\right\|_2 =
O_p\!\left(\frac{\varepsilon_N}{p_k}\right)
+
O_p\!\left(
\frac{(\nicefrac1{\sqrt{N}}+\varepsilon_N)^2}{p_k}
\right)
\end{aligned}
\]
Because $p_k>0$ is fixed pointwise in $k$ and $\varepsilon_N=o_p(1)$,
\[
O_p\!\left(
\frac{(N^{-1/2}+\varepsilon_N)^2}{p_k}
\right)
=
O_p\left(\frac1{\sqrt{N}}\right)
+
o_p\!\left(\frac{\varepsilon_N}{p_k}\right)
\]
and can therefore be absorbed into the leading terms.

\textbf{Step 9: Combining the bounds.}
From the Schur-complement representation in \eqref{eq:vk_norm_bound}
the estimation error decomposes into the intrinsic block-gradient term,
the cross-block leakage term, and the optimization residual term.

By Steps 6-8,
\[
\left\|
S^{-1}_{k,N}\nabla_{\boldsymbol{\pi}_k}L_N(\theta^\star)
\right\|_2
=
O_p\!\left(\frac{1}{\sqrt{Np_k}}\right), \qquad
\left\|
S^{-1}_{k,N}R_{k,N}
\right\|_2
=
O_p\!\left(\frac{1}{\sqrt N}\right)
\]
and
\[
\left\|
S^{-1}_{k,N}
\left(
r_{k,N}
-
H_{k,-k,N}H^{-1}_{-k,-k,N}r_{-k,N}
\right)
\right\|_2
=
O_p\!\left(\frac{\varepsilon_N}{p_k}\right)
+
o_p\!\left(\frac{1}{\sqrt N}
+
\frac{\varepsilon_N}{p_k}\right)
\]
Therefore, for every fixed active codeword $k$ with $p_k>0$,
\[
\|\boldsymbol{\hat{\pi}}_k-\boldsymbol{\pi}_k^\star\|_2
=
O_p\!\left(\frac{1}{\sqrt{Np_k}}\right)
+
O_p\!\left(\frac{1}{\sqrt N}\right)
+
O_p\!\left(\frac{\varepsilon_N}{p_k}\right)
\]
Finally, since
\[
\frac{N_k}{N}\xrightarrow{p}p_k,
\]
we have
\[
\frac1{\sqrt{(Np_k)}}
=
\frac{1+o_p(1)}{\sqrt{N_k}}
\]
so the leading codeword-specific statistical term scales with the
effective number of occurrences of codeword $k$. This concludes the proof.
\end{proof} 

\subsection{Proof of \texorpdfstring{\cref{cor:cell_est_err}}{Corollary 1}}
\label{app:proof_cor_1}

\begin{paragraph}{Statement:}
Under the setting of \cref{thm:convergence_rates}, further assume that the link function $\phi$ is positive, monotone and $1$-Lipschitz, \eg the softplus. Since the codebook is fixed and finite, for every fixed cell $\mcV$ with index sequence $\mathbf s$ of active codewords, the Frobenius norm of the estimation error in the local Dirichlet concentrations satisfies:
\[
\big\|
\hat{\boldsymbol\alpha}^{(\mcV)}
-
\boldsymbol\alpha^{(\mcV)\star}
\big\|_F
=
O_p\!\left(
\sum_{i=1}^w
\sigma_i^2
\left[
\frac{1}{\sqrt{N p_{s(i)}}}
+
\frac{1}{\sqrt N}
+
\frac{\varepsilon_N}{p_{s(i)}}
\right]
\right)
\]
\end{paragraph}

\begin{proof}
Throughout this proof, $\|\cdot\|_F$ denotes the Frobenius norm.
For a cell $\mcV$ with index sequence $\mathbf s$, define the pre-link
matrix
\[
\mathbf M^{(\mcV)}
=
\sum_{i=1}^w
\sigma_i^2\,
\mba_{s(i)}\mathbf b_{s(i)}^\top,
\]
so that
\[
\boldsymbol\alpha^{(\mcV)}
=
\phi\!\left(\mathbf M^{(\mcV)}\right)+\mathbf I
\]
Analogously, let $\hat{\mathbf M}^{(\mcV)}$ denote the corresponding
matrix obtained from the estimated codeword parameters.
Since $\phi$ is applied element-wise and is $1$-Lipschitz, it holds that
\[
\big\|
\hat{\boldsymbol\alpha}^{(\mcV)}
-
\boldsymbol\alpha^{(\mcV)\star}
\big\|_F
\leq
\big\|
\hat{\mathbf M}^{(\mcV)}
-
\mathbf M^{(\mcV)\star}
\big\|_F
\]
Moreover, by definition,
\[
\hat{\mathbf M}^{(\mcV)}
-
\mathbf M^{(\mcV)\star}
=
\sum_{i=1}^w
\sigma_i^2
\left(
\hat\mba_{s(i)}\hat{\mathbf b}_{s(i)}^\top
-
\mba^\star_{s(i)}\mathbf b_{s(i)}^{\star\top}
\right)
\]
Let us denote for each active codeword $k$ the difference of parameters with respect to the optimal values as
$\Delta\mba_k=\hat\mba_k-\mba_k^\star$ and
$\Delta\mathbf b_k=\hat{\mathbf b}_k-\mathbf b_k^\star$, for $\mba$ and $\mathbf b$ respectively. 
Therefore, we have that
\[
\hat\mba_k\hat{\mathbf b}_k^\top
-
\mba_k^\star\mathbf b_k^{\star\top}
=
\Delta\mba_k\mathbf b_k^{\star\top}
+
\mba_k^\star\Delta\mathbf b_k^\top
+
\Delta\mba_k\Delta\mathbf b_k^\top
\]
Note that for any vectors $\mathbf{u}, \mathbf{v}\in \mathbb R^d$, with $d\geq1$, the following holds:
$\|\mathbf u\mathbf v^\top\|_F=\|\mathbf u\|_2\|\mathbf v\|_2$.
Therefore, by applying the triangle inequality, we have that:
\[
\begin{aligned}
\big\|
\hat{\boldsymbol\alpha}^{(\mcV)}
-
\boldsymbol\alpha^{(\mcV)\star}
\big\|_F
\leq
\sum_{i=1}^w \sigma_i^2
\Big(
\|\Delta\mba_{s(i)}\|_2
  \|\mathbf b^\star_{s(i)}\|_2+
\|\mba^\star_{s(i)}\|_2
  \|\Delta\mathbf b_{s(i)}\|_2+
\|\Delta\mba_{s(i)}\|_2
  \|\Delta\mathbf b_{s(i)}\|_2
\Big)
\end{aligned}
\]

Now, let us define
\[
\rho_k
=
\frac{1}{\sqrt{Np_k}}
+
\frac{1}{\sqrt N}
+
\frac{\varepsilon_N}{p_k}
\]

Since the codebook is fixed and finite, there exists a constant
$C<\infty$ such that
\[
\|\mba_k^\star\|_2\le C,
\qquad
\|\mathbf b_k^\star\|_2\le C
\]
for every active codeword $k$. Moreover, by \cref{thm:convergence_rates},
\[
\|\Delta \mba_k\|_2=O_p(\rho_k),
\qquad
\|\Delta \mathbf b_k\|_2=O_p(\rho_k)
\]
Therefore, for each slot $i$,
\[
\|\Delta \mba_{s(i)}\|_2\|\mathbf b_{s(i)}^\star\|_2
=
O_p(\rho_{s(i)}), \qquad
\|\mba_{s(i)}^\star\|_2\|\Delta \mathbf b_{s(i)}\|_2
=
O_p(\rho_{s(i)})
\]
while
\[
\|\Delta \mba_{s(i)}\|_2
\|\Delta \mathbf b_{s(i)}\|_2
=
O_p(\rho_{s(i)}^2)
\]
Substituting these bounds into the previous inequality and using that $w$ is fixed yields
\[
\big\|
\hat{\boldsymbol\alpha}^{(\mcV)}
-
\boldsymbol\alpha^{(\mcV)\star}
\big\|_F
=
O_p\!\left(
\sum_{i=1}^w
\sigma_i^2
\big(
\rho_{s(i)}+\rho_{s(i)}^2
\big)
\right)
\]
For every active codeword, $p_k>0$ is fixed and
$\varepsilon_N=o_p(1)$, so that $\rho_k=o_p(1)$. Therefore
$\rho_k^2=o_p(\rho_k)$, and the quadratic term can be absorbed into the
leading rate. Thus,
\[
\big\|
\hat{\boldsymbol\alpha}^{(\mcV)}
-
\boldsymbol\alpha^{(\mcV)\star}
\big\|_F
=
O_p\!\left(
\sum_{i=1}^w
\sigma_i^2
\left[
\frac{1}{\sqrt{Np_{s(i)}}}
+
\frac{1}{\sqrt N}
+
\frac{\varepsilon_N}{p_{s(i)}}
\right]
\right)
\]
which proves the claim.
\end{proof}

\section{Metrics}
\label{sec:metrics}

We consider both \textit{global} and \textit{local} metrics to assess the calibration of a classifier.

\textbf{Global metrics.} Expected Calibration Error (ECE)~\citep{DBLP:conf/aaai/NaeiniCH15} is a standard metric for binary calibration, measuring the discrepancy between predicted confidence ($\text{conf}(\cdot)$) and empirical accuracy ($\text{acc}(\cdot)$) across confidence bins:
\[
    ECE = \sum_{b=1}^{m_b} \frac{|B_b|}{n} \left| \text{acc}(B_b) - \text{conf}(B_b) \right|
\]
where $B_b$ is the set of instances in the $b$-th bin and $m_b$ the number of bins.
In multiclass settings, this idea extends to Class-wise ECE \citep{DBLP:conf/nips/KullPKFSF19}, which averages per-class ECEs:
\[
\textit{Class-wise }ECE = \frac{1}{|\mcY|} \sum_{c=1}^{|\mcY|} ECE_c
\]

Following~\citet{DBLP:journals/jmlr/IbarraGTTX22}, we construct cumulative calibration residuals. In our experiments, we use a binned variant that averages the absolute cumulative residuals across bins. More precisely, we summarize residuals using their empirically weighted average absolute magnitude:
\[
\mathrm{ECCE}_c
=
\sum_{b=1}^{m_b} \frac{|B_b|}{n}
\left|
\sum_{i=1}^{b}
\frac{|B_{i,c}|}{n}
\left(
\mathrm{freq}_c(B_i)-\mathrm{conf}_c(B_i)
\right)
\right|
\]
where $\mathrm{freq}_c(B_i)$ and $\mathrm{conf}_c(B_i)$ denote, respectively,
the empirical frequency and the average predicted probability of class $c$
within bin $B_i$.

\textbf{Local metrics.} Regarding local calibration metrics, we consider the multiclass extension of \textit{Local Calibration Error}~\citep{DBLP:conf/uai/LuoBBZWXSESP22} proposed by~\citep{DBLP:journals/corr/Barbera25}: \[
    LCE = \frac1{|\mcY|}\sum_{b=1}^{m_b}
\frac{1}{n}\sum_{i\in B_b}\left\|\frac{\sum_{j\in B_b} \bigl(\hat{\mathbf{p}}_{j} - \mathbf{y}_j \bigr)\, k_\gamma(\mbx_i, \mbx_j)}
     {\sum_{j \in B_b} k_\gamma(\mbx_i, \mbx_j)}\right\|_1
    \]
where $\|\cdot\|_1$ is the $\ell^1$ norm and $k_\gamma(\mbx_i, \mbx_j)$ is a kernel function that weights the influence of neighboring points of the anchor $\mbx_i$ on its individual $LCE$ score. In practice, this metric captures the differences in the predicted probabilities and the corresponding ground truths for neighbors of an anchor point $\mbx_i$. 

Moreover, we also consider the \textit{Maximum Local Calibration Error}, defined as the worst-case per-anchor LCE residual:
    \[
    MLCE = \max_{\mbx_i \in D}\left\|\frac{\sum_{j\in B_{b(i)}} \bigl(\hat{\mathbf{p}}_{j} - \mathbf{y}_j \bigr)\, k_\gamma(\mbx_i, \mbx_j)}
     {\sum_{j \in B_{b(i)}} k_\gamma(\mbx_i, \mbx_j)}\right\|_1
    \]
Intuitively, this metric captures the largest local calibration error the ML model can make, making it particularly relevant in high-stakes settings. 

\section{Additional Experiments}
\label{app:extraRes}

\subsection{Tabular data}
\label{app:tabular}
\textbf{Setup.} We evaluate our method on the \texttt{Weather} dataset, a tabular benchmark for multi-class classification. For tabular data, we leverage the same data-splitting criterion of image data \ie $(a)$ a training set ($\approx 100k$ samples); $(b)$ we use the development set as a calibration set ($\approx 40k$ samples) with an internal $90\%/10\%$ calibration/validation split; and $(c)$ a held-out test set used exclusively for evaluation ($\approx 100k$ samples). Due to the extreme class imbalance we restrict the problem to $5$ classes which present a similar imbalance to \tissue{}. 
As the underlying predictor, we use an FT-Transformer (\texttt{FTT}) model \cite{DBLP:conf/nips/GorishniyRKB21}, a state-of-the-art neural network architecture for tabular data.

\textbf{Calibration and performance.} We report results for \texttt{Weather} data in \cref{tab:results_weather}. \VQ{} achieves the best performance in terms of local calibration, obtaining the lowest $LCE$ ($.0053$) and significantly improving over all baselines. While VQ does not achieve the best $MLCE$, it remains competitive and substantially improves over global methods, which exhibit much higher worst-case local errors. Regarding global calibration, VQ attains competitive $ECCE$ ($.0010$), slightly worse than the best-performing global methods (\SM{} and \DC{}), but still within the same order of magnitude. Finally, \VQ{} maintains comparable predictive performance ($ACC$ $\approx .603$) to all baselines.

\textbf{Under data sparsity.} For \texttt{Weather} data, we report in \cref{fig:ess_weather} local calibration error as a function of neighborhood density, measured via ESS. As expected, all methods exhibit a clear trend: calibration error is highest in low-density regions (low ESS) and decreases as the local support increases. However, VQ consistently achieves the lowest $LCE$ across the entire ESS spectrum, with the most pronounced improvements in the lowest-density bins. In particular, in the sparsest regions (leftmost bins), \VQ{} outperforms all baselines with local methods like \LN{} and \KC{} closely following. As density increases, the gap narrows, and all methods converge to similar performance, indicating that calibration is easier when sufficient local data is available. These results highlight that \VQ{} is particularly effective in mitigating calibration errors in low-support regions, where traditional methods struggle, while remaining competitive in high-density regimes. 
\begin{center}
\begin{table*}[t]
\scriptsize
\caption{Results on the \texttt{Weather} dataset with an \texttt{FT-Transformer} classifier.} \label{tab:results_weather} \begin{tabular}{c@{\hspace{7pt}} c@{\hspace{7pt}} c@{\hspace{7pt}} c@{\hspace{7pt}} c@{\hspace{7pt}} c@{\hspace{7pt}} c@{\hspace{7pt}} c}
\toprule
Data & Method & $LCE\,\downarrow$ & $MLCE\,\downarrow$ & $ECCE\,\downarrow$ & $ECE\,\downarrow$ & $ACC\,\uparrow$ & $NLL\,\downarrow$ \\
\midrule

\multirow{10}{*}{\rotatebox{90}{\texttt{weather}}}
& \cellcolor{blue!10}$\textsc{VQ}$ & \cellcolor{blue!10}$\mathbf{.0053\pm .0007}$ & \cellcolor{blue!10}$.1782\pm .0344$ & \cellcolor{blue!10}$.0010\pm .0004$ & \cellcolor{blue!10}$.0030\pm .0007$ & \cellcolor{blue!10}$.603\pm .001$ & \cellcolor{blue!10}$.943\pm .003$ \\
\cline{2-8}
& \cellcolor{gray!10}$\textsc{DC}$ & \cellcolor{gray!10}$.0221\pm .0018$ & \cellcolor{gray!10}$.4509\pm .0102$ & \cellcolor{gray!10}$\underline{.0007\pm .0001}$ & \cellcolor{gray!10}$.0022\pm .0001$ & \cellcolor{gray!10}$.607\pm .002$ & \cellcolor{gray!10}$\underline{.933\pm .003}$ \\
& $\textsc{KC}$ & $\underline{.0101\pm .0007}$ & $\mathbf{.0998\pm .0208}$ & $.0028\pm .0003$ & $.0091\pm .0005$ & $.606\pm .002$ & $.941\pm .002$ \\
& \cellcolor{gray!10}$\textsc{LN}$ & \cellcolor{gray!10}$.0107\pm .0014$ & \cellcolor{gray!10}$\underline{.1021\pm .0138}$ & \cellcolor{gray!10}$.0028\pm .0003$ & \cellcolor{gray!10}$.0100\pm .0015$ & \cellcolor{gray!10}$.605\pm .001$ & \cellcolor{gray!10}$.945\pm .004$ \\
& $\textsc{SM}$ & $.0221\pm .0018$ & $.4495\pm .0105$ & $\mathbf{.0005\pm .0001}$ & $\mathbf{.0015\pm .0002}$ & $\mathbf{.607\pm .001}$ & $\mathbf{.933\pm .003}$ \\
& \cellcolor{gray!10}$\textsc{PS}$ & \cellcolor{gray!10}$.0242\pm .0018$ & \cellcolor{gray!10}$.4379\pm .0104$ & \cellcolor{gray!10}$.0009\pm .0000$ & \cellcolor{gray!10}$.0068\pm .0003$ & \cellcolor{gray!10}$.607\pm .001$ & \cellcolor{gray!10}$.948\pm .003$ \\
& $\textsc{IR}$ & $.0223\pm .0018$ & $.4505\pm .0109$ & $.0006\pm .0001$ & $\underline{.0020\pm .0001}$ & $.606\pm .001$ & $.936\pm .003$ \\
& \cellcolor{gray!10}$\textsc{TS}$ & \cellcolor{gray!10}$.0234\pm .0018$ & \cellcolor{gray!10}$.4429\pm .0145$ & \cellcolor{gray!10}$.0019\pm .0004$ & \cellcolor{gray!10}$.0045\pm .0005$ & \cellcolor{gray!10}$.607\pm .001$ & \cellcolor{gray!10}$.936\pm .003$ \\
& $\textsc{PC}$ & $.0227\pm .0018$ & $.4471\pm .0165$ & $.0013\pm .0004$ & $.0029\pm .0002$ & $.606\pm .001$ & $.935\pm .003$ \\
& \cellcolor{gray!10}$\textsc{NC}$ & \cellcolor{gray!10}$.0226\pm .0017$ & \cellcolor{gray!10}$.4528\pm .0108$ & \cellcolor{gray!10}$.0019\pm .0005$ & \cellcolor{gray!10}$.0047\pm .0004$ & \cellcolor{gray!10}$.607\pm .001$ & \cellcolor{gray!10}$.936\pm .003$ \\

\bottomrule
\end{tabular} \end{table*}
\end{center}
\begin{figure*}[t]
\centering
        \includegraphics[width=.6\linewidth]{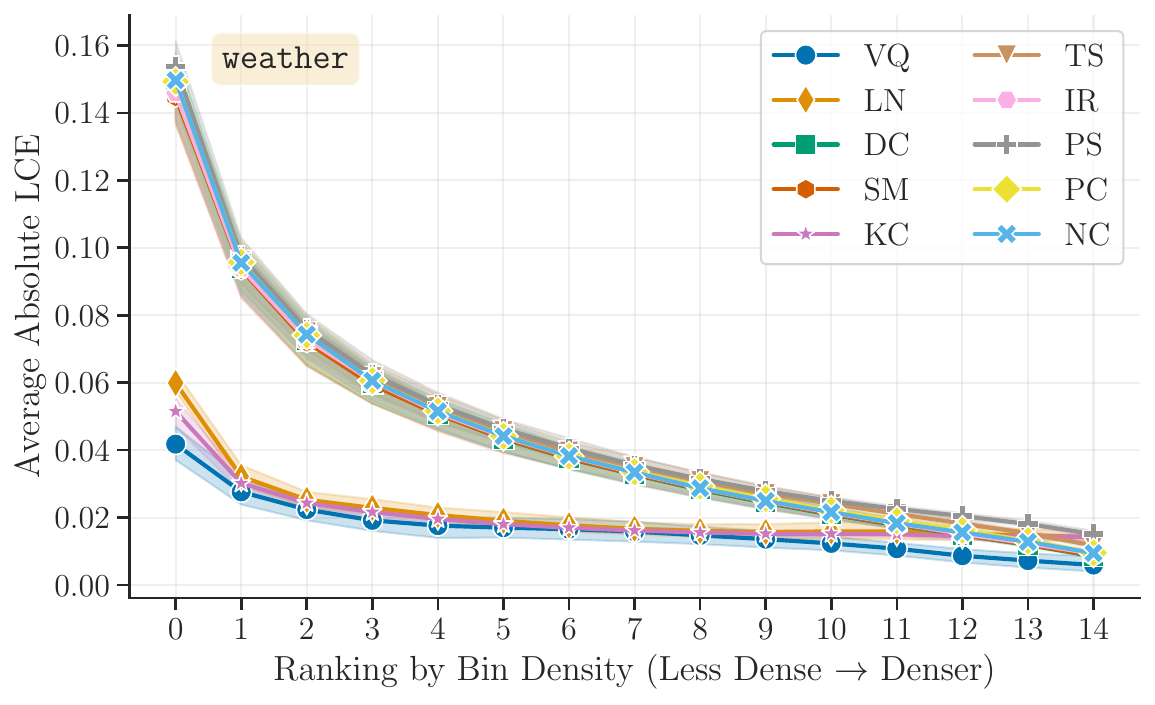}
\caption{{Local calibration in density-based sub-bins over five runs on \texttt{Weather} (FTT backbone)}.} \label{fig:ess_weather} \end{figure*}
\textbf{Out of Distribution.} The \texttt{Weather} dataset naturally supports out-of-distribution (OOD) evaluation, as it provides a dedicated test split containing data from different climatic regions. To assess robustness, we train the base model on in-distribution data and perform calibration using only in-distribution samples for all methods, then evaluate performance on the OOD split. This setting reflects a distribution shift at test time, while preserving the standard calibration protocol.

As expected, calibration performance degrades under distribution shift for all methods. Non-parametric local methods such as \KC{} or hybrid methods like \LN{}, which rely on kernel-based or neighborhood adaptations, achieve the best results in this setting, benefiting from their ability to adapt more flexibly to local test-time structure. \VQ{} remains competitive, outperforming all global calibration methods in terms of both $LCE$ and $MLCE$, but does not match the performance of its local competitors.

We emphasize that calibration approaches are designed for in-distribution settings, where the calibration and test distributions coincide. Therefore, we argue OOD experiments should be interpreted as a robustness check rather than a core requirement.
Still, we observe that \VQ{} provides a favourable trade-off, maintaining strong performance under distribution shift while retaining its advantages in the standard in-distribution setting.
More importantly, the method shows consistent improvements in local calibration across the main in-distribution experiments, ablations, and additional architectures.

\begin{center}
\begin{table*}[t]
\scriptsize
\caption{\VQ{} results on OOD \texttt{Weather} data with an \texttt{FT-Transformer} classifier.} \label{tab:results_weather_ood} \begin{tabular}{c@{\hspace{7pt}} c@{\hspace{7pt}} c@{\hspace{7pt}} c@{\hspace{7pt}} c@{\hspace{7pt}} c@{\hspace{7pt}} c@{\hspace{7pt}} c}
\toprule
Data & Method & $LCE\,\downarrow$ & $MLCE\,\downarrow$ & $ECCE\,\downarrow$ & $ECE\,\downarrow$ & $ACC\,\uparrow$ & $NLL\,\downarrow$ \\
\midrule

\multirow{10}{*}{\rotatebox{90}{\texttt{weather}}}
& \cellcolor{blue!10}$\textsc{VQ}$ 
& \cellcolor{blue!10}$.0144\pm .0017$ 
& \cellcolor{blue!10}$.3388\pm .0327$ 
& \cellcolor{blue!10}$.0038\pm .0003$ 
& \cellcolor{blue!10}$.0122\pm .0020$ 
& \cellcolor{blue!10}$.474\pm .005$ 
& \cellcolor{blue!10}$1.175\pm .011$ \\
\cline{2-8}

& \cellcolor{gray!10}$\textsc{LN}$ 
& \cellcolor{gray!10}$\mathbf{.0089\pm .0005}$ 
& \cellcolor{gray!10}$\mathbf{.2029\pm .0699}$ 
& \cellcolor{gray!10}$\mathbf{.0027\pm .0005}$ 
& \cellcolor{gray!10}$\mathbf{.0059\pm .0004}$ 
& \cellcolor{gray!10}$\underline{.478\pm .005}$ 
& \cellcolor{gray!10}$\underline{1.162\pm .008}$ \\

& $\textsc{KC}$ 
& $\underline{.0108\pm .0004}$ 
& $\underline{.2829\pm .0292}$ 
& $\underline{.0028\pm .0004}$ 
& $\underline{.0075\pm .0007}$ 
& $.477\pm .004$ 
& $\mathbf{1.159\pm .007}$ \\

& \cellcolor{gray!10}$\textsc{SM}$ 
& \cellcolor{gray!10}$.0252\pm .0018$ 
& \cellcolor{gray!10}$.4886\pm .0047$ 
& \cellcolor{gray!10}$.0034\pm .0008$ 
& \cellcolor{gray!10}$.0094\pm .0011$ 
& \cellcolor{gray!10}$\mathbf{.481\pm .005}$ 
& \cellcolor{gray!10}$1.164\pm .011$ \\

& $\textsc{DC}$ 
& $.0261\pm .0018$ 
& $.4891\pm .0047$ 
& $.0038\pm .0013$ 
& $.0108\pm .0011$ 
& $.480\pm .005$ 
& $1.167\pm .011$ \\

& \cellcolor{gray!10}$\textsc{PS}$ 
& \cellcolor{gray!10}$.0283\pm .0022$ 
& \cellcolor{gray!10}$.4742\pm .0043$ 
& \cellcolor{gray!10}$.0047\pm .0011$ 
& \cellcolor{gray!10}$.0142\pm .0012$ 
& \cellcolor{gray!10}$.480\pm .006$ 
& \cellcolor{gray!10}$1.175\pm .012$ \\

& $\textsc{TS}$ 
& $.0258\pm .0023$ 
& $.4858\pm .0048$ 
& $.0044\pm .0021$ 
& $.0087\pm .0023$ 
& $.479\pm .006$ 
& $1.163\pm .011$ \\

& \cellcolor{gray!10}$\textsc{IR}$ 
& \cellcolor{gray!10}$.0263\pm .0024$ 
& \cellcolor{gray!10}$.4898\pm .0053$ 
& \cellcolor{gray!10}$.0038\pm .0017$ 
& \cellcolor{gray!10}$.0105\pm .0017$ 
& \cellcolor{gray!10}$.479\pm .006$ 
& \cellcolor{gray!10}$1.168\pm .012$ \\

& $\textsc{PC}$ 
& $.0266\pm .0024$ 
& $.4879\pm .0038$ 
& $.0040\pm .0015$ 
& $.0111\pm .0017$ 
& $.478\pm .006$ 
& $1.169\pm .012$ \\

& \cellcolor{gray!10}$\textsc{NC}$ 
& \cellcolor{gray!10}$.0285\pm .0023$ 
& \cellcolor{gray!10}$.4898\pm .0042$ 
& \cellcolor{gray!10}$.0050\pm .0017$ 
& \cellcolor{gray!10}$.0135\pm .0017$ 
& \cellcolor{gray!10}$.478\pm .006$ 
& \cellcolor{gray!10}$1.174\pm .012$ \\

\bottomrule
\end{tabular} \end{table*}
\end{center}

\subsection{Ablating the role of quantization and calibration}
\label{sec:ablate_roles}

\begin{table*}[t]
\centering
\scriptsize
\caption{Ablation results for \VQ{} (highlighted in blue) when evaluating variants of our pipeline components.}
\label{tab:ablation}

\begin{tabular}{c c  c  c  c  c c}
\toprule
Data & Method & $LCE\,\downarrow$ & $MLCE\,\downarrow$ & $ECCE\,\downarrow$ & $ACC\,\uparrow$ & $NLL\,\downarrow$ \\
\midrule

\multirow{5}{*}{\rotatebox{90}{\texttt{cifar10}}} 
& \cellcolor{blue!10}$\textsc{VQ}$ 
& \cellcolor{blue!10}$\mathbf{.0059\pm .0002}$ 
& \cellcolor{blue!10}$\underline{.5595\pm .0276}$ 
& \cellcolor{blue!10}$\underline{.0013\pm .0002}$ 
& \cellcolor{blue!10}$\underline{.889\pm .001}$ &
\cellcolor{blue!10}$\underline{.350\pm .001}$ \\

& \cellcolor{gray!10}$\textsc{VQ-DC}$ 
& \cellcolor{gray!10}$\underline{.0062\pm .0003}$ 
& \cellcolor{gray!10}$\mathbf{.5534\pm .0311}$ 
& \cellcolor{gray!10}$.0014\pm .0004$ 
& \cellcolor{gray!10}$\underline{.889\pm .001}$ &
\cellcolor{gray!10}$\underline{.350\pm .001}$ \\

& $\textsc{VQ-NC}$ 
& $.0065\pm .0002$ 
& $.5586\pm .0253$ 
& $.0022\pm .0003$ 
& $\underline{.889\pm .001}$ &
$.351\pm .001$\\

& \cellcolor{gray!10}$\textsc{DC}$ 
& \cellcolor{gray!10}$.0099\pm .0003$ 
& \cellcolor{gray!10}$.8127\pm .0194$ 
& \cellcolor{gray!10}$\mathbf{.0011\pm .0001}$ 
& \cellcolor{gray!10}$\mathbf{.890\pm .002}$ &
\cellcolor{gray!10}$\mathbf{.332\pm .007}$ \\

& $\textsc{NC}$ 
& $.0154\pm .0005$ 
& $.8729\pm .0235$ 
& $.0066\pm .0004$ 
& $.884\pm .001$ &
$.494\pm .0018$\\

\midrule

\multirow{5}{*}{\rotatebox{90}{\texttt{cifar100}}} 
& \cellcolor{blue!10}$\textsc{VQ}$ 
& \cellcolor{blue!10}$\mathbf{.0017\pm .0001}$ 
& \cellcolor{blue!10}$\mathbf{.2932\pm .0099}$ 
& \cellcolor{blue!10}$\underline{.0008\pm .0001}$ 
& \cellcolor{blue!10}$\underline{.681\pm .001}$ &
\cellcolor{blue!10}$1.237\pm .010$\\

& \cellcolor{gray!10}$\textsc{VQ-DC}$ 
& \cellcolor{gray!10}$\underline{.0020\pm .0001}$ 
& \cellcolor{gray!10}$\underline{.3050\pm .0110}$ 
& \cellcolor{gray!10}$.0012\pm .0001$ 
& \cellcolor{gray!10}$.674\pm .004$ &
\cellcolor{gray!10}$\underline{1.210\pm .005}$\\

& $\textsc{VQ-NC}$ 
& $.0021\pm .0001$ 
& $.3199\pm .0074$ 
& $.0013\pm .0001$ 
& $\underline{.681\pm .001}$ &
$1.227\pm .007$\\

& \cellcolor{gray!10}$\textsc{DC}$ 
& \cellcolor{gray!10}$.0027\pm .0001$ 
& \cellcolor{gray!10}$.8113\pm .0121$ 
& \cellcolor{gray!10}$\mathbf{.0007\pm .0001}$ 
& \cellcolor{gray!10}$\mathbf{.690\pm .002}$ &
\cellcolor{gray!10}$\mathbf{1.154\pm .010}$ \\

& $\textsc{NC}$ 
& $.0032\pm .0001$ 
& $.9155\pm .0104$ 
& $.0017\pm .0001$ 
& $.670\pm .003$ &
$1.502\pm .036$\\
\midrule

\multirow{5}{*}{\rotatebox{90}{\tissue{}}} 
& \cellcolor{blue!10}$\textsc{VQ}$ 
& \cellcolor{blue!10}$\underline{.0088\pm .0004}$ 
& \cellcolor{blue!10}$\mathbf{.5760\pm .0349}$ 
& \cellcolor{blue!10}$\mathbf{.0017\pm .0004}$ 
& \cellcolor{blue!10}$\underline{.624\pm .002}$ &
\cellcolor{blue!10}$\mathbf{1.028\pm .007}$\\

& \cellcolor{gray!10}$\textsc{VQ-DC}$ 
& \cellcolor{gray!10}$\mathbf{.0087\pm .0005}$ 
& \cellcolor{gray!10}$\underline{.6631\pm .0401}$ 
& \cellcolor{gray!10}$\underline{.0023\pm .0005}$ 
& \cellcolor{gray!10}$\underline{.624\pm .001}$ &
\cellcolor{blue!10}$\underline{1.029\pm .007}$\\

& $\textsc{VQ-NC}$ 
& $.0112\pm .0016$ 
& $.6557\pm .0649$ 
& $.0049\pm .0017$ 
& $\mathbf{.625\pm .001}$ &
$1.033\pm .009$\\

& \cellcolor{gray!10}$\textsc{DC}$ 
& \cellcolor{gray!10}$.0284\pm .0013$ 
& \cellcolor{gray!10}$.9613\pm .0097$ 
& \cellcolor{gray!10}$.0065\pm .0009$ 
& \cellcolor{gray!10}$.615\pm .004$ &
\cellcolor{gray!10}$1.079\pm .012$ \\

& $\textsc{NC}$ 
& $.0739\pm .0018$ 
& $.9741\pm .0086$ 
& $.0415\pm .0029$ 
& $.603\pm .008$ &
$2.100\pm .010$ \\

\midrule

\multirow{4}{*}{\rotatebox{90}{\texttt{weather}}} 
& \cellcolor{blue!10}$\textsc{VQ}$ 
& \cellcolor{blue!10}$\mathbf{.0053\pm .0007}$ 
& \cellcolor{blue!10}$\underline{.1782\pm .0344}$ 
& \cellcolor{blue!10}$\underline{.0010\pm .0004}$ 
& \cellcolor{blue!10}$.603\pm .001$ &
\cellcolor{blue!10}$.943\pm .003$\\

& \cellcolor{gray!10}$\textsc{VQ-DC}$ 
& \cellcolor{gray!10}$\underline{.0056\pm .0006}$ 
& \cellcolor{gray!10}$.1806\pm .0203$ 
& \cellcolor{gray!10}$.0012\pm .0004$ 
& \cellcolor{gray!10}$.603\pm .002$ &
\cellcolor{gray!10}$.943\pm .003$\\

& $\textsc{VQ-NC}$ 
& $.0084\pm .0016$ 
& $\mathbf{.1736\pm .0156}$ 
& $.0039\pm .0015$ 
& $.601\pm .001$ &
$.948\pm .003$\\

& \cellcolor{gray!10}$\textsc{DC}$ 
& \cellcolor{gray!10}$.0221\pm .0018$ 
& \cellcolor{gray!10}$.4509\pm .0102$ 
& \cellcolor{gray!10}$\mathbf{.0007\pm .0001}$ 
& \cellcolor{gray!10}$\underline{.607\pm .002}$ &
\cellcolor{gray!10}$\mathbf{.933\pm .003}$ \\

& $\textsc{NC}$ 
& $.0226\pm .0017$ 
& $.4528\pm .0108$ 
& $.0019\pm .0005$ 
& $\mathbf{.607\pm .001}$ &
$\underline{.936\pm .003}$ \\

\bottomrule
\end{tabular} \end{table*}

We consider two main ablations of \VQ{}: $(i)$ \VQNC{} only employs vector quantization with no calibration procedure afterwards; $(ii)$ \VQDC{} does not use \VQ{}'s bilinear factorization, but employs the standard Dirichlet calibration on the quantization head. For completeness, we also include \NC{} and \DC{} in the comparison. \Cref{tab:ablation} shows the results.

For local metrics, we see that vector quantization is the main driver of improvements: all the ablated versions of our approach outperform both \NC{} and \DC{}, with impressive gains in terms of $MLCE$ (\textit{e.g.}, for \cifar{} we pass from $\approx .8248\pm .0161$ of \DC{} to $\approx .2932\pm.0099$ of \VQ{}). Still, the choice of parametrization matters: while we do not always observe statistically significant differences between \VQ{} and \VQDC{}, there are gains over~\VQNC {}.

When looking at $ECCE$, we can see that \VQ{} is better compared to both \VQDC{} and \VQNC{}, suggesting that our parametrization offers advantages in terms of global calibration. This is expected, as the bilinear parametrization regularizes calibration maps across regions. More precisely, it constrains regional differences to lie in a shared, low-dimensional space of miscalibration factors, allowing region-specific behavior while preserving global coherence. Finally, regarding $NLL$, \VQ{} achieves competitive values across datasets, with only minor gaps to globally optimized methods such as DC. Not surprisingly, the main driver of $NLL$ gains is the quantization step as the new learned classification-head optimizes this quantity explicitly.

\subsection{Ablating the number of slots and codebook size}
\label{sec:codewordsandslots}
We study the sensitivity of \VQ{} to the number of slots $w$ and the codebook size $|\mathcal C|$. Results are reported in~\cref{tab:ablationSlots} by varying $w \in \{16,32,64\text{ (default)},128,256\}$ for $|\mathcal C|\in\{32,64\}$, and in \cref{tab:ablationCodebook} by varying $|\mathcal C|$ while fixing $w=64$ (the best-performing value). 

\begin{table*}[t]
\centering
\scriptsize
\caption{\VQ{} results when varying the number of slots ($w$) and codebook size ($|\mathcal{C}|$). Blue row reports main-paper parameters.} \label{tab:ablationSlots}
\begin{tabular}{c@{\hspace{7pt}} c@{\hspace{7pt}} c@{\hspace{7pt}} c@{\hspace{7pt}} c@{\hspace{7pt}}  c@{\hspace{7pt}}  c@{\hspace{7pt}}  c@{\hspace{7pt}}  c@{\hspace{7pt}}  c}
\toprule
Data & $w$ & $|\mathcal{C}|$ &  $LCE\,\downarrow$ & $MLCE\,\downarrow$ & $ECCE\,\downarrow$ & $ECE\,\downarrow$ & $ACC\,\uparrow$ & $NLL\,\downarrow$ \\

\midrule
\midrule
 \multirow{10}{*}{\rotatebox{90}{\texttt{cifar10}}} & $16$ & $32$ &  $.0104\pm .0005$ & $.8247\pm .0058$ & $.0012\pm .0002$ & $.0044\pm .0004$ & $.878\pm .001$ & $.412\pm .005$ \\
 & \cellcolor{gray!10}$16$ & \cellcolor{gray!10}$64$ & \cellcolor{gray!10} $.0091\pm .0002$ & \cellcolor{gray!10}$.8040\pm .0141$ & \cellcolor{gray!10}$\underline{.0011\pm .0002}$ & \cellcolor{gray!10}$.0044\pm .0003$ & \cellcolor{gray!10}$.879\pm .002$ & \cellcolor{gray!10}$.422\pm .004$ \\
 & $32$ & $32$ &  $.0081\pm .0002$ & $.6852\pm .0327$ & $\mathbf{.0010\pm .0002}$ & $.0039\pm .0002$ & $.885\pm .001$ & $.374\pm .002$ \\
 & \cellcolor{gray!10}$32$ & \cellcolor{gray!10}$64$ & \cellcolor{gray!10} $.0076\pm .0003$ & \cellcolor{gray!10}$.6365\pm .0276$ & \cellcolor{gray!10}$.0013\pm .0005$ & \cellcolor{gray!10}$.0043\pm .0003$ & \cellcolor{gray!10}$.884\pm .001$ & \cellcolor{gray!10}$.380\pm .004$ \\
 & $64$ & $32$ &  $.0065\pm .0002$ & $\underline{.5665\pm .0163}$ & $.0012\pm .0001$ & $.0038\pm .0002$ & $.889\pm .001$ & $.352\pm .001$ \\
 & \cellcolor{blue!10}$64$ & \cellcolor{blue!10}$64$ & \cellcolor{blue!10} $\mathbf{.0059\pm .0002}$ & \cellcolor{blue!10}$\mathbf{.5595\pm .0276}$ & \cellcolor{blue!10}$.0013\pm .0002$ & \cellcolor{blue!10}$\underline{.0037\pm .0002}$ & \cellcolor{blue!10}$.889\pm .001$ & \cellcolor{blue!10}$.348\pm .002$ \\
 & $128$ & $32$ &  $\underline{.0063\pm .0002}$ & $.6365\pm .0363$ & $.0012\pm .0001$ & $.0039\pm .0003$ & $.890\pm .002$ & $.341\pm .003$ \\
 & \cellcolor{gray!10}$128$ & \cellcolor{gray!10}$64$ & \cellcolor{gray!10} $.0066\pm .0001$ & \cellcolor{gray!10}$.6552\pm .0317$ & \cellcolor{gray!10}$.0014\pm .0004$ & \cellcolor{gray!10}$.0039\pm .0004$ & \cellcolor{gray!10}$.889\pm .001$ & \cellcolor{gray!10}$.341\pm .003$ \\
 & $256$ & $32$ &  $.0075\pm .0001$ & $.7306\pm .0241$ & $.0011\pm .0001$ & $\mathbf{.0036\pm .0002}$ & $\mathbf{.892\pm .001}$ & $\mathbf{.333\pm .003}$ \\
 & \cellcolor{gray!10}$256$ & \cellcolor{gray!10}$64$ & \cellcolor{gray!10} $.0082\pm .0003$ & \cellcolor{gray!10}$.7290\pm .0298$ & \cellcolor{gray!10}$.0016\pm .0003$ & \cellcolor{gray!10}$.0040\pm .0004$ & \cellcolor{gray!10}$\underline{.891\pm .001}$ & \cellcolor{gray!10}$\underline{.334\pm .002}$ \\
\midrule
 \multirow{10}{*}{\rotatebox{90}{\texttt{cifar100}}} & $16$ & $32$ &  $.0032\pm .0002$ & $.7790\pm .0337$ & $.0007\pm .0001$ & $.0017\pm .0001$ & $.638\pm .003$ & $1.428\pm .020$ \\
 & \cellcolor{gray!10}$16$ & \cellcolor{gray!10}$64$ & \cellcolor{gray!10} $.0023\pm .0001$ & \cellcolor{gray!10}$.6163\pm .0245$ & \cellcolor{gray!10}$.0008\pm .0001$ & \cellcolor{gray!10}$.0018\pm .0001$ & \cellcolor{gray!10}$.645\pm .002$ & \cellcolor{gray!10}$1.395\pm .006$ \\
 & $32$ & $32$ &  $.0023\pm .0005$ & $.6206\pm .1806$ & $.0008\pm .0001$ & $.0017\pm .0001$ & $.666\pm .002$ & $1.289\pm .006$ \\
 & \cellcolor{gray!10}$32$ & \cellcolor{gray!10}$64$ & \cellcolor{gray!10} $.0017\pm .0002$ & \cellcolor{gray!10}$\mathbf{.2758\pm .0104}$ & \cellcolor{gray!10}$.0009\pm .0001$ & \cellcolor{gray!10}$.0018\pm .0001$ & \cellcolor{gray!10}$.670\pm .001$ & \cellcolor{gray!10}$1.276\pm .007$ \\
 & $64$ & $32$ &  $\mathbf{.0016\pm .0001}$ & $.3274\pm .0133$ & $.0007\pm .0001$ & $.0017\pm .0001$ & $.680\pm .003$ & $1.213\pm .004$ \\
 & \cellcolor{blue!10}$64$ & \cellcolor{blue!10}$64$ & \cellcolor{blue!10} $.0017\pm .0001$ & \cellcolor{blue!10}$\underline{.2932\pm .0099}$ & \cellcolor{blue!10}$.0008\pm .0001$ & \cellcolor{blue!10}$.0018\pm .0001$ & \cellcolor{blue!10}$.681\pm .001$ & \cellcolor{blue!10}$1.208\pm .007$ \\
 & $128$ & $32$ &  $\underline{.0017\pm .0001}$ & $.3882\pm .0104$ & $\underline{.0006\pm .0001}$ & $\underline{.0016\pm .0001}$ & $.688\pm .001$ & $1.167\pm .008$ \\
 & \cellcolor{gray!10}$128$ & \cellcolor{gray!10}$64$ & \cellcolor{gray!10} $.0018\pm .0001$ & \cellcolor{gray!10}$.4280\pm .0208$ & \cellcolor{gray!10}$.0007\pm .0001$ & \cellcolor{gray!10}$.0017\pm .0001$ & \cellcolor{gray!10}$.688\pm .001$ & \cellcolor{gray!10}$1.159\pm .006$ \\
 & $256$ & $32$ &  $.0019\pm .0001$ & $.5717\pm .0054$ & $\mathbf{.0005\pm .0001}$ & $\mathbf{.0015\pm .0001}$ & $\mathbf{.692\pm .001}$ & $\underline{1.137\pm .004}$ \\
 & \cellcolor{gray!10}$256$ & \cellcolor{gray!10}$64$ & \cellcolor{gray!10} $.0021\pm .0001$ & \cellcolor{gray!10}$.6285\pm .0092$ & \cellcolor{gray!10}$.0006\pm .0001$ & \cellcolor{gray!10}$.0016\pm .0001$ & \cellcolor{gray!10}$\underline{.692\pm .001}$ & \cellcolor{gray!10}$\mathbf{1.135\pm .005}$ \\
\midrule
\multirow{10}{*}{\rotatebox{90}
{\tissue{}}} & $16$ & $32$ 
& $.0178\pm .0011$ & $.8854\pm .0248$ & $.0011\pm .0004$ & $.0043\pm .0003$ & $.589\pm .005$ & $1.117\pm .012$ \\
 
& \cellcolor{gray!10}$16$ & \cellcolor{gray!10}$64$ 
& \cellcolor{gray!10}$.0147\pm .0006$ & \cellcolor{gray!10}$.8420\pm .0424$ & \cellcolor{gray!10}$\underline{.0010\pm .0002}$ & \cellcolor{gray!10}$.0042\pm .0003$ & \cellcolor{gray!10}$.598\pm .002$ & \cellcolor{gray!10}$1.105\pm .004$ \\

& $32$ & $32$ 
& $.0113\pm .0005$ & $.6746\pm .0607$ & $.0011\pm .0004$ & $\underline{.0037\pm .0005}$ & $.616\pm .003$ & $1.051\pm .007$ \\

& \cellcolor{gray!10}$32$ & \cellcolor{gray!10}$64$ 
& \cellcolor{gray!10}$.0097\pm .0005$ & \cellcolor{gray!10}$\underline{.6498\pm .0528}$ & \cellcolor{gray!10}$\mathbf{.0010\pm .0003}$ & \cellcolor{gray!10}$\mathbf{.0036\pm .0002}$ & \cellcolor{gray!10}$.618\pm .003$ & \cellcolor{gray!10}$1.047\pm .008$ \\

& $64$ & $32$ 
& $\mathbf{.0082\pm .0003}$ & $\mathbf{.5091\pm .0653}$ & $.0012\pm .0002$ & $.0038\pm .0003$ & $.623\pm .003$ & $1.032\pm .007$ \\

& \cellcolor{blue!10}$64$ & \cellcolor{blue!10}$64$ 
& \cellcolor{blue!10}$\underline{.0083\pm .0002}$ & \cellcolor{blue!10}$.6984\pm .0368$ & \cellcolor{blue!10}$.0017\pm .0004$ & \cellcolor{blue!10}$.0043\pm .0002$ & \cellcolor{blue!10}$.624\pm .002$ & \cellcolor{blue!10}$1.028\pm .007$ \\

& $128$ & $32$ 
& $.0095\pm .0004$ & $.7525\pm .0382$ & $.0016\pm .0007$ & $.0040\pm .0007$ & $.628\pm .002$ & $1.018\pm .006$ \\

& \cellcolor{gray!10}$128$ & \cellcolor{gray!10}$64$ 
& \cellcolor{gray!10}$.0105\pm .0002$ & \cellcolor{gray!10}$.8273\pm .0244$ & \cellcolor{gray!10}$.0017\pm .0003$ & \cellcolor{gray!10}$.0043\pm .0003$ & \cellcolor{gray!10}$.628\pm .002$ & \cellcolor{gray!10}$1.014\pm .005$ \\

& $256$ & $32$ 
& $.0133\pm .0002$ & $.8543\pm .0203$ & $.0021\pm .0003$ & $.0046\pm .0006$ & $\underline{.632\pm .001}$ & $\underline{1.006\pm .005}$ \\

& \cellcolor{gray!10}$256$ & \cellcolor{gray!10}$64$ 
& \cellcolor{gray!10}$.0153\pm .0004$ & \cellcolor{gray!10}$.8799\pm .0146$ & \cellcolor{gray!10}$.0021\pm .0004$ & \cellcolor{gray!10}$.0048\pm .0005$ & \cellcolor{gray!10}$\mathbf{.632\pm .002}$ & \cellcolor{gray!10}$\mathbf{1.004\pm .005}$ \\
\bottomrule
\end{tabular} \end{table*}

Across datasets, performance exhibits a clear U-shape trend for $w$: small values (\eg $w=16$) provide too coarse a partition of the representation space and limit the benefit of locality, while very large values (\eg $w \ge 128$) make the region assignment and the resulting local estimates less stable. Consistently, the strongest local calibration performance is achieved for intermediate values, with $w\in\{32,64\}$ yielding the best or comparable results. Finally, \cref{tab:ablationCodebook} indicates that the effect of $|\mathcal C|$ is secondary once $w$ is fixed: once again we observe that the best results are achieved at $|\mathcal{C}|\in\{32,64\}$. 

\begin{table*}[h]
\centering
\scriptsize
\caption{\VQ{} results when changing codebook size ($|\mathcal{C}|$) for a fixed number of slots ($w=64$). Blue row reports main-paper parameters.} \label{tab:ablationCodebook} \begin{tabular}{c@{\hspace{7pt}} c@{\hspace{7pt}} c@{\hspace{7pt}}  c@{\hspace{7pt}}  c@{\hspace{7pt}}  c@{\hspace{7pt}}  c@{\hspace{7pt}}  c@{\hspace{7pt}}  c}
\toprule
Data & $w$ & $|\mathcal{C}|$ &  $LCE\,\downarrow$ & $MLCE\,\downarrow$ & $ECCE\,\downarrow$ & $ECE\,\downarrow$ & $ACC\,\uparrow$ & $NLL\,\downarrow$ \\

\midrule
\midrule
 \multirow{5}{*}{\rotatebox{90}{\texttt{cifar10}}} & $64$ & $16$ &  $.0082\pm .0001$ & $.6662\pm .0122$ & $\mathbf{.0011\pm .0002}$ & $\mathbf{.0037\pm .0003}$ & $.887\pm .001$ & $.359\pm .003$ \\
 & \cellcolor{gray!10}$64$ & \cellcolor{gray!10}$32$ & \cellcolor{gray!10} $.0065\pm .0002$ & \cellcolor{gray!10}$.5665\pm .0173$ & \cellcolor{gray!10}$.0012\pm .0001$ & \cellcolor{gray!10}$.0038\pm .0002$ & \cellcolor{gray!10}$\underline{.889\pm .001}$ & \cellcolor{gray!10}$\underline{.352\pm .001}$ \\
 & \cellcolor{blue!10}$64$ & \cellcolor{blue!10}$64$ & \cellcolor{blue!10} $\mathbf{.0059\pm .0002}$ & \cellcolor{blue!10}$\underline{.5595\pm .0276}$ & \cellcolor{blue!10}$.0013\pm .0002$ & \cellcolor{blue!10}$\underline{.0037\pm .0002}$ & \cellcolor{blue!10}$\mathbf{.889\pm .001}$ & \cellcolor{blue!10}$\mathbf{.348\pm .002}$ \\
 & \cellcolor{gray!10}$64$ & \cellcolor{gray!10}$128$ & \cellcolor{gray!10} $.0061\pm .0002$ & \cellcolor{gray!10}$.5650\pm .0478$ & \cellcolor{gray!10}$.0014\pm .0003$ & \cellcolor{gray!10}$.0040\pm .0005$ & \cellcolor{gray!10}$.888\pm .001$ & \cellcolor{gray!10}$.352\pm .004$ \\
 & $64$ & $256$ &  $\underline{.0060\pm .0001}$ & $\mathbf{.5358\pm .0218}$ & $\underline{.0011\pm .0003}$ & $.0038\pm .0004$ & $.888\pm .001$ & $.354\pm .004$ \\
\midrule
 \multirow{5}{*}{\rotatebox{90}{\texttt{cifar100}}} & $64$ & $16$ &  $.0023\pm .0002$ & $.6288\pm .0573$ & $\mathbf{.0006\pm .0001}$ & $\mathbf{.0016\pm .0001}$ & $.677\pm .002$ & $1.221\pm .004$ \\
 & \cellcolor{gray!10}$64$ & \cellcolor{gray!10}$32$ & \cellcolor{gray!10} $\mathbf{.0016\pm .0001}$ & \cellcolor{gray!10}$\underline{.3274\pm .0141}$ & \cellcolor{gray!10}$\underline{.0007\pm .0001}$ & \cellcolor{gray!10}$\underline{.0017\pm .0001}$ & \cellcolor{gray!10}$.680\pm .003$ & \cellcolor{gray!10}$1.213\pm .004$ \\
 & \cellcolor{blue!10}$64$ & \cellcolor{blue!10}$64$ & \cellcolor{blue!10} $\underline{.0017\pm .0001}$ & \cellcolor{blue!10}$\mathbf{.2932\pm .0099}$ & \cellcolor{blue!10}$.0008\pm .0001$ & \cellcolor{blue!10}$.0018\pm .0001$ & \cellcolor{blue!10}$\underline{.681\pm .001}$ & \cellcolor{blue!10}$\underline{1.208\pm .007}$ \\
 & \cellcolor{gray!10}$64$ & \cellcolor{gray!10}$128$ & \cellcolor{gray!10} $.0017\pm .0002$ & \cellcolor{gray!10}$.3299\pm .0068$ & \cellcolor{gray!10}$.0008\pm .0001$ & \cellcolor{gray!10}$.0018\pm .0001$ & \cellcolor{gray!10}$.681\pm .002$ & \cellcolor{gray!10}$1.212\pm .011$ \\
 & $64$ & $256$ &  $.0017\pm .0001$ & $.3546\pm .0127$ & $.0008\pm .0001$ & $.0018\pm .0001$ & $\mathbf{.682\pm .001}$ & $\mathbf{1.204\pm .006}$ \\
\midrule
\multirow{5}{*}{\rotatebox{90}{\tissue{}}}
& $64$ & $16$ 
& $.0115\pm .0005$ 
& $.5393\pm .0397$ 
& $\underline{.0012\pm .0004}$ 
& $.0039\pm .0005$ 
& $.621 \pm .003$ 
& $1.039 \pm 0.007$ \\
 
& \cellcolor{gray!10}$64$ & \cellcolor{gray!10}$32$ 
& \cellcolor{gray!10}$\mathbf{.0082\pm .0003}$ 
& \cellcolor{gray!10}$\mathbf{.5091\pm .0653}$ 
& \cellcolor{gray!10}$\mathbf{.0012\pm .0002}$ 
& \cellcolor{gray!10}$\mathbf{.0038\pm .0003}$ 
& \cellcolor{gray!10}$.623 \pm .003$ 
& \cellcolor{gray!10}$1.032 \pm 0.007$ \\

& \cellcolor{blue!10}$64$ & \cellcolor{blue!10}$64$ 
& \cellcolor{blue!10}$\underline{.0083\pm .0002}$ 
& \cellcolor{blue!10}$\underline{.6984\pm .0368}$ 
& \cellcolor{blue!10}$.0017\pm .0004$ 
& \cellcolor{blue!10}$.0043\pm .0002$ 
& \cellcolor{blue!10}$.624 \pm .002$ 
& \cellcolor{blue!10}$1.028 \pm 0.007$ \\

& \cellcolor{gray!10}$64$ & \cellcolor{gray!10}$128$ 
& \cellcolor{gray!10}$.0086\pm .0003$ 
& \cellcolor{gray!10}$.7504\pm .0304$ 
& \cellcolor{gray!10}$.0015\pm .0004$ 
& \cellcolor{gray!10}$\underline{.0038\pm .0004}$ 
& \cellcolor{gray!10}$\underline{.625 \pm .002}$ 
& \cellcolor{gray!10}$\underline{1.026 \pm 0.006}$ \\

& $64$ & $256$ 
& $.0088\pm .0004$ 
& $.7341\pm .0508$ 
& $.0013\pm .0004$ 
& ${.0038\pm .0007}$ 
& $\mathbf{.625\pm .003}$ 
& $\mathbf{1.024\pm .007}$ \\
\bottomrule
\end{tabular} \end{table*}

\subsection{Ablating the calibration-set size}
\label{sec:ablate_calsize}
We report the effect of calibration set size on local calibration ($LCE$) and negative log-likelihood ($NLL$) for \tissue{} in \cref{fig:calsize_tissue}. As expected, all methods benefit from increased calibration data, with both LCE and NLL generally decreasing as the calibration set grows with the sole exception of \LN{} which shows rather unstable performance. However, \VQ{} consistently outperforms all baselines across the entire data regime. Notably, \VQ{} achieves strong performance even in the low-data setting, where it already attains the lowest $LCE$ and competitive $NLL$, indicating superior data efficiency. As the calibration set size increases, \VQ{} continues to improve and maintains a clear margin in $LCE$, while remaining among the best methods in terms of $NLL$. In contrast, several baselines exhibit slower improvements or plateau early, particularly in low-data regimes. These results suggest that VQ is not only more accurate in well-sampled settings but is also substantially more robust under data scarcity, highlighting the effectiveness of parameter sharing in leveraging limited calibration data.
\begin{center}
\begin{figure*}[t]
        \includegraphics[width=\linewidth]{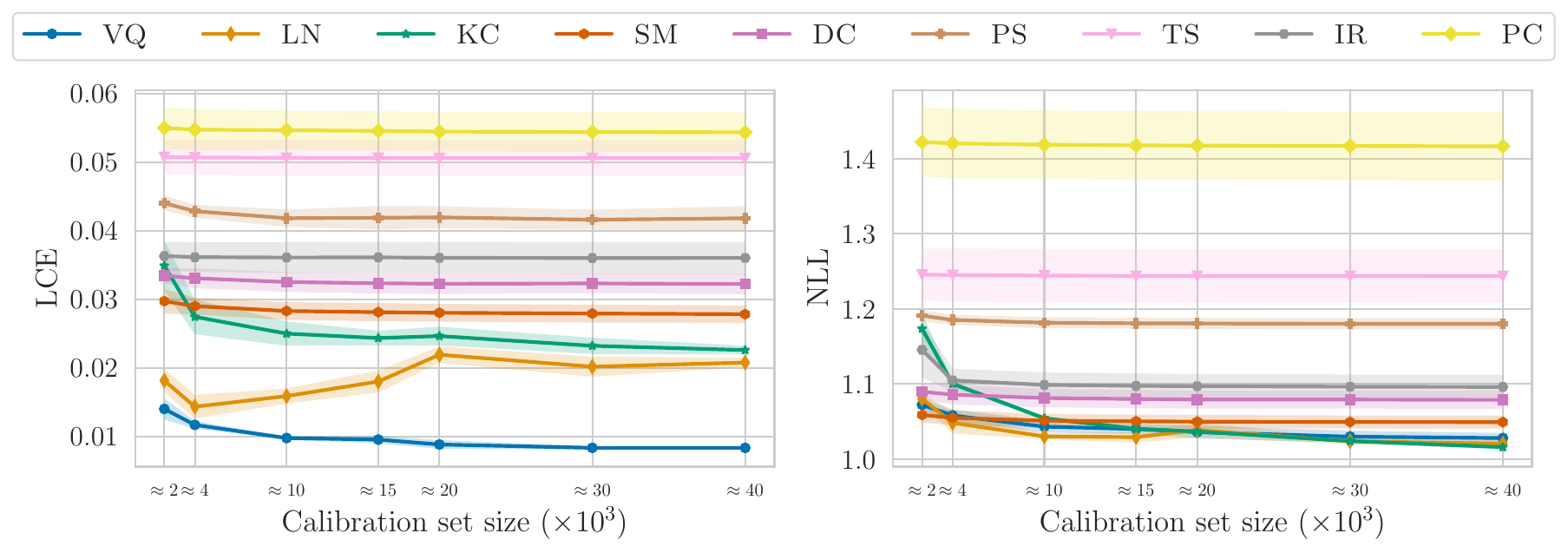}
\caption{{$LCE$ (left) and $NLL$ (right) for various calibration-set sizes on \texttt{TissueMNIST}.}} \label{fig:calsize_tissue}
\end{figure*}
\end{center}

We report the effect of calibration set size on local calibration ($LCE$) and negative log-likelihood ($NLL$) for \texttt{Weather} in \cref{fig:calsize_weather}. Consistent with the previous experiment, all methods generally benefit from increased calibration data, with both $LCE$ and $NLL$ decreasing as the calibration set grows. However, \VQ{} consistently achieves the lowest $LCE$ across all calibration sizes, with a clear margin over competing methods. Notably, \VQ{} exhibits strong data efficiency, attaining low $LCE$ even in the smallest data regime and continuing to improve steadily as more calibration data becomes available. In terms of $NLL$, all methods converge to similar performance as the calibration set increases, with only minor differences in the large-data regime. While \VQ{} is not always the best in $NLL$ for very small calibration sizes, it quickly becomes competitive and remains among the top-performing methods as more data is available. Overall, these results reinforce that \VQ{}'s primary advantage lies in robust and data-efficient local calibration, particularly in low-data settings, while maintaining competitive negative log-likelihood performance.

\begin{center}
\begin{figure*}[t]
        \includegraphics[width=\linewidth]{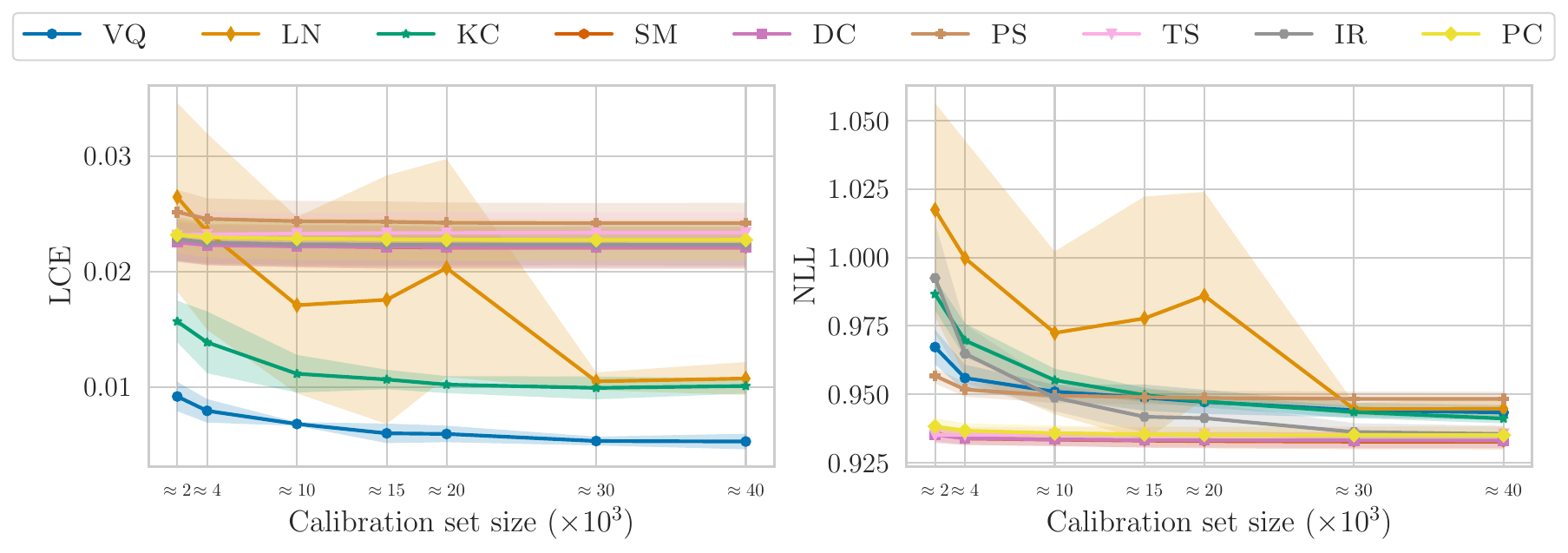}
\caption{{$LCE$ (left) and $NLL$ (right) for various calibration-set sizes on \texttt{Weather}.}} \label{fig:calsize_weather}
\end{figure*}
\end{center}
\subsection{Additional metrics} We additionally evaluate our method and all baselines according to expected calibration error ($ECE$) and report accuracy ($ACC$) (\cref{tab:allResults}).

Across all three datasets, \VQ{} consistently attains very strong results in terms of $ECE$. It ties for first place with \DC{} and \SM{} on \cifarTen{}, is best on \cifar{} and second best on \tissue{}. This result further validates the effectiveness of our approach in jointly achieving strong local and global calibration. For $ACC$, \VQ{} yields improvements over the \NC{} baseline. Overall, its performance is competitive with the other approaches, though it never attains the best individual score.

\begin{table*}[t]
\centering
\scriptsize
\caption{Results for our approach (highlighted in blue) vs baselines (\texttt{ResNet} backbone).}
\label{tab:allResults}
    \begin{tabular}{c@{\hspace{7pt}} c@{\hspace{7pt}}  c@{\hspace{7pt}}  c@{\hspace{7pt}}  c@{\hspace{7pt}}  c@{\hspace{7pt}}  c@{\hspace{7pt}}  c}
\toprule
Data & Method & $LCE\,\downarrow$ & $MLCE\,\downarrow$ & $ECCE\,\downarrow$ & $ECE\,\downarrow$ & $ACC\,\uparrow$ & $NLL\,\downarrow$ \\
\midrule
\multirow{10}{*}{\rotatebox{90}{\texttt{cifar10}}} & \cellcolor{blue!10}$\textsc{VQ}$ & \cellcolor{blue!10}$\mathbf{.0059\pm .0002}$ & \cellcolor{blue!10}$\mathbf{.5595\pm .0276}$ & \cellcolor{blue!10}$.0013\pm .0002$ & \cellcolor{blue!10}$.0037\pm .0002$ & \cellcolor{blue!10}$.889\pm .001$ & \cellcolor{blue!10}$.348\pm .002$ \\

& $\textsc{SM}$ & $.0102\pm .0003$ & $.8027\pm .0166$ & $\mathbf{.0008\pm .0001}$ & $\mathbf{.0037\pm .0001}$ & $\underline{.890\pm .002}$ & $\underline{.333\pm .007}$ \\
& \cellcolor{gray!10}$\textsc{DC}$ 
& \cellcolor{gray!10}$.0099\pm .0003$ 
& \cellcolor{gray!10}$.8127\pm .0194$ 
& \cellcolor{gray!10}$\underline{.0011\pm .0001}$ 
& \cellcolor{gray!10}$\underline{.0037\pm .0001}$ 
& \cellcolor{gray!10}$.890\pm .002$ 
& \cellcolor{gray!10}$\mathbf{.332\pm .007}$ \\
 & $\textsc{KC}$ & $.0095\pm .0004$ & $.7923\pm .0241$ & $.0015\pm .0003$ & $.0054\pm .0004$ & $\mathbf{.893\pm .001}$ & $.340\pm .002$ \\
 & \cellcolor{gray!10}$\textsc{LN}$ & \cellcolor{gray!10}$\underline{.0079\pm .0003}$ & \cellcolor{gray!10}$\underline{.6279\pm .0157}$ & \cellcolor{gray!10}$.0017\pm .0001$ & \cellcolor{gray!10}$.0048\pm .0003$ & \cellcolor{gray!10}$.888\pm .002$ & \cellcolor{gray!10}$.347\pm .002$ \\
 & $\textsc{PS}$ & $.0180\pm .0005$ & $.8542\pm .0224$ & $.0014\pm .0001$ & $.0085\pm .0003$ & $.884\pm .001$ & $.466\pm .004$ \\
 & \cellcolor{gray!10}$\textsc{IR}$ & \cellcolor{gray!10}$.0115\pm .0003$ & \cellcolor{gray!10}$.7978\pm .0162$ & \cellcolor{gray!10}$.0014\pm .0002$ & \cellcolor{gray!10}$.0047\pm .0001$ & \cellcolor{gray!10}$.884\pm .001$ & \cellcolor{gray!10}$.364\pm .008$ \\
 & $\textsc{TS}$ & $.0127\pm .0006$ & $.8200\pm .0191$ & $.0036\pm .0008$ & $.0072\pm .0010$ & $.884\pm .001$ & $.362\pm .008$ \\
 & \cellcolor{gray!10}$\textsc{PC}$ & \cellcolor{gray!10}$.0133\pm .0006$ & \cellcolor{gray!10}$.8339\pm .0456$ & \cellcolor{gray!10}$.0039\pm .0008$ & \cellcolor{gray!10}$.0085\pm .0009$ & \cellcolor{gray!10}$.880\pm .002$ & \cellcolor{gray!10}$.377\pm .008$ \\
& $\textsc{NC}$ 
& $.0154\pm .0005$ 
& $.8729\pm .0235$ 
& $.0066\pm .0004$ 
& $.0153\pm .0006$ 
& $.8838\pm .0014$ 
& $.4943\pm .0180$ \\

\midrule
\multirow{10}{*}{\rotatebox{90}{\texttt{cifar100}}} & \cellcolor{blue!10}$\textsc{VQ}$ & \cellcolor{blue!10}$\mathbf{.0017\pm .0001}$ & \cellcolor{blue!10}$\mathbf{.2932\pm .0099}$ & \cellcolor{blue!10}$.0008\pm .0001$ & \cellcolor{blue!10}$\mathbf{.0018\pm .0001}$ & \cellcolor{blue!10}$.681\pm .001$ & \cellcolor{blue!10}$1.208\pm .007$ \\

& $\textsc{SM}$ & $.0031\pm .0002$ & $.8121\pm .0190$ & $.0008\pm .0006$ & $.0024\pm .0015$ & $.677\pm .013$ & $1.225\pm .043$ \\
& \cellcolor{gray!10}$\textsc{DC}$ 
& \cellcolor{gray!10}$.0027\pm .0001$ 
& \cellcolor{gray!10}$.8113\pm .0121$ 
& \cellcolor{gray!10}$\mathbf{.0007\pm .0001}$ 
& \cellcolor{gray!10}$.0019\pm .0001$ 
& \cellcolor{gray!10}$\mathbf{.690\pm .002}$ 
& \cellcolor{gray!10}$\underline{1.154\pm .009}$ \\
 & $\textsc{KC}$ & $.0043\pm .0001$ & $.7575\pm .0101$ & $.0013\pm .0001$ & $.0044\pm .0001$ & $.679\pm .003$ & $1.352\pm .009$ \\
 & \cellcolor{gray!10}$\textsc{LN}$ & \cellcolor{gray!10}$\underline{.0024\pm .0001}$ & \cellcolor{gray!10}$\underline{.7022\pm .0070}$ & \cellcolor{gray!10}$\mathbf{.0007\pm .0001}$ & \cellcolor{gray!10}$\underline{.0018\pm .0001}$ & \cellcolor{gray!10}$\underline{.688\pm .001}$ & \cellcolor{gray!10}$\mathbf{1.125\pm .002}$ \\
 & $\textsc{PS}$ & $.0053\pm .0001$ & $.7031\pm .0409$ & $.0011\pm .0001$ & $.0033\pm .0001$ & $.670\pm .002$ & $1.618\pm .007$ \\
 & \cellcolor{gray!10}$\textsc{IR}$ & \cellcolor{gray!10}$.0031\pm .0001$ & \cellcolor{gray!10}$.8049\pm .0116$ & \cellcolor{gray!10}$\mathbf{.0007\pm .0001}$ & \cellcolor{gray!10}$.0020\pm .0001$ & \cellcolor{gray!10}$.670\pm .002$ & \cellcolor{gray!10}$1.437\pm .029$ \\
 & $\textsc{TS}$ & $.0034\pm .0001$ & $.8239\pm .0173$ & $.0013\pm .0001$ & $.0024\pm .0001$ & $.670\pm .002$ & $1.277\pm .009$ \\
 & \cellcolor{gray!10}$\textsc{PC}$ & \cellcolor{gray!10}$.0041\pm .0002$ & \cellcolor{gray!10}$.8187\pm .0127$ & \cellcolor{gray!10}$.0013\pm .0001$ & \cellcolor{gray!10}$.0026\pm .0001$ & \cellcolor{gray!10}$.662\pm .003$ & \cellcolor{gray!10}$1.272\pm .012$ \\
 & $\textsc{NC}$ & $.0032\pm .0001$ & $.9155\pm .0104$ & $.0017\pm .0001$ & $.0039\pm .0001$ & $.670\pm .002$ & $1.502\pm .035$ \\

\midrule
\multirow{10}{*}{\rotatebox{90}{\tissue{}}} 
& \cellcolor{blue!10}$\textsc{VQ}$ 
& \cellcolor{blue!10}$\mathbf{.0083\pm .0002}$ 
& \cellcolor{blue!10}$\underline{.6984\pm .0368}$ 
& \cellcolor{blue!10}$.0017\pm .0004$ 
& \cellcolor{blue!10}$\underline{.0043\pm .0002}$ 
& \cellcolor{blue!10}$.624\pm .002$ 
& \cellcolor{blue!10}$1.028\pm .007$ \\

& $\textsc{SM}$ 
& $.0278\pm .0013$ 
& $.9609\pm .0106$ 
& $\mathbf{.0008\pm .0001}$ 
& $\mathbf{.0037\pm .0001}$ 
& $.618\pm .003$ 
& $1.049\pm .009$ \\

& \cellcolor{gray!10}$\textsc{DC}$ 
& \cellcolor{gray!10}$.0322\pm .0015$ 
& \cellcolor{gray!10}$.9706\pm .0098$ 
& \cellcolor{gray!10}$.0065\pm .0009$ 
& \cellcolor{gray!10}$.0183\pm .0010$ 
& \cellcolor{gray!10}$.615\pm .004$ 
& \cellcolor{gray!10}$1.079\pm .012$ \\

& $\textsc{KC}$ 
& $.0226\pm .0006$ 
& $.7296\pm .0570$ 
& $.0061\pm .0004$ 
& $.0195\pm .0004$ 
& $\mathbf{.638\pm .001}$ 
& $\mathbf{1.016\pm .003}$ \\

& \cellcolor{gray!10}$\textsc{LN}$ 
& \cellcolor{gray!10}$\underline{.0208\pm .0007}$ 
& \cellcolor{gray!10}$\mathbf{.6932\pm .0661}$ 
& \cellcolor{gray!10}$.0058\pm .0004$ 
& \cellcolor{gray!10}$.0181\pm .0006$ 
& \cellcolor{gray!10}$\underline{.634\pm .002}$ 
& \cellcolor{gray!10}$\underline{1.021\pm .006}$ \\

& $\textsc{PS}$ 
& $.0418\pm .0018$ 
& $.9518\pm .0045$ 
& $\underline{.0013\pm .0001}$ 
& $.0135\pm .0012$ 
& $.611\pm .004$ 
& $1.180\pm .008$ \\

& \cellcolor{gray!10}$\textsc{IR}$ 
& \cellcolor{gray!10}$.0360\pm .0024$ 
& \cellcolor{gray!10}$.9602\pm .0103$ 
& \cellcolor{gray!10}$.0026\pm .0004$ 
& \cellcolor{gray!10}$.0095\pm .0005$ 
& \cellcolor{gray!10}$.612\pm .004$ 
& \cellcolor{gray!10}$1.096\pm .016$ \\

& $\textsc{TS}$ 
& $.0506\pm .0026$ 
& $.9822\pm .0084$ 
& $.0168\pm .0023$ 
& $.0410\pm .0022$ 
& $.603\pm .008$ 
& $1.244\pm .036$ \\

& \cellcolor{gray!10}$\textsc{PC}$ 
& \cellcolor{gray!10}$.0543\pm .0029$ 
& \cellcolor{gray!10}$.9768\pm .0133$ 
& \cellcolor{gray!10}$.0145\pm .0022$ 
& \cellcolor{gray!10}$.0384\pm .0025$ 
& \cellcolor{gray!10}$.575\pm .008$ 
& \cellcolor{gray!10}$1.417\pm .046$ \\

& $\textsc{NC}$ 
& $.1031\pm .0020$ 
& $.9953\pm .0021$ 
& $.0415\pm .0029$ 
& $.0977\pm .0020$ 
& $.603\pm .008$ 
& $2.100\pm .101$ \\
    \bottomrule
    \end{tabular}
     \end{table*}

\subsection{More architectures: ConvNext} 
\label{sec:convnext_results}

\textbf{Calibration and performance.} We report results for \texttt{ConvNext} architecture in \cref{tab:results_convnext}. Results with a \texttt{ConvNext} backbone confirm the trends observed with \texttt{ResNet}. Across all datasets, \VQ{} consistently achieves the best local calibration, attaining the lowest LCE and MLCE by a clear margin. On \texttt{cifar10} and \texttt{cifar100}, \VQ{} substantially reduces both average and worst-case local calibration error compared to all baselines, with particularly large improvements in $MLCE$, highlighting its robustness in hard regions. Similar behavior is observed on \tissue{}, where \VQ{} again achieves the lowest $LCE$ and $MLCE$, while competing local methods exhibit significantly higher errors.

In terms of global calibration, \VQ{} remains competitive but does not always achieve the best $ECCE$, with global methods such as \DC{} and \SM{} slightly outperforming it in some cases. However, these differences are small and come at the cost of significantly worse local calibration. Finally, predictive performance remains comparable across methods: \VQ{} achieves accuracy on par with other calibrators, with only minor variations across datasets, indicating that improvements in local calibration do not come at the expense of classification performance.

\begin{table*}[t]
\centering
\scriptsize
\caption{\VQ{} results when the classifier backbone is a \texttt{ConvNext} model.} \label{tab:results_convnext} \begin{tabular}{c@{\hspace{7pt}} c@{\hspace{7pt}} c@{\hspace{7pt}} c@{\hspace{7pt}} c@{\hspace{7pt}} c@{\hspace{7pt}} c@{\hspace{7pt}} c}
\toprule
Data & Method & $LCE\,\downarrow$ & $MLCE\,\downarrow$ & $ECCE\,\downarrow$ & $ECE\,\downarrow$ & $ACC\,\uparrow$ & $NLL\,\downarrow$ \\
\midrule

\multirow{10}{*}{\rotatebox{90}{\texttt{cifar10}}}
& \cellcolor{blue!10}$\textsc{VQ}$ & \cellcolor{blue!10}$\mathbf{.0024\pm .0002}$ & \cellcolor{blue!10}$\mathbf{.3326\pm .0353}$ & \cellcolor{blue!10}$\underline{.0006\pm .0003}$ & \cellcolor{blue!10}$\underline{.0021\pm .0004}$ & \cellcolor{blue!10}$.958\pm .002$ & \cellcolor{blue!10}$.128\pm .006$ \\
\cline{2-8}
& \cellcolor{gray!10}$\textsc{DC}$ & \cellcolor{gray!10}$.0034\pm .0006$ & \cellcolor{gray!10}$.5810\pm .0829$ & \cellcolor{gray!10}$\mathbf{.0006\pm .0001}$ & \cellcolor{gray!10}$\mathbf{.0020\pm .0002}$ & \cellcolor{gray!10}$.958\pm .002$ & \cellcolor{gray!10}$.124\pm .003$ \\
& $\textsc{KC}$ & $.0038\pm .0004$ & $\underline{.5501\pm .0854}$ & $.0008\pm .0001$ & $.0025\pm .0003$ & $\mathbf{.960\pm .001}$ & $\mathbf{.122\pm .005}$ \\
& \cellcolor{gray!10}$\textsc{LN}$ & \cellcolor{gray!10}$\underline{.0032\pm .0006}$ & \cellcolor{gray!10}$.5615\pm .1347$ & \cellcolor{gray!10}$.0007\pm .0003$ & \cellcolor{gray!10}$.0023\pm .0004$ & \cellcolor{gray!10}$\underline{.960\pm .002}$ & \cellcolor{gray!10}$\underline{.122\pm .002}$ \\
& $\textsc{SM}$ & $.0036\pm .0006$ & $.5791\pm .0867$ & $.0006\pm .0001$ & $.0021\pm .0001$ & $.959\pm .002$ & $.124\pm .003$ \\
& \cellcolor{gray!10}$\textsc{PS}$ & \cellcolor{gray!10}$.0077\pm .0005$ & \cellcolor{gray!10}$.7295\pm .0649$ & \cellcolor{gray!10}$.0013\pm .0001$ & \cellcolor{gray!10}$.0059\pm .0007$ & \cellcolor{gray!10}$.956\pm .002$ & \cellcolor{gray!10}$.187\pm .005$ \\
& $\textsc{IR}$ & $.0040\pm .0007$ & $.5589\pm .0676$ & $.0006\pm .0001$ & $.0024\pm .0002$ & $.957\pm .002$ & $.148\pm .001$ \\
& \cellcolor{gray!10}$\textsc{TS}$ & \cellcolor{gray!10}$.0047\pm .0004$ & \cellcolor{gray!10}$.5893\pm .0908$ & \cellcolor{gray!10}$.0014\pm .0004$ & \cellcolor{gray!10}$.0033\pm .0006$ & \cellcolor{gray!10}$.956\pm .003$ & \cellcolor{gray!10}$.135\pm .007$ \\
& $\textsc{PC}$ & $.0045\pm .0006$ & $.5987\pm .1074$ & $.0014\pm .0003$ & $.0032\pm .0003$ & $.955\pm .002$ & $.139\pm .005$ \\
& \cellcolor{gray!10}$\textsc{NC}$ & \cellcolor{gray!10}$.0043\pm .0003$ & \cellcolor{gray!10}$.6400\pm .1017$ & \cellcolor{gray!10}$.0017\pm .0003$ & \cellcolor{gray!10}$.0039\pm .0007$ & \cellcolor{gray!10}$.956\pm .003$ & \cellcolor{gray!10}$.143\pm .007$ \\

\midrule

\multirow{10}{*}{\rotatebox{90}{\texttt{cifar100}}}
& \cellcolor{blue!10}$\textsc{VQ}$ & \cellcolor{blue!10}$\mathbf{.0012\pm .0001}$ & \cellcolor{blue!10}$\mathbf{.2778\pm .0078}$ & \cellcolor{blue!10}$.0005\pm .0001$ & \cellcolor{blue!10}$.0014\pm .0001$ & \cellcolor{blue!10}$.799\pm .003$ & \cellcolor{blue!10}$.690\pm .009$ \\
\cline{2-8}
& \cellcolor{gray!10}$\textsc{DC}$ & \cellcolor{gray!10}$\underline{.0013\pm .0001}$ & \cellcolor{gray!10}$.4576\pm .0056$ & \cellcolor{gray!10}$\underline{.0004\pm .0001}$ & \cellcolor{gray!10}$\underline{.0012\pm .0001}$ & \cellcolor{gray!10}$\underline{.817\pm .002}$ & \cellcolor{gray!10}$\mathbf{.609\pm .006}$ \\
& $\textsc{KC}$ & $.0023\pm .0001$ & $.5495\pm .0085$ & $.0009\pm .0001$ & $.0019\pm .0001$ & $.789\pm .001$ & $.729\pm .004$ \\
& \cellcolor{gray!10}$\textsc{LN}$ & \cellcolor{gray!10}$.0015\pm .0001$ & \cellcolor{gray!10}$.6542\pm .0277$ & \cellcolor{gray!10}$.0005\pm .0001$ & \cellcolor{gray!10}$.0013\pm .0001$ & \cellcolor{gray!10}$\mathbf{.819\pm .001}$ & \cellcolor{gray!10}$\underline{.616\pm .006}$ \\
& $\textsc{SM}$ & $.0013\pm .0001$ & $.4611\pm .0059$ & $\mathbf{.0004\pm .0001}$ & $\mathbf{.0012\pm .0001}$ & $.811\pm .001$ & $.631\pm .004$ \\
& \cellcolor{gray!10}$\textsc{PS}$ & \cellcolor{gray!10}$.0037\pm .0001$ & \cellcolor{gray!10}$\underline{.4288\pm .0202}$ & \cellcolor{gray!10}$.0012\pm .0001$ & \cellcolor{gray!10}$.0036\pm .0001$ & \cellcolor{gray!10}$.801\pm .001$ & \cellcolor{gray!10}$.989\pm .008$ \\
& $\textsc{IR}$ & $.0014\pm .0001$ & $.4560\pm .0095$ & $.0004\pm .0001$ & $.0014\pm .0001$ & $.805\pm .002$ & $.837\pm .020$ \\
& \cellcolor{gray!10}$\textsc{TS}$ & \cellcolor{gray!10}$.0020\pm .0001$ & \cellcolor{gray!10}$.4876\pm .0135$ & \cellcolor{gray!10}$.0011\pm .0001$ & \cellcolor{gray!10}$.0020\pm .0001$ & \cellcolor{gray!10}$.791\pm .003$ & \cellcolor{gray!10}$.702\pm .012$ \\
& $\textsc{PC}$ & $.0019\pm .0001$ & $.4851\pm .0159$ & $.0011\pm .0001$ & $.0019\pm .0001$ & $.790\pm .003$ & $.696\pm .011$ \\
& \cellcolor{gray!10}$\textsc{NC}$ & \cellcolor{gray!10}$.0018\pm .0001$ & \cellcolor{gray!10}$.5137\pm .0082$ & \cellcolor{gray!10}$.0010\pm .0001$ & \cellcolor{gray!10}$.0020\pm .0001$ & \cellcolor{gray!10}$.791\pm .003$ & \cellcolor{gray!10}$.711\pm .013$ \\

\midrule

\multirow{10}{*}{\rotatebox{90}{\tissue{}}}
& \cellcolor{blue!10}$\textsc{VQ}$ & \cellcolor{blue!10}$\mathbf{.0041\pm .0003}$ & \cellcolor{blue!10}$\mathbf{.1553\pm .0323}$ & \cellcolor{blue!10}$.0011\pm .0002$ & \cellcolor{blue!10}$.0033\pm .0003$ & \cellcolor{blue!10}$.671\pm .004$ & \cellcolor{blue!10}$.903\pm .011$ \\
\cline{2-8}
& \cellcolor{gray!10}$\textsc{DC}$ & \cellcolor{gray!10}$\underline{.0060\pm .0004}$ & \cellcolor{gray!10}$.4298\pm .0422$ & \cellcolor{gray!10}$\underline{.0011\pm .0002}$ & \cellcolor{gray!10}$\underline{.0031\pm .0003}$ & \cellcolor{gray!10}$.674\pm .003$ & \cellcolor{gray!10}$\underline{.894\pm .008}$ \\
& $\textsc{KC}$ & $.0200\pm .0012$ & $.2108\pm .0127$ & $.0066\pm .0005$ & $.0191\pm .0013$ & $\mathbf{.678\pm .003}$ & $.913\pm .008$ \\
& \cellcolor{gray!10}$\textsc{LN}$ & \cellcolor{gray!10}$.0171\pm .0003$ & \cellcolor{gray!10}$\underline{.1858\pm .0084}$ & \cellcolor{gray!10}$.0061\pm .0003$ & \cellcolor{gray!10}$.0166\pm .0004$ & \cellcolor{gray!10}$\underline{.677\pm .004}$ & \cellcolor{gray!10}$.907\pm .009$ \\
& $\textsc{SM}$ & $.0060\pm .0003$ & $.4325\pm .0424$ & $\mathbf{.0006\pm .0001}$ & $\mathbf{.0029\pm .0001}$ & $.674\pm .003$ & $\mathbf{.894\pm .008}$ \\
& \cellcolor{gray!10}$\textsc{PS}$ & \cellcolor{gray!10}$.0175\pm .0005$ & \cellcolor{gray!10}$.4316\pm .0305$ & \cellcolor{gray!10}$.0023\pm .0001$ & \cellcolor{gray!10}$.0152\pm .0007$ & \cellcolor{gray!10}$.672\pm .003$ & \cellcolor{gray!10}$.953\pm .006$ \\
& $\textsc{IR}$ & $.0075\pm .0006$ & $.4405\pm .0550$ & $.0011\pm .0002$ & $.0041\pm .0004$ & $.672\pm .003$ & $.903\pm .009$ \\
& \cellcolor{gray!10}$\textsc{TS}$ & \cellcolor{gray!10}$.0127\pm .0025$ & \cellcolor{gray!10}$.4303\pm .0458$ & \cellcolor{gray!10}$.0068\pm .0020$ & \cellcolor{gray!10}$.0111\pm .0029$ & \cellcolor{gray!10}$.668\pm .002$ & \cellcolor{gray!10}$.909\pm .006$ \\
& $\textsc{PC}$ & $.0124\pm .0023$ & $.4331\pm .0497$ & $.0066\pm .0019$ & $.0109\pm .0027$ & $.668\pm .002$ & $.910\pm .005$ \\
& \cellcolor{gray!10}$\textsc{NC}$ & \cellcolor{gray!10}$.0159\pm .0036$ & \cellcolor{gray!10}$.4537\pm .0480$ & \cellcolor{gray!10}$.0084\pm .0033$ & \cellcolor{gray!10}$.0144\pm .0041$ & \cellcolor{gray!10}$.668\pm .002$ & \cellcolor{gray!10}$.910\pm .004$ \\

\bottomrule
\end{tabular} \end{table*}

\textbf{Under data sparsity.} For \texttt{ConvNext} architecture we report in \cref{fig:ess_convnext} local calibration error as a function of neighborhood density, measured via ESS. As in the \texttt{ResNet} setting, all methods exhibit a consistent trend: calibration error is highest in low-density regions and decreases as the effective sample size increases. However, \VQ{} consistently achieves the lowest $LCE$ across all datasets and ESS bins, with the largest gains in the sparsest regions. On \texttt{cifar10} and \texttt{cifar100}, \VQ{} significantly reduces local calibration error in low-density bins and quickly reaches near-zero error as density increases, outperforming both global and local baselines. A similar pattern is observed on \tissue{}, where \VQ{} maintains a clear advantage across the entire ESS spectrum, particularly in the low-support regime where other methods exhibit substantially higher errors.

As density increases, the performance gap narrows and all methods converge to similar $LCE$ values, indicating that calibration becomes easier with sufficient local data. These results confirm that the benefits of \VQ{} generalize across different convolutional architectures.

\begin{center}
\begin{figure*}[t]
        \includegraphics[width=\linewidth]{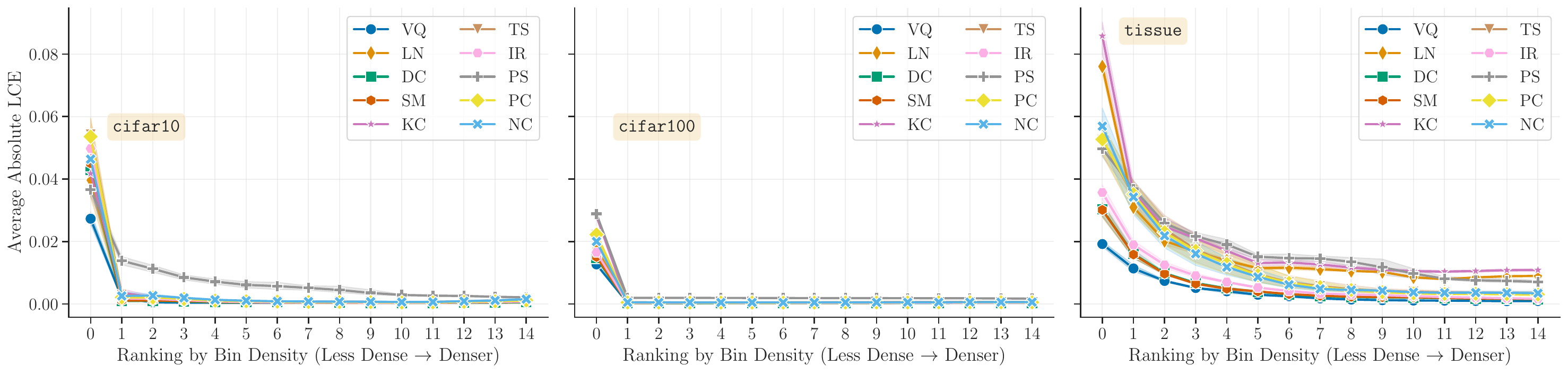}
\caption{{Local calibration in density-based sub-bins over five runs (\texttt{ConvNext} backbone).}} \label{fig:ess_convnext}
\end{figure*}
\end{center}

\subsection{More Architectures: ViT}
\label{sec:vit_results}
\textbf{Calibration and performance.} We report results for \texttt{ViT} architecture in \cref{tab:results_vit}. Results with a \texttt{ViT} backbone further confirm the trends observed across other architectures. VQ consistently achieves the best or near-best local calibration, attaining the lowest $LCE$ across all datasets (tied on \texttt{cifar10}, second-best but within noise on \texttt{cifar100}, and best on \tissue{}). More importantly, \VQ{} clearly dominates in terms of $MLCE$, achieving the lowest worst-case local calibration error on all datasets by a substantial margin with the sole exception of \cifarTen{} where differences are not significant. This highlights its robustness in challenging regions of the representation space, where competing methods exhibit significantly higher variability.

In terms of global calibration, \VQ{} remains competitive but does not consistently achieve the best $ECCE$, with global methods such as \DC{} and \SM{} occasionally outperforming it. However, these gains in global calibration are accompanied by substantially worse local calibration. Finally, \VQ{} shows a slight decrease in $ACC$ in some ViT settings, likely due to interactions between quantization and the geometry of the representation space. However, this effect is small and does not impact its strong gains in local calibration.
\begin{table*}[t]
\centering
\scriptsize
\caption{\VQ{} results when the classifier backbone is a \texttt{ViT} model.} \label{tab:results_vit} \begin{tabular}{c@{\hspace{7pt}} c@{\hspace{7pt}} c@{\hspace{7pt}} c@{\hspace{7pt}} c@{\hspace{7pt}} c@{\hspace{7pt}} c@{\hspace{7pt}} c}
\toprule
Data & Method & $LCE\,\downarrow$ & $MLCE\,\downarrow$ & $ECCE\,\downarrow$ & $ECE\,\downarrow$ & $ACC\,\uparrow$ & $NLL\,\downarrow$ \\
\midrule

\multirow{10}{*}{\rotatebox{90}{\texttt{cifar10}}}
& \cellcolor{blue!10}$\textsc{VQ}$ & \cellcolor{blue!10}$\mathbf{.0016\pm .0001}$ & \cellcolor{blue!10}$.2808\pm .0210$ & \cellcolor{blue!10}$.0004\pm .0001$ & \cellcolor{blue!10}$.0016\pm .0001$ & \cellcolor{blue!10}$.974\pm .001$ & \cellcolor{blue!10}$.086\pm .002$ \\
\cline{2-8}
& \cellcolor{gray!10}$\textsc{DC}$ & \cellcolor{gray!10}$\mathbf{.0016\pm .0001}$ & \cellcolor{gray!10}$.4789\pm .0354$ & \cellcolor{gray!10}$\underline{.0003\pm .0001}$ & \cellcolor{gray!10}$\mathbf{.0013\pm .0001}$ & \cellcolor{gray!10}$.977\pm .001$ & \cellcolor{gray!10}$.073\pm .001$ \\
& $\textsc{KC}$ & $.0032\pm .0002$ & $.7915\pm .0214$ & $.0006\pm .0001$ & $.0018\pm .0002$ & $\mathbf{.979\pm .001}$ & $\mathbf{.071\pm .002}$ \\
& \cellcolor{gray!10}$\textsc{LN}$ & \cellcolor{gray!10}$.0025\pm .0003$ & \cellcolor{gray!10}$.9027\pm .0709$ & \cellcolor{gray!10}$.0013\pm .0005$ & \cellcolor{gray!10}$.0032\pm .0011$ & \cellcolor{gray!10}$.977\pm .001$ & \cellcolor{gray!10}$.132\pm .039$ \\
& $\textsc{SM}$ & $.0017\pm .0001$ & $.4389\pm .0408$ & $\mathbf{.0003\pm .0001}$ & $\underline{.0014\pm .0001}$ & $\underline{.978\pm .001}$ & $\underline{.071\pm .002}$ \\
& \cellcolor{gray!10}$\textsc{PS}$ & \cellcolor{gray!10}$.0043\pm .0002$ & \cellcolor{gray!10}$\mathbf{.2219\pm .1137}$ & \cellcolor{gray!10}$.0011\pm .0001$ & \cellcolor{gray!10}$.0035\pm .0005$ & \cellcolor{gray!10}$.977\pm .001$ & \cellcolor{gray!10}$.118\pm .007$ \\
& $\textsc{IR}$ & $.0018\pm .0001$ & $.4422\pm .0467$ & $.0004\pm .0001$ & $.0016\pm .0001$ & $.977\pm .001$ & $.087\pm .004$ \\
& \cellcolor{gray!10}$\textsc{TS}$ & \cellcolor{gray!10}$.0020\pm .0001$ & \cellcolor{gray!10}$.4728\pm .0424$ & \cellcolor{gray!10}$.0008\pm .0002$ & \cellcolor{gray!10}$.0019\pm .0003$ & \cellcolor{gray!10}$.977\pm .001$ & \cellcolor{gray!10}$.078\pm .004$ \\
& $\textsc{PC}$ & $.0022\pm .0003$ & $.5402\pm .0483$ & $.0009\pm .0002$ & $.0024\pm .0004$ & $.975\pm .001$ & $.086\pm .006$ \\
& \cellcolor{gray!10}$\textsc{NC}$ & \cellcolor{gray!10}$.0021\pm .0002$ & \cellcolor{gray!10}$.5239\pm .0750$ & \cellcolor{gray!10}$.0012\pm .0003$ & \cellcolor{gray!10}$.0031\pm .0007$ & \cellcolor{gray!10}$.977\pm .001$ & \cellcolor{gray!10}$.106\pm .018$ \\

\midrule

\multirow{10}{*}{\rotatebox{90}{\texttt{cifar100}}}
& \cellcolor{blue!10}$\textsc{VQ}$ & \cellcolor{blue!10}$\underline{.0012\pm .0001}$ & \cellcolor{blue!10}$\mathbf{.2680\pm .0075}$ & \cellcolor{blue!10}$.0005\pm .0001$ & \cellcolor{blue!10}$.0014\pm .0001$ & \cellcolor{blue!10}$.773\pm .002$ & \cellcolor{blue!10}$.838\pm .003$ \\
\cline{2-8}
& \cellcolor{gray!10}$\textsc{DC}$ & \cellcolor{gray!10}$.0023\pm .0002$ & \cellcolor{gray!10}$.4558\pm .0088$ & \cellcolor{gray!10}$.0009\pm .0001$ & \cellcolor{gray!10}$.0026\pm .0002$ & \cellcolor{gray!10}$.786\pm .004$ & \cellcolor{gray!10}$1.013\pm .083$ \\
& $\textsc{KC}$ & $.0049\pm .0001$ & $.8920\pm .0029$ & $.0010\pm .0001$ & $.0022\pm .0001$ & $.782\pm .001$ & $.818\pm .005$ \\
& \cellcolor{gray!10}$\textsc{LN}$ & \cellcolor{gray!10}$.0028\pm .0001$ & \cellcolor{gray!10}$.9671\pm .0017$ & \cellcolor{gray!10}$.0006\pm .0001$ & \cellcolor{gray!10}$.0017\pm .0001$ & \cellcolor{gray!10}$\underline{.819\pm .002}$ & \cellcolor{gray!10}$.717\pm .008$ \\
& $\textsc{SM}$ & $\mathbf{.0011\pm .0001}$ & $.3590\pm .0090$ & $\mathbf{.0003\pm .0001}$ & $\mathbf{.0011\pm .0001}$ & $\mathbf{.819\pm .002}$ & $\mathbf{.660\pm .005}$ \\
& \cellcolor{gray!10}$\textsc{PS}$ & \cellcolor{gray!10}$.0030\pm .0001$ & \cellcolor{gray!10}$\underline{.3317\pm .0149}$ & \cellcolor{gray!10}$.0011\pm .0001$ & \cellcolor{gray!10}$.0031\pm .0001$ & \cellcolor{gray!10}$.818\pm .002$ & \cellcolor{gray!10}$.969\pm .007$ \\
& $\textsc{IR}$ & $\mathbf{.0011\pm .0001}$ & $.3439\pm .0073$ & $\underline{.0004\pm .0001}$ & $.0013\pm .0001$ & $.816\pm .002$ & $.868\pm .011$ \\
& \cellcolor{gray!10}$\textsc{TS}$ & \cellcolor{gray!10}$.0012\pm .0001$ & \cellcolor{gray!10}$.3715\pm .0048$ & \cellcolor{gray!10}$.0005\pm .0001$ & \cellcolor{gray!10}$\underline{.0013\pm .0001}$ & \cellcolor{gray!10}$.816\pm .002$ & \cellcolor{gray!10}$\underline{.677\pm .007}$ \\
& $\textsc{PC}$ & $.0014\pm .0001$ & $.3550\pm .0094$ & $.0005\pm .0001$ & $.0015\pm .0001$ & $.812\pm .002$ & $.700\pm .005$ \\
& \cellcolor{gray!10}$\textsc{NC}$ & \cellcolor{gray!10}$.0017\pm .0001$ & \cellcolor{gray!10}$.4098\pm .0056$ & \cellcolor{gray!10}$.0008\pm .0001$ & \cellcolor{gray!10}$.0020\pm .0001$ & \cellcolor{gray!10}$.816\pm .002$ & \cellcolor{gray!10}$.768\pm .008$ \\

\midrule

\multirow{10}{*}{\rotatebox{90}{\tissue{}}}
& \cellcolor{blue!10}$\textsc{VQ}$ & \cellcolor{blue!10}$\mathbf{.0042\pm .0006}$ & \cellcolor{blue!10}$\mathbf{.1064\pm .0101}$ & \cellcolor{blue!10}$.0014\pm .0004$ & \cellcolor{blue!10}$.0037\pm .0005$ & \cellcolor{blue!10}$.662\pm .001$ & \cellcolor{blue!10}$.925\pm .001$ \\
\cline{2-8}
& \cellcolor{gray!10}$\textsc{DC}$ & \cellcolor{gray!10}$.0062\pm .0003$ & \cellcolor{gray!10}$.5737\pm .0755$ & \cellcolor{gray!10}$\underline{.0011\pm .0002}$ & \cellcolor{gray!10}$\mathbf{.0034\pm .0002}$ & \cellcolor{gray!10}$.677\pm .002$ & \cellcolor{gray!10}$\mathbf{.886\pm .005}$ \\
& $\textsc{KC}$ & $.0225\pm .0015$ & $.6667\pm .0451$ & $.0066\pm .0005$ & $.0193\pm .0018$ & $\mathbf{.678\pm .002}$ & $.911\pm .002$ \\
& \cellcolor{gray!10}$\textsc{LN}$ & \cellcolor{gray!10}$.0189\pm .0007$ & \cellcolor{gray!10}$.5906\pm .0351$ & \cellcolor{gray!10}$.0059\pm .0005$ & \cellcolor{gray!10}$.0159\pm .0008$ & \cellcolor{gray!10}$.676\pm .002$ & \cellcolor{gray!10}$.906\pm .003$ \\
& $\textsc{SM}$ & $.0066\pm .0005$ & $.5603\pm .0670$ & $\mathbf{.0009\pm .0001}$ & $\underline{.0036\pm .0002}$ & $\underline{.677\pm .002}$ & $\underline{.885\pm .005}$ \\
& \cellcolor{gray!10}$\textsc{PS}$ & \cellcolor{gray!10}$.0180\pm .0008$ & \cellcolor{gray!10}$.4960\pm .0901$ & \cellcolor{gray!10}$.0025\pm .0002$ & \cellcolor{gray!10}$.0166\pm .0010$ & \cellcolor{gray!10}$.674\pm .002$ & \cellcolor{gray!10}$.951\pm .008$ \\
& $\textsc{IR}$ & $.0077\pm .0004$ & $.5725\pm .0937$ & $.0014\pm .0002$ & $.0042\pm .0002$ & $.675\pm .002$ & $.896\pm .007$ \\
& \cellcolor{gray!10}$\textsc{TS}$ & \cellcolor{gray!10}$.0132\pm .0006$ & \cellcolor{gray!10}$.5615\pm .0898$ & \cellcolor{gray!10}$.0066\pm .0007$ & \cellcolor{gray!10}$.0108\pm .0008$ & \cellcolor{gray!10}$.672\pm .002$ & \cellcolor{gray!10}$.902\pm .005$ \\
& $\textsc{PC}$ & $.0126\pm .0006$ & $.5676\pm .0732$ & $.0062\pm .0006$ & $.0110\pm .0007$ & $.672\pm .001$ & $.905\pm .006$ \\
& \cellcolor{gray!10}$\textsc{NC}$ & \cellcolor{gray!10}$.0169\pm .0033$ & \cellcolor{gray!10}$.6226\pm .0826$ & \cellcolor{gray!10}$.0076\pm .0016$ & \cellcolor{gray!10}$.0164\pm .0044$ & \cellcolor{gray!10}$.672\pm .002$ & \cellcolor{gray!10}$.911\pm .011$ \\

\bottomrule
\end{tabular} \end{table*}

\textbf{Under data sparsity.} For \texttt{ViT} architecture we report in \cref{fig:ess_vit} local calibration error as a function of neighborhood density, measured via ESS. As in previous architectures, all methods exhibit the expected trend: calibration error is highest in low-density regions and decreases as ESS increases. However, the relative improvements of \VQ{} vary across datasets. On \texttt{cifar10} and \texttt{cifar100}, differences between methods are relatively small, and \VQ{} performs comparably to strong baselines across most ESS bins, with no substantial advantage in low-density regions.

In contrast, on \tissue{}, \VQ{} consistently achieves lower $LCE$ across the entire ESS spectrum, with particularly pronounced gains in the lowest-density bins. This indicates that \VQ{} remains effective in mitigating calibration errors under data sparsity when the representation space exhibits stronger heterogeneity. Overall, these results suggest that the benefits of \VQ{} depend on the structure of the underlying representation space: while gains may be limited when features are already well-behaved (as in \cifarTen{} or \cifar{} with \texttt{ViT}), \VQ{} provides clear advantages in more challenging settings.

\begin{figure*}[t]
        \includegraphics[width=\linewidth]{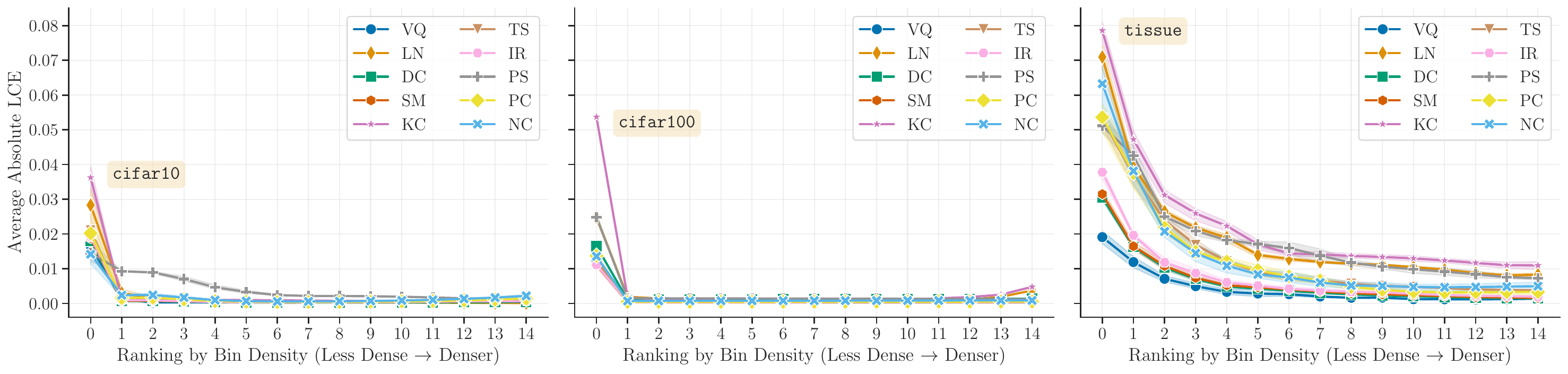}
\caption{{Local calibration in density-based sub-bins over five runs (\texttt{ViT} backbone).}} \label{fig:ess_vit}
\end{figure*}

\subsection{Codewords usage statistics}
\label{sec:usage_stats}
We report quantization codewords usage statistics across datasets and architectures in \cref{tab:usage_stats}. 

Min Usage and Max Usage report the minimum and maximum number of assignments received by any codeword, respectively, while Usage Std denotes the standard deviation of codeword assignments, capturing the variability in usage across the codebook. \# Assign. indicates the total number of assignments (in millions), obtained as the number of quantization slots times the test-set cardinality. Each value is assigned a standard deviation across the 5 used seeds. These statistics provide a global view of how evenly the codebook is utilized and allow us to detect potential issues such as underutilized (dead) codewords.

We observe that all codewords are consistently assigned a non-zero number of samples, with minimum usage well above zero in all settings. This confirms that no dead codewords are present and that the learned codebooks are fully utilized. While usage is naturally imbalanced, with some codewords receiving more assignments than others, the standard deviation remains moderate relative to the total number of assignments, indicating a healthy distribution of mass across the codebook. Overall, these results show that the quantization scheme effectively leverages the full capacity of the codebook across different architectures and datasets.

\begin{center}
\begin{table*}[t]
\centering
\caption{\VQ{} codewords usage statistics across all datasets and model architectures on the test set.} 
\label{tab:usage_stats}

\begin{tabular}{c@{\hspace{6pt}} c@{\hspace{6pt}} 
@{\hspace{6pt}}S[table-format=6.0(4)]
@{\hspace{6pt}}S[table-format=6.0(5)]
@{\hspace{6pt}}S[table-format=6.0(4)]
@{\hspace{6pt}}S[table-format=1.3]}
\toprule
Data & Model & {Min Usage} & {Max Usage} & {Usage Std} & {\# Assign. ($\times 10^6$)} \\
\midrule

\multirow{3}{*}{\texttt{cifar10}}
&  \cellcolor{gray!10}\textsc{resnet} 
&  \cellcolor{gray!10}\num{1790(208)} 
&  \cellcolor{gray!10}\num{126583(10273)} 
&  \cellcolor{gray!10}\num{17974(1123)} 
&  \cellcolor{gray!10}\num{1.296} \\

& \textsc{convnext} 
& \num{2874(1200)} 
& \num{42831(8268)} 
& \num{9245(2427)} 
& \num{1.296} \\

&  \cellcolor{gray!10}\textsc{vit} 
&  \cellcolor{gray!10}\num{4968(1425)} 
&  \cellcolor{gray!10}\num{27955(1325)} 
&  \cellcolor{gray!10}\num{5252(580)} 
&  \cellcolor{gray!10}\num{1.296} \\

\midrule

\multirow{3}{*}{\texttt{cifar100}}
&  \cellcolor{gray!10}\textsc{resnet} 
&  \cellcolor{gray!10}\num{2481(1167)} 
&  \cellcolor{gray!10}\num{106616(12547)} 
&  \cellcolor{gray!10}\num{14896(683)} 
&  \cellcolor{gray!10}\num{1.296} \\

& \textsc{convnext} 
& \num{4294(792)} 
& \num{31017(3162)} 
& \num{5207(602)} 
& \num{1.296} \\

&  \cellcolor{gray!10}\textsc{vit} 
&  \cellcolor{gray!10}\num{5435(1635)} 
&  \cellcolor{gray!10}\num{23984(599)} 
&  \cellcolor{gray!10}\num{3565(425)} 
&  \cellcolor{gray!10}\num{1.296} \\

\midrule

\multirow{3}{*}{\tissue{}}
&  \cellcolor{gray!10}\textsc{resnet} 
&  \cellcolor{gray!10}\num{4825(897)} 
&  \cellcolor{gray!10}\num{503999(17512)} 
&  \cellcolor{gray!10}\num{71517(1148)} 
&  \cellcolor{gray!10}\num{4.765} \\

& \textsc{convnext} 
& \num{9233(2048)} 
& \num{254738(32372)} 
& \num{50291(4460)} 
& \num{4.765} \\

&  \cellcolor{gray!10}\textsc{vit} 
&  \cellcolor{gray!10}\num{14126(2716)} 
&  \cellcolor{gray!10}\num{116994(7445)} 
&  \cellcolor{gray!10}\num{23582(1222)} 
&  \cellcolor{gray!10}\num{4.765} \\

\midrule

\multirow{1}{*}{\texttt{weather}}
&  \cellcolor{gray!10}\textsc{ftt} 
&  \cellcolor{gray!10}\num{28983(5381)} 
&  \cellcolor{gray!10}\num{177686(21865)} 
&  \cellcolor{gray!10}\num{41796(3553)} 
&  \cellcolor{gray!10}\num{5.760} \\

\bottomrule
\end{tabular}

 \end{table*}
\end{center}

\subsection{Regularisation of the parameter codebooks and of the diagonal scaling component.}
\label{sec:reg_sigma}

\begin{center}
\begin{table*}[t]
\scriptsize
\caption{Results with a \texttt{ResNet} classifier. $\lambda_{\mathcal A}, \lambda_{\mathcal B},\lambda_{\boldsymbol\sigma}=1e^{-4}$} 
\label{tab:results_resnet_diag} 
    \begin{tabular}{c@{\hspace{7pt}} c@{\hspace{7pt}}  c@{\hspace{7pt}}  c@{\hspace{7pt}}  c@{\hspace{7pt}}  c@{\hspace{7pt}}  c@{\hspace{7pt}}  c}
\toprule
Data & Method & $LCE\,\downarrow$ & $MLCE\,\downarrow$ & $ECCE\,\downarrow$ & $ECE\,\downarrow$ & $ACC\,\uparrow$ & $NLL\,\downarrow$ \\
\midrule
\multirow{10}{*}{\rotatebox{90}{\texttt{cifar10}}} & \cellcolor{blue!10}$\textsc{VQ}$ & \cellcolor{blue!10}$\underline{.0059\pm .0002}$ & \cellcolor{blue!10}$\underline{.5595\pm .0276}$ & \cellcolor{blue!10}$.0013\pm .0002$ & \cellcolor{blue!10}$.0037\pm .0002$ & \cellcolor{blue!10}$.889\pm .001$ & \cellcolor{blue!10}$.348\pm .002$ \\
& $\textsc{VQ}_\lambda$ & $\mathbf{.0058\pm .0002}$ & $\mathbf{.5368\pm .0591}$ & $.0011\pm .0003$ & $.0038\pm .0003$ & $.890 \pm .002$ & $.342\pm .006$ \\
& $\textsc{SM}$ & $.0102\pm .0003$ & $.8027\pm .0166$ & $\mathbf{.0008\pm .0001}$ & $\mathbf{.0037\pm .0001}$ & $\underline{.890\pm .002}$ & $\underline{.333\pm .007}$ \\
& \cellcolor{gray!10}$\textsc{DC}$ 
& \cellcolor{gray!10}$.0099\pm .0003$ 
& \cellcolor{gray!10}$.8127\pm .0194$ 
& \cellcolor{gray!10}$\underline{.0011\pm .0001}$ 
& \cellcolor{gray!10}$\underline{.0037\pm .0001}$ 
& \cellcolor{gray!10}$.890\pm .002$ 
& \cellcolor{gray!10}$\mathbf{.332\pm .007}$ \\
 & $\textsc{KC}$ & $.0095\pm .0004$ & $.7923\pm .0241$ & $.0015\pm .0003$ & $.0054\pm .0004$ & $\mathbf{.893\pm .001}$ & $.340\pm .002$ \\
 & \cellcolor{gray!10}$\textsc{LN}$ & \cellcolor{gray!10}$.0079\pm .0003$ & \cellcolor{gray!10}$.6279\pm .0157$ & \cellcolor{gray!10}$.0017\pm .0001$ & \cellcolor{gray!10}$.0048\pm .0003$ & \cellcolor{gray!10}$.888\pm .002$ & \cellcolor{gray!10}$.347\pm .002$ \\
 & $\textsc{PS}$ & $.0180\pm .0005$ & $.8542\pm .0224$ & $.0014\pm .0001$ & $.0085\pm .0003$ & $.884\pm .001$ & $.466\pm .004$ \\
 & \cellcolor{gray!10}$\textsc{IR}$ & \cellcolor{gray!10}$.0115\pm .0003$ & \cellcolor{gray!10}$.7978\pm .0162$ & \cellcolor{gray!10}$.0014\pm .0002$ & \cellcolor{gray!10}$.0047\pm .0001$ & \cellcolor{gray!10}$.884\pm .001$ & \cellcolor{gray!10}$.364\pm .008$ \\
 & $\textsc{TS}$ & $.0127\pm .0006$ & $.8200\pm .0191$ & $.0036\pm .0008$ & $.0072\pm .0010$ & $.884\pm .001$ & $.362\pm .008$ \\
 & \cellcolor{gray!10}$\textsc{PC}$ & \cellcolor{gray!10}$.0133\pm .0006$ & \cellcolor{gray!10}$.8339\pm .0456$ & \cellcolor{gray!10}$.0039\pm .0008$ & \cellcolor{gray!10}$.0085\pm .0009$ & \cellcolor{gray!10}$.880\pm .002$ & \cellcolor{gray!10}$.377\pm .008$ \\
& $\textsc{NC}$ 
& $.0154\pm .0005$ 
& $.8729\pm .0235$ 
& $.0066\pm .0004$ 
& $.0153\pm .0006$ 
& $.8838\pm .0014$ 
& $.4943\pm .0180$ \\

\midrule
\multirow{10}{*}{\rotatebox{90}{\texttt{cifar100}}} & \cellcolor{blue!10}$\textsc{VQ}$ & \cellcolor{blue!10}$\mathbf{.0017\pm .0001}$ & \cellcolor{blue!10}$\mathbf{.2932\pm .0099}$ & \cellcolor{blue!10}$\underline{.0008\pm .0001}$ & \cellcolor{blue!10}$\mathbf{.0018\pm .0001}$ & \cellcolor{blue!10}$.681\pm .001$ & \cellcolor{blue!10}$1.208\pm .007$ \\
& $\textsc{VQ}_\lambda$ & $\mathbf{.0017\pm .0001}$ & $\underline{.3073\pm .0115}$ & $\underline{.0008\pm .0001}$ & $\mathbf{.0018\pm .0001}$ & $.679 \pm .002$ & $1.227\pm .013$ \\
& $\textsc{SM}$ & $.0031\pm .0002$ & $.8121\pm .0190$ & $.0008\pm .0006$ & $.0024\pm .0015$ & $.677\pm .013$ & $1.225\pm .043$ \\
& \cellcolor{gray!10}$\textsc{DC}$ 
& \cellcolor{gray!10}$.0027\pm .0001$ 
& \cellcolor{gray!10}$.8113\pm .0121$ 
& \cellcolor{gray!10}$\mathbf{.0007\pm .0001}$ 
& \cellcolor{gray!10}$.0019\pm .0001$ 
& \cellcolor{gray!10}$\mathbf{.690\pm .002}$ 
& \cellcolor{gray!10}$\underline{1.154\pm .009}$ \\
 & $\textsc{KC}$ & $.0043\pm .0001$ & $.7575\pm .0101$ & $.0013\pm .0001$ & $.0044\pm .0001$ & $.679\pm .003$ & $1.352\pm .009$ \\
 & \cellcolor{gray!10}$\textsc{LN}$ & \cellcolor{gray!10}$\underline{.0024\pm .0001}$ & \cellcolor{gray!10}$.7022\pm .0070$ & \cellcolor{gray!10}$\mathbf{.0007\pm .0001}$ & \cellcolor{gray!10}$\underline{.0018\pm .0001}$ & \cellcolor{gray!10}$\underline{.688\pm .001}$ & \cellcolor{gray!10}$\mathbf{1.125\pm .002}$ \\
 & $\textsc{PS}$ & $.0053\pm .0001$ & $.7031\pm .0409$ & $.0011\pm .0001$ & $.0033\pm .0001$ & $.670\pm .002$ & $1.618\pm .007$ \\
 & \cellcolor{gray!10}$\textsc{IR}$ & \cellcolor{gray!10}$.0031\pm .0001$ & \cellcolor{gray!10}$.8049\pm .0116$ & \cellcolor{gray!10}$\mathbf{.0007\pm .0001}$ & \cellcolor{gray!10}$.0020\pm .0001$ & \cellcolor{gray!10}$.670\pm .002$ & \cellcolor{gray!10}$1.437\pm .029$ \\
 & $\textsc{TS}$ & $.0034\pm .0001$ & $.8239\pm .0173$ & $.0013\pm .0001$ & $.0024\pm .0001$ & $.670\pm .002$ & $1.277\pm .009$ \\
 & \cellcolor{gray!10}$\textsc{PC}$ & \cellcolor{gray!10}$.0041\pm .0002$ & \cellcolor{gray!10}$.8187\pm .0127$ & \cellcolor{gray!10}$.0013\pm .0001$ & \cellcolor{gray!10}$.0026\pm .0001$ & \cellcolor{gray!10}$.662\pm .003$ & \cellcolor{gray!10}$1.272\pm .012$ \\
 & $\textsc{NC}$ & $.0032\pm .0001$ & $.9155\pm .0104$ & $.0017\pm .0001$ & $.0039\pm .0001$ & $.670\pm .002$ & $1.502\pm .035$ \\

\midrule
\multirow{10}{*}{\rotatebox{90}{\tissue{}}} 
& \cellcolor{blue!10}$\textsc{VQ}$ 
& \cellcolor{blue!10}$\underline{.0083\pm .0002}$ 
& \cellcolor{blue!10}$.6984\pm .0368$ 
& \cellcolor{blue!10}$.0017\pm .0004$ 
& \cellcolor{blue!10}$.0043\pm .0002$ 
& \cellcolor{blue!10}$.624\pm .002$ 
& \cellcolor{blue!10}$1.028\pm .007$ \\
& $\textsc{VQ}_\lambda$ & $\mathbf{.0080\pm .0004}$ & $\underline{.6945\pm .0472}$ & $\underline{.0012\pm .0004}$ & $\underline{.0037\pm .0006}$ & $.625 \pm .002$ & $1.028\pm .007$ \\
& $\textsc{SM}$ 
& $.0278\pm .0013$ 
& $.9609\pm .0106$ 
& $\mathbf{.0008\pm .0001}$ 
& $\mathbf{.0037\pm .0001}$ 
& $.618\pm .003$ 
& $1.049\pm .009$ \\

& \cellcolor{gray!10}$\textsc{DC}$ 
& \cellcolor{gray!10}$.0322\pm .0015$ 
& \cellcolor{gray!10}$.9706\pm .0098$ 
& \cellcolor{gray!10}$.0065\pm .0009$ 
& \cellcolor{gray!10}$.0183\pm .0010$ 
& \cellcolor{gray!10}$.615\pm .004$ 
& \cellcolor{gray!10}$1.079\pm .012$ \\

& $\textsc{KC}$ 
& $.0226\pm .0006$ 
& $.7296\pm .0570$ 
& $.0061\pm .0004$ 
& $.0195\pm .0004$ 
& $\mathbf{.638\pm .001}$ 
& $\mathbf{1.016\pm .003}$ \\

& \cellcolor{gray!10}$\textsc{LN}$ 
& \cellcolor{gray!10}$.0208\pm .0007$ 
& \cellcolor{gray!10}$\mathbf{.6932\pm .0661}$ 
& \cellcolor{gray!10}$.0058\pm .0004$ 
& \cellcolor{gray!10}$.0181\pm .0006$ 
& \cellcolor{gray!10}$\underline{.634\pm .002}$ 
& \cellcolor{gray!10}$\underline{1.021\pm .006}$ \\

& $\textsc{PS}$ 
& $.0418\pm .0018$ 
& $.9518\pm .0045$ 
& $.0013\pm .0001$ 
& $.0135\pm .0012$ 
& $.611\pm .004$ 
& $1.180\pm .008$ \\

& \cellcolor{gray!10}$\textsc{IR}$ 
& \cellcolor{gray!10}$.0360\pm .0024$ 
& \cellcolor{gray!10}$.9602\pm .0103$ 
& \cellcolor{gray!10}$.0026\pm .0004$ 
& \cellcolor{gray!10}$.0095\pm .0005$ 
& \cellcolor{gray!10}$.612\pm .004$ 
& \cellcolor{gray!10}$1.096\pm .016$ \\

& $\textsc{TS}$ 
& $.0506\pm .0026$ 
& $.9822\pm .0084$ 
& $.0168\pm .0023$ 
& $.0410\pm .0022$ 
& $.603\pm .008$ 
& $1.244\pm .036$ \\

& \cellcolor{gray!10}$\textsc{PC}$ 
& \cellcolor{gray!10}$.0543\pm .0029$ 
& \cellcolor{gray!10}$.9768\pm .0133$ 
& \cellcolor{gray!10}$.0145\pm .0022$ 
& \cellcolor{gray!10}$.0384\pm .0025$ 
& \cellcolor{gray!10}$.575\pm .008$ 
& \cellcolor{gray!10}$1.417\pm .046$ \\

& $\textsc{NC}$ 
& $.1031\pm .0020$ 
& $.9953\pm .0021$ 
& $.0415\pm .0029$ 
& $.0977\pm .0020$ 
& $.603\pm .008$ 
& $2.100\pm .101$ \\
    \bottomrule
    \end{tabular}
     \end{table*}
\end{center}

In our experiments, we keep the diagonal scaling $\boldsymbol{\sigma}^2$ fixed. In what follows we explore a regularisation strategy to improve the flexibility of the model allowing it to freely learn the diagonal scaling component while preventing overfitting. 
Interestingly, \cref{cor:cell_est_err} already highlights the role of the diagonal scaling. The cell-level estimation error grows proportionally to the slot weights \(\boldsymbol\sigma_i^2\), so excessively large scaling coefficients may amplify estimation noise from the corresponding codeword parameters. This provides additional motivation for regularizing \(\boldsymbol{\sigma}^2\) when it is learned.
Following \citep{DBLP:conf/nips/KullPKFSF19}, we propose to add the following regularisation component to the cross-entropy loss:
\begin{align}    
\mathcal R(\mathcal A,\mathcal  B,\boldsymbol \sigma)
& =
\lambda_{\mathcal A}
\sum_{c=1}^{|\mathcal C|}
\frac{|\calY|}{n_c+\eta}
\|\mathbf a_c-\mathbf a_c^0\|_2^2 \notag \\
& \qquad +
\lambda_{\mathcal B}
\sum_{c=1}^{|\mathcal C|}
\frac{|\calY|}{n_c+\eta}
\|\mathbf b_c-\mathbf b_c^0\|_2^2 \notag \\
& \qquad \qquad +
\lambda_{\boldsymbol \sigma}
\sum_{j=1}^{w}
\frac{1}{n_{\text{cal}}+\eta}
(\sigma_j^2-1)^2 
\end{align}
\begin{center}
\begin{table*}[t]
\scriptsize
\caption{Results on the \texttt{Weather} dataset with an \texttt{FT-Transformer}. $\lambda_{\mathcal A}, \lambda_{\mathcal B},\lambda_{\boldsymbol\sigma}=1e^{-4}$}
\label{tab:results_weather_diag} 
\begin{tabular}{c@{\hspace{7pt}} c@{\hspace{7pt}} c@{\hspace{7pt}} c@{\hspace{7pt}} c@{\hspace{7pt}} c@{\hspace{7pt}} c@{\hspace{7pt}} c}
\toprule
Data & Method & $LCE\,\downarrow$ & $MLCE\,\downarrow$ & $ECCE\,\downarrow$ & $ECE\,\downarrow$ & $ACC\,\uparrow$ & $NLL\,\downarrow$ \\
\midrule

\multirow{10}{*}{\rotatebox{90}{\texttt{weather}}}
& \cellcolor{blue!10}$\textsc{VQ}$ & \cellcolor{blue!10}$\underline{.0053\pm .0007}$ & \cellcolor{blue!10}$.1782\pm .0344$ & \cellcolor{blue!10}$.0010\pm .0004$ & \cellcolor{blue!10}$.0030\pm .0007$ & \cellcolor{blue!10}$.603\pm .001$ & \cellcolor{blue!10}$.943\pm .003$ \\
\cline{2-8}
& $\textsc{VQ}_\lambda$ & $\mathbf{.0047\pm .0003}$ & $.1968\pm .0325$ & $\underline{.0007\pm .0001}$ & $.0024\pm .0002$ & $.603 \pm .001$ & $.943\pm .003$ \\
& \cellcolor{gray!10}$\textsc{DC}$ & \cellcolor{gray!10}$.0221\pm .0018$ & \cellcolor{gray!10}$.4509\pm .0102$ & \cellcolor{gray!10}$\underline{.0007\pm .0001}$ & \cellcolor{gray!10}$.0022\pm .0001$ & \cellcolor{gray!10}$.607\pm .002$ & \cellcolor{gray!10}$\underline{.933\pm .003}$ \\
& $\textsc{KC}$ & $.0101\pm .0007$ & $\mathbf{.0998\pm .0208}$ & $.0028\pm .0003$ & $.0091\pm .0005$ & $.606\pm .002$ & $.941\pm .002$ \\
& \cellcolor{gray!10}$\textsc{LN}$ & \cellcolor{gray!10}$.0107\pm .0014$ & \cellcolor{gray!10}$\underline{.1021\pm .0138}$ & \cellcolor{gray!10}$.0028\pm .0003$ & \cellcolor{gray!10}$.0100\pm .0015$ & \cellcolor{gray!10}$.605\pm .001$ & \cellcolor{gray!10}$.945\pm .004$ \\
& $\textsc{SM}$ & $.0221\pm .0018$ & $.4495\pm .0105$ & $\mathbf{.0005\pm .0001}$ & $\mathbf{.0015\pm .0002}$ & $\mathbf{.607\pm .001}$ & $\mathbf{.933\pm .003}$ \\
& \cellcolor{gray!10}$\textsc{PS}$ & \cellcolor{gray!10}$.0242\pm .0018$ & \cellcolor{gray!10}$.4379\pm .0104$ & \cellcolor{gray!10}$.0009\pm .0000$ & \cellcolor{gray!10}$.0068\pm .0003$ & \cellcolor{gray!10}$.607\pm .001$ & \cellcolor{gray!10}$.948\pm .003$ \\
& $\textsc{IR}$ & $.0223\pm .0018$ & $.4505\pm .0109$ & $.0006\pm .0001$ & $\underline{.0020\pm .0001}$ & $.606\pm .001$ & $.936\pm .003$ \\
& \cellcolor{gray!10}$\textsc{TS}$ & \cellcolor{gray!10}$.0234\pm .0018$ & \cellcolor{gray!10}$.4429\pm .0145$ & \cellcolor{gray!10}$.0019\pm .0004$ & \cellcolor{gray!10}$.0045\pm .0005$ & \cellcolor{gray!10}$.607\pm .001$ & \cellcolor{gray!10}$.936\pm .003$ \\
& $\textsc{PC}$ & $.0227\pm .0018$ & $.4471\pm .0165$ & $.0013\pm .0004$ & $.0029\pm .0002$ & $.606\pm .001$ & $.935\pm .003$ \\
& \cellcolor{gray!10}$\textsc{NC}$ & \cellcolor{gray!10}$.0226\pm .0017$ & \cellcolor{gray!10}$.4528\pm .0108$ & \cellcolor{gray!10}$.0019\pm .0005$ & \cellcolor{gray!10}$.0047\pm .0004$ & \cellcolor{gray!10}$.607\pm .001$ & \cellcolor{gray!10}$.936\pm .003$ \\

\bottomrule
\end{tabular} \end{table*}
\end{center}
\begin{center}
\begin{table*}[t]
\scriptsize
\caption{Results with a \texttt{ConvNext} classifier. $\lambda_{\mathcal A}, \lambda_{\mathcal B},\lambda_{\boldsymbol\sigma}=1e^{-4}$} 
\label{tab:results_convnext_diag} 
\begin{tabular}{c@{\hspace{7pt}} c@{\hspace{7pt}} c@{\hspace{7pt}} c@{\hspace{7pt}} c@{\hspace{7pt}} c@{\hspace{7pt}} c@{\hspace{7pt}} c}
\toprule
Data & Method & $LCE\,\downarrow$ & $MLCE\,\downarrow$ & $ECCE\,\downarrow$ & $ECE\,\downarrow$ & $ACC\,\uparrow$ & $NLL\,\downarrow$ \\
\midrule

\multirow{10}{*}{\rotatebox{90}{\texttt{cifar10}}}
& \cellcolor{blue!10}$\textsc{VQ}$ & \cellcolor{blue!10}$\underline{.0024\pm .0002}$ & \cellcolor{blue!10}$\underline{.3326\pm .0353}$ & \cellcolor{blue!10}$\underline{.0006\pm .0003}$ & \cellcolor{blue!10}$.0021\pm .0004$ & \cellcolor{blue!10}$.958\pm .002$ & \cellcolor{blue!10}$.128\pm .006$ \\
& $\textsc{VQ}_\lambda$ 
& $\mathbf{.0024\pm .0001}$ 
& $\mathbf{.3259\pm .0244}$ 
& $\mathbf{.0006\pm .0001}$ 
& $\mathbf{.0020\pm .0001}$ 
& $.958\pm .001$ 
& $.127\pm .004$ \\
\cline{2-8}
& \cellcolor{gray!10}$\textsc{DC}$ & \cellcolor{gray!10}$.0034\pm .0006$ & \cellcolor{gray!10}$.5810\pm .0829$ & \cellcolor{gray!10}$\mathbf{.0006\pm .0001}$ & \cellcolor{gray!10}$\underline{.0020\pm .0002}$ & \cellcolor{gray!10}$.958\pm .002$ & \cellcolor{gray!10}$.124\pm .003$ \\
& $\textsc{KC}$ & $.0038\pm .0004$ & $.5501\pm .0854$ & $.0008\pm .0001$ & $.0025\pm .0003$ & $\mathbf{.960\pm .001}$ & $\mathbf{.122\pm .005}$ \\
& \cellcolor{gray!10}$\textsc{LN}$ & \cellcolor{gray!10}$.0032\pm .0006$ & \cellcolor{gray!10}$.5615\pm .1347$ & \cellcolor{gray!10}$.0007\pm .0003$ & \cellcolor{gray!10}$.0023\pm .0004$ & \cellcolor{gray!10}$\underline{.960\pm .002}$ & \cellcolor{gray!10}$\underline{.122\pm .002}$ \\
& $\textsc{SM}$ & $.0036\pm .0006$ & $.5791\pm .0867$ & $.0006\pm .0001$ & $.0021\pm .0001$ & $.959\pm .002$ & $.124\pm .003$ \\
& \cellcolor{gray!10}$\textsc{PS}$ & \cellcolor{gray!10}$.0077\pm .0005$ & \cellcolor{gray!10}$.7295\pm .0649$ & \cellcolor{gray!10}$.0013\pm .0001$ & \cellcolor{gray!10}$.0059\pm .0007$ & \cellcolor{gray!10}$.956\pm .002$ & \cellcolor{gray!10}$.187\pm .005$ \\
& $\textsc{IR}$ & $.0040\pm .0007$ & $.5589\pm .0676$ & $.0006\pm .0001$ & $.0024\pm .0002$ & $.957\pm .002$ & $.148\pm .001$ \\
& \cellcolor{gray!10}$\textsc{TS}$ & \cellcolor{gray!10}$.0047\pm .0004$ & \cellcolor{gray!10}$.5893\pm .0908$ & \cellcolor{gray!10}$.0014\pm .0004$ & \cellcolor{gray!10}$.0033\pm .0006$ & \cellcolor{gray!10}$.956\pm .003$ & \cellcolor{gray!10}$.135\pm .007$ \\
& $\textsc{PC}$ & $.0045\pm .0006$ & $.5987\pm .1074$ & $.0014\pm .0003$ & $.0032\pm .0003$ & $.955\pm .002$ & $.139\pm .005$ \\
& \cellcolor{gray!10}$\textsc{NC}$ & \cellcolor{gray!10}$.0043\pm .0003$ & \cellcolor{gray!10}$.6400\pm .1017$ & \cellcolor{gray!10}$.0017\pm .0003$ & \cellcolor{gray!10}$.0039\pm .0007$ & \cellcolor{gray!10}$.956\pm .003$ & \cellcolor{gray!10}$.143\pm .007$ \\

\midrule

\multirow{10}{*}{\rotatebox{90}{\texttt{cifar100}}}
& \cellcolor{blue!10}$\textsc{VQ}$ & \cellcolor{blue!10}$\mathbf{.0012\pm .0001}$ & \cellcolor{blue!10}$\underline{.2778\pm .0078}$ & \cellcolor{blue!10}$.0005\pm .0001$ & \cellcolor{blue!10}$.0014\pm .0001$ & \cellcolor{blue!10}$.799\pm .003$ & \cellcolor{blue!10}$.690\pm .009$ \\
& $\textsc{VQ}_\lambda$ 
& $\mathbf{.0012\pm .0001}$ 
& $\mathbf{.2743\pm .0080}$ 
& $.0005\pm .0001$ 
& $.0013\pm .0001$ 
& $.799\pm .004$ 
& $.687\pm .009$ \\
\cline{2-8}
& \cellcolor{gray!10}$\textsc{DC}$ & \cellcolor{gray!10}$\underline{.0013\pm .0001}$ & \cellcolor{gray!10}$.4576\pm .0056$ & \cellcolor{gray!10}$\underline{.0004\pm .0001}$ & \cellcolor{gray!10}$\underline{.0012\pm .0001}$ & \cellcolor{gray!10}$\underline{.817\pm .002}$ & \cellcolor{gray!10}$\mathbf{.609\pm .006}$ \\
& $\textsc{KC}$ & $.0023\pm .0001$ & $.5495\pm .0085$ & $.0009\pm .0001$ & $.0019\pm .0001$ & $.789\pm .001$ & $.729\pm .004$ \\
& \cellcolor{gray!10}$\textsc{LN}$ & \cellcolor{gray!10}$.0015\pm .0001$ & \cellcolor{gray!10}$.6542\pm .0277$ & \cellcolor{gray!10}$.0005\pm .0001$ & \cellcolor{gray!10}$.0013\pm .0001$ & \cellcolor{gray!10}$\mathbf{.819\pm .001}$ & \cellcolor{gray!10}$\underline{.616\pm .006}$ \\
& $\textsc{SM}$ & $.0013\pm .0001$ & $.4611\pm .0059$ & $\mathbf{.0004\pm .0001}$ & $\mathbf{.0012\pm .0001}$ & $.811\pm .001$ & $.631\pm .004$ \\
& \cellcolor{gray!10}$\textsc{PS}$ & \cellcolor{gray!10}$.0037\pm .0001$ & \cellcolor{gray!10}$.4288\pm .0202$ & \cellcolor{gray!10}$.0012\pm .0001$ & \cellcolor{gray!10}$.0036\pm .0001$ & \cellcolor{gray!10}$.801\pm .001$ & \cellcolor{gray!10}$.989\pm .008$ \\
& $\textsc{IR}$ & $.0014\pm .0001$ & $.4560\pm .0095$ & $.0004\pm .0001$ & $.0014\pm .0001$ & $.805\pm .002$ & $.837\pm .020$ \\
& \cellcolor{gray!10}$\textsc{TS}$ & \cellcolor{gray!10}$.0020\pm .0001$ & \cellcolor{gray!10}$.4876\pm .0135$ & \cellcolor{gray!10}$.0011\pm .0001$ & \cellcolor{gray!10}$.0020\pm .0001$ & \cellcolor{gray!10}$.791\pm .003$ & \cellcolor{gray!10}$.702\pm .012$ \\
& $\textsc{PC}$ & $.0019\pm .0001$ & $.4851\pm .0159$ & $.0011\pm .0001$ & $.0019\pm .0001$ & $.790\pm .003$ & $.696\pm .011$ \\
& \cellcolor{gray!10}$\textsc{NC}$ & \cellcolor{gray!10}$.0018\pm .0001$ & \cellcolor{gray!10}$.5137\pm .0082$ & \cellcolor{gray!10}$.0010\pm .0001$ & \cellcolor{gray!10}$.0020\pm .0001$ & \cellcolor{gray!10}$.791\pm .003$ & \cellcolor{gray!10}$.711\pm .013$ \\

\midrule

\multirow{10}{*}{\rotatebox{90}{\tissue{}}}
& \cellcolor{blue!10}$\textsc{VQ}$ & \cellcolor{blue!10}$\mathbf{.0041\pm .0003}$ & \cellcolor{blue!10}$\mathbf{.1553\pm .0323}$ & \cellcolor{blue!10}$.0011\pm .0002$ & \cellcolor{blue!10}$.0033\pm .0003$ & \cellcolor{blue!10}$.671\pm .004$ & \cellcolor{blue!10}$.903\pm .011$ \\
& $\textsc{VQ}_\lambda$ 
& $.0043\pm .0003$ 
& $.1421\pm .0336$ 
& $.0009\pm .0003$ 
& $.0035\pm .0003$ 
& $.672\pm .004$ 
& $.904\pm .011$ \\
\cline{2-8}
& \cellcolor{gray!10}$\textsc{DC}$ & \cellcolor{gray!10}$\underline{.0060\pm .0004}$ & \cellcolor{gray!10}$.4298\pm .0422$ & \cellcolor{gray!10}$\underline{.0011\pm .0002}$ & \cellcolor{gray!10}$\underline{.0031\pm .0003}$ & \cellcolor{gray!10}$.674\pm .003$ & \cellcolor{gray!10}$\underline{.894\pm .008}$ \\
& $\textsc{KC}$ & $.0200\pm .0012$ & $.2108\pm .0127$ & $.0066\pm .0005$ & $.0191\pm .0013$ & $\mathbf{.678\pm .003}$ & $.913\pm .008$ \\
& \cellcolor{gray!10}$\textsc{LN}$ & \cellcolor{gray!10}$.0171\pm .0003$ & \cellcolor{gray!10}$\underline{.1858\pm .0084}$ & \cellcolor{gray!10}$.0061\pm .0003$ & \cellcolor{gray!10}$.0166\pm .0004$ & \cellcolor{gray!10}$\underline{.677\pm .004}$ & \cellcolor{gray!10}$.907\pm .009$ \\
& $\textsc{SM}$ & $.0060\pm .0003$ & $.4325\pm .0424$ & $\mathbf{.0006\pm .0001}$ & $\mathbf{.0029\pm .0001}$ & $.674\pm .003$ & $\mathbf{.894\pm .008}$ \\
& \cellcolor{gray!10}$\textsc{PS}$ & \cellcolor{gray!10}$.0175\pm .0005$ & \cellcolor{gray!10}$.4316\pm .0305$ & \cellcolor{gray!10}$.0023\pm .0001$ & \cellcolor{gray!10}$.0152\pm .0007$ & \cellcolor{gray!10}$.672\pm .003$ & \cellcolor{gray!10}$.953\pm .006$ \\
& $\textsc{IR}$ & $.0075\pm .0006$ & $.4405\pm .0550$ & $.0011\pm .0002$ & $.0041\pm .0004$ & $.672\pm .003$ & $.903\pm .009$ \\
& \cellcolor{gray!10}$\textsc{TS}$ & \cellcolor{gray!10}$.0127\pm .0025$ & \cellcolor{gray!10}$.4303\pm .0458$ & \cellcolor{gray!10}$.0068\pm .0020$ & \cellcolor{gray!10}$.0111\pm .0029$ & \cellcolor{gray!10}$.668\pm .002$ & \cellcolor{gray!10}$.909\pm .006$ \\
& $\textsc{PC}$ & $.0124\pm .0023$ & $.4331\pm .0497$ & $.0066\pm .0019$ & $.0109\pm .0027$ & $.668\pm .002$ & $.910\pm .005$ \\
& \cellcolor{gray!10}$\textsc{NC}$ & \cellcolor{gray!10}$.0159\pm .0036$ & \cellcolor{gray!10}$.4537\pm .0480$ & \cellcolor{gray!10}$.0084\pm .0033$ & \cellcolor{gray!10}$.0144\pm .0041$ & \cellcolor{gray!10}$.668\pm .002$ & \cellcolor{gray!10}$.910\pm .004$ \\

\bottomrule
\end{tabular} \end{table*}
\end{center}
The factor, \(|\calY|\) denotes the number of classes and appears in the penalty because each
calibration parameter \(\mathbf a_c \in \calA\) and \(\mathbf b_c \in \calB\) is a vector in \(\mathbb R^{|\calY|}\), so the
regularization strength is scaled by the dimensionality of each factor.
The quantity \(n_c\) denotes the empirical usage count of the \(c\)-th codeword in the
calibration set. If each calibration example \(i\) is associated with a sequence
of VQ indices
\(
\mathbf s_i = \bigl(s_i(1),\ldots,s_i(w)\bigr),
\)
then
\(
n_c
=
\sum_{i=1}^{n_{\mathrm{cal}}}
\sum_{j=1}^{w}
\mathbf 1\{s_i(j)=c\}
\)
where $n_{\text{cal}}$ denotes the cardinality of the available calibration set. Thus, \(n_c\) counts how often the \(c\)-th codeword appears across all calibration examples and all positions. Frequent codewords receive a weaker penalty, whereas rare codewords are kept closer to their reference values.
The constant \(\eta>0\) is a small stabilizing term that prevents division by
zero and avoids excessively large weights for codewords with very small support.
We use the simple default \(\eta=1\).
The vectors \(\mathbf a_c^0\) and \(\mathbf b_c^0\) denote the initialisation values. These reference values are chosen so that
the calibrator starts close to its default, identity-preserving behavior.
Therefore, the terms
\[
\|\mathbf a_c-\mathbf a_c^0\|_2^2
\qquad\text{and}\qquad
\|\mathbf b_c-\mathbf b_c^0\|_2^2
\]
penalize deviations of the learned codeword factors from their initial
calibration state.
Overall, the factor
\(
\nicefrac{|\calY|}{(n_c+\eta)}
\)
makes the penalty stronger for rarely observed codewords. The regularizer therefore allows well-supported codewords to adapt more freely, while keeping poorly supported codewords close to their initial values.
Across \texttt{ResNet}, \texttt{FT-Transformer}, \texttt{ConvNeXt}, and \texttt{ViT}, the regularized variant \(\operatorname{VQ}_{\lambda}\) behaves very similarly to the main $\operatorname{VQ}$ formulation, while often yielding small improvements on several calibration metrics. On Weather (\cref{tab:results_weather_diag}), \(\operatorname{VQ}_{\lambda}\) improves LCE, ECCE, ECE, and NLL over VQ, while preserving accuracy, although MLCE slightly worsens. With ResNet and ConvNext backbones (Tables \ref{tab:results_resnet_diag} and \ref{tab:results_convnext_diag}), the two variants remain closely matched, with \(\operatorname{VQ}_{\lambda}\) typically improving or preserving local calibration and occasionally reducing NLL, while retaining the same overall advantage over competing local methods. The same pattern is observed with ViT (\cref{tab:results_vit_diag}), where regularization produces only modest changes relative to VQ. Overall, these results indicate that learning \(\mathrm{diag}(\boldsymbol{\sigma}^2)\) under regularization is a competitive alternative, but does not materially alter the conclusions obtained with the simpler fixed-scaling formulation.
\begin{center}
\begin{table*}[t]
\scriptsize
\caption{Results with a \texttt{ViT} classifier. $\lambda_{\mathcal A}, \lambda_{\mathcal B},\lambda_{\boldsymbol\sigma}=1e^{-4}$} 
\label{tab:results_vit_diag} 
\begin{tabular}{c@{\hspace{7pt}} c@{\hspace{7pt}} c@{\hspace{7pt}} c@{\hspace{7pt}} c@{\hspace{7pt}} c@{\hspace{7pt}} c@{\hspace{7pt}} c}
\toprule
Data & Method & $LCE\,\downarrow$ & $MLCE\,\downarrow$ & $ECCE\,\downarrow$ & $ECE\,\downarrow$ & $ACC\,\uparrow$ & $NLL\,\downarrow$ \\
\midrule

\multirow{11}{*}{\rotatebox{90}{\texttt{cifar10}}}
& \cellcolor{blue!10}$\textsc{VQ}$ & \cellcolor{blue!10}$\mathbf{.0016\pm .0001}$ & \cellcolor{blue!10}$\underline{.2808\pm .0210}$ & \cellcolor{blue!10}$.0004\pm .0001$ & \cellcolor{blue!10}$.0016\pm .0001$ & \cellcolor{blue!10}$.974\pm .001$ & \cellcolor{blue!10}$.086\pm .002$ \\
& $\textsc{VQ}_{\lambda}$ & $\mathbf{.0016\pm .0001}$ & $.2849\pm .0002$ & $.0004\pm .0000$ & $.0016\pm .0000$ & $.974\pm .000$ & $.086\pm .001$ \\
\cline{2-8}
& \cellcolor{gray!10}$\textsc{DC}$ & \cellcolor{gray!10}$\mathbf{.0016\pm .0001}$ & \cellcolor{gray!10}$.4789\pm .0354$ & \cellcolor{gray!10}$\underline{.0003\pm .0001}$ & \cellcolor{gray!10}$\mathbf{.0013\pm .0001}$ & \cellcolor{gray!10}$.977\pm .001$ & \cellcolor{gray!10}$.073\pm .001$ \\
& $\textsc{KC}$ & $.0032\pm .0002$ & $.7915\pm .0214$ & $.0006\pm .0001$ & $.0018\pm .0002$ & $\mathbf{.979\pm .001}$ & $\mathbf{.071\pm .002}$ \\
& \cellcolor{gray!10}$\textsc{LN}$ & \cellcolor{gray!10}$.0025\pm .0003$ & \cellcolor{gray!10}$.9027\pm .0709$ & \cellcolor{gray!10}$.0013\pm .0005$ & \cellcolor{gray!10}$.0032\pm .0011$ & \cellcolor{gray!10}$.977\pm .001$ & \cellcolor{gray!10}$.132\pm .039$ \\
& $\textsc{SM}$ & $\underline{.0017\pm .0001}$ & $.4389\pm .0408$ & $\mathbf{.0003\pm .0001}$ & $\underline{.0014\pm .0001}$ & $\underline{.978\pm .001}$ & $\underline{.071\pm .002}$ \\
& \cellcolor{gray!10}$\textsc{PS}$ & \cellcolor{gray!10}$.0043\pm .0002$ & \cellcolor{gray!10}$\mathbf{.2219\pm .1137}$ & \cellcolor{gray!10}$.0011\pm .0001$ & \cellcolor{gray!10}$.0035\pm .0005$ & \cellcolor{gray!10}$.977\pm .001$ & \cellcolor{gray!10}$.118\pm .007$ \\
& $\textsc{IR}$ & $.0018\pm .0001$ & $.4422\pm .0467$ & $.0004\pm .0001$ & $.0016\pm .0001$ & $.977\pm .001$ & $.087\pm .004$ \\
& \cellcolor{gray!10}$\textsc{TS}$ & \cellcolor{gray!10}$.0020\pm .0001$ & \cellcolor{gray!10}$.4728\pm .0424$ & \cellcolor{gray!10}$.0008\pm .0002$ & \cellcolor{gray!10}$.0019\pm .0003$ & \cellcolor{gray!10}$.977\pm .001$ & \cellcolor{gray!10}$.078\pm .004$ \\
& $\textsc{PC}$ & $.0022\pm .0003$ & $.5402\pm .0483$ & $.0009\pm .0002$ & $.0024\pm .0004$ & $.975\pm .001$ & $.086\pm .006$ \\
& \cellcolor{gray!10}$\textsc{NC}$ & \cellcolor{gray!10}$.0021\pm .0002$ & \cellcolor{gray!10}$.5239\pm .0750$ & \cellcolor{gray!10}$.0012\pm .0003$ & \cellcolor{gray!10}$.0031\pm .0007$ & \cellcolor{gray!10}$.977\pm .001$ & \cellcolor{gray!10}$.106\pm .018$ \\

\midrule

\multirow{11}{*}{\rotatebox{90}{\texttt{cifar100}}}
& \cellcolor{blue!10}$\textsc{VQ}$ & \cellcolor{blue!10}$\underline{.0012\pm .0001}$ & \cellcolor{blue!10}$\underline{.2680\pm .0075}$ & \cellcolor{blue!10}$.0005\pm .0001$ & \cellcolor{blue!10}$.0014\pm .0001$ & \cellcolor{blue!10}$.773\pm .002$ & \cellcolor{blue!10}$.838\pm .003$ \\
& $\textsc{VQ}_{\lambda}$ & $\underline{.0012\pm .0001}$ & $\mathbf{.2636\pm .0001}$ & $.0005\pm .0000$ & $.0014\pm .0000$ & $.774\pm .000$ & $.834\pm .001$ \\
\cline{2-8}
& \cellcolor{gray!10}$\textsc{DC}$ & \cellcolor{gray!10}$.0023\pm .0002$ & \cellcolor{gray!10}$.4558\pm .0088$ & \cellcolor{gray!10}$.0009\pm .0001$ & \cellcolor{gray!10}$.0026\pm .0002$ & \cellcolor{gray!10}$.786\pm .004$ & \cellcolor{gray!10}$1.01\pm .083$ \\
& $\textsc{KC}$ & $.0049\pm .0001$ & $.8920\pm .0029$ & $.0010\pm .0001$ & $.0022\pm .0001$ & $.782\pm .001$ & $.818\pm .005$ \\
& \cellcolor{gray!10}$\textsc{LN}$ & \cellcolor{gray!10}$.0028\pm .0001$ & \cellcolor{gray!10}$.9671\pm .0017$ & \cellcolor{gray!10}$.0006\pm .0001$ & \cellcolor{gray!10}$.0017\pm .0001$ & \cellcolor{gray!10}$\underline{.819\pm .002}$ & \cellcolor{gray!10}$.717\pm .008$ \\
& $\textsc{SM}$ & $\mathbf{.0011\pm .0001}$ & $.3590\pm .0090$ & $\mathbf{.0003\pm .0001}$ & $\mathbf{.0011\pm .0001}$ & $\mathbf{.819\pm .002}$ & $\mathbf{.660\pm .005}$ \\
& \cellcolor{gray!10}$\textsc{PS}$ & \cellcolor{gray!10}$.0030\pm .0001$ & \cellcolor{gray!10}$.3317\pm .0149$ & \cellcolor{gray!10}$.0011\pm .0001$ & \cellcolor{gray!10}$.0031\pm .0001$ & \cellcolor{gray!10}$.818\pm .002$ & \cellcolor{gray!10}$.969\pm .007$ \\
& $\textsc{IR}$ & $\mathbf{.0011\pm .0001}$ & $.3439\pm .0073$ & $\underline{.0004\pm .0001}$ & $.0013\pm .0001$ & $.816\pm .002$ & $.868\pm .011$ \\
& \cellcolor{gray!10}$\textsc{TS}$ & \cellcolor{gray!10}$.0012\pm .0001$ & \cellcolor{gray!10}$.3715\pm .0048$ & \cellcolor{gray!10}$.0005\pm .0001$ & \cellcolor{gray!10}$\underline{.0013\pm .0001}$ & \cellcolor{gray!10}$.816\pm .002$ & \cellcolor{gray!10}$\underline{.677\pm .007}$ \\
& $\textsc{PC}$ & $.0014\pm .0001$ & $.3550\pm .0094$ & $.0005\pm .0001$ & $.0015\pm .0001$ & $.812\pm .002$ & $.700\pm .005$ \\
& \cellcolor{gray!10}$\textsc{NC}$ & \cellcolor{gray!10}$.0017\pm .0001$ & \cellcolor{gray!10}$.4098\pm .0056$ & \cellcolor{gray!10}$.0008\pm .0001$ & \cellcolor{gray!10}$.0020\pm .0001$ & \cellcolor{gray!10}$.816\pm .002$ & \cellcolor{gray!10}$.768\pm .008$ \\

\midrule

\multirow{11}{*}{\rotatebox{90}{\tissue{}}}
& \cellcolor{blue!10}$\textsc{VQ}$ & \cellcolor{blue!10}$\underline{.0042\pm .0006}$ & \cellcolor{blue!10}$\underline{.1064\pm .0101}$ & \cellcolor{blue!10}$.0014\pm .0004$ & \cellcolor{blue!10}$.0037\pm .0005$ & \cellcolor{blue!10}$.662\pm .001$ & \cellcolor{blue!10}$.925\pm .001$ \\
& $\textsc{VQ}_{\lambda}$ & $\mathbf{.0041\pm .0001}$ & $\mathbf{.1044\pm .0002}$ & $\underline{.0011\pm .0001}$ & $\underline{.0035\pm .0001}$ & $.662\pm .001$ & $.925\pm .001$ \\
\cline{2-8}
& \cellcolor{gray!10}$\textsc{DC}$ & \cellcolor{gray!10}$.0062\pm .0003$ & \cellcolor{gray!10}$.5737\pm .0755$ & \cellcolor{gray!10}$.0011\pm .0002$ & \cellcolor{gray!10}$\mathbf{.0034\pm .0002}$ & \cellcolor{gray!10}$.677\pm .002$ & \cellcolor{gray!10}$\mathbf{.886\pm .005}$ \\
& $\textsc{KC}$ & $.0225\pm .0015$ & $.6667\pm .0451$ & $.0066\pm .0005$ & $.0193\pm .0018$ & $\mathbf{.678\pm .002}$ & $.911\pm .002$ \\
& \cellcolor{gray!10}$\textsc{LN}$ & \cellcolor{gray!10}$.0189\pm .0007$ & \cellcolor{gray!10}$.5906\pm .0351$ & \cellcolor{gray!10}$.0059\pm .0005$ & \cellcolor{gray!10}$.0159\pm .0008$ & \cellcolor{gray!10}$.676\pm .002$ & \cellcolor{gray!10}$.906\pm .003$ \\
& $\textsc{SM}$ & $.0066\pm .0005$ & $.5603\pm .0670$ & $\mathbf{.0009\pm .0001}$ & $.0036\pm .0002$ & $\underline{.677\pm .002}$ & $\underline{.885\pm .005}$ \\
& \cellcolor{gray!10}$\textsc{PS}$ & \cellcolor{gray!10}$.0180\pm .0008$ & \cellcolor{gray!10}$.4960\pm .0901$ & \cellcolor{gray!10}$.0025\pm .0002$ & \cellcolor{gray!10}$.0166\pm .0010$ & \cellcolor{gray!10}$.674\pm .002$ & \cellcolor{gray!10}$.951\pm .008$ \\
& $\textsc{IR}$ & $.0077\pm .0004$ & $.5725\pm .0937$ & $.0014\pm .0002$ & $.0042\pm .0002$ & $.675\pm .002$ & $.896\pm .007$ \\
& \cellcolor{gray!10}$\textsc{TS}$ & \cellcolor{gray!10}$.0132\pm .0006$ & \cellcolor{gray!10}$.5615\pm .0898$ & \cellcolor{gray!10}$.0066\pm .0007$ & \cellcolor{gray!10}$.0108\pm .0008$ & \cellcolor{gray!10}$.672\pm .002$ & \cellcolor{gray!10}$.902\pm .005$ \\
& $\textsc{PC}$ & $.0126\pm .0006$ & $.5676\pm .0732$ & $.0062\pm .0006$ & $.0110\pm .0007$ & $.672\pm .001$ & $.905\pm .006$ \\
& \cellcolor{gray!10}$\textsc{NC}$ & \cellcolor{gray!10}$.0169\pm .0033$ & \cellcolor{gray!10}$.6226\pm .0826$ & \cellcolor{gray!10}$.0076\pm .0016$ & \cellcolor{gray!10}$.0164\pm .0044$ & \cellcolor{gray!10}$.672\pm .002$ & \cellcolor{gray!10}$.911\pm .011$ \\

\bottomrule
\end{tabular} \end{table*}
\end{center}

\subsection{Disentangling quantization and head adaptation}
\label{sec:disentangle}
We finetune a classifier head on the calibration set ($\operatorname{FT}$), to clearly isolate the gains of our proposal with respect to simple fine-tuning of the classification head on the calibration set.  We further test the proposed pipeline with a frozen classification head ($\operatorname{VQ}_{\ice}$), where only the quantization module and the corresponding calibration maps can be learned. 
Results (\cref{tab:isolating}) indicate that,  although $\operatorname{FT}$ is effective in terms of $NLL$ and $ACC$, simply refitting the head severely underperforms on both local metrics ($LCE$ and $MLCE$) and global metrics like $ECE$ and $ECCE$ with respect to $\operatorname{VQ}$ and its variant $\operatorname{VQ}_{\ice}$. 
Moreover, $\operatorname{VQ}_{\ice}$ has similar performance in terms of local metrics to $\operatorname{VQ}$, suggesting that local gains do not primarily arise from replacing the classifier head.
Overall, the full pipeline provides a stronger balance between local calibration and global performance than these approaches.
\begin{center}
\begin{table*}[t]
\scriptsize
\caption{Finetune ($\operatorname{FT}$) and frozen classification head ($\operatorname{VQ}_{\ice}$) performance with a \texttt{ResNet}.} 
\label{tab:isolating} 
\renewcommand{\arraystretch}{1.85}
\begin{tabular}{c@{\hspace{7pt}} c@{\hspace{7pt}} c@{\hspace{7pt}} c@{\hspace{7pt}} c@{\hspace{7pt}} c@{\hspace{7pt}} c@{\hspace{7pt}} c}
\toprule
Data & Method & $LCE\,\downarrow$ & $MLCE\,\downarrow$ & $ECCE\,\downarrow$ & $ECE\,\downarrow$ & $ACC\,\uparrow$ & $NLL\,\downarrow$ \\
\midrule

\multirow{3}{*}{\rotatebox{90}{\texttt{cifar10}}}
& \cellcolor{blue!10}$\textsc{VQ}$
& \cellcolor{blue!10}$\mathbf{.0059\pm .0002}$
& \cellcolor{blue!10}$\underline{.5595\pm .0276}$
& \cellcolor{blue!10}$\underline{.0013\pm .0002}$
& \cellcolor{blue!10}$\mathbf{.0037\pm .0002}$
& \cellcolor{blue!10}$\underline{.889\pm .001}$
& \cellcolor{blue!10}$\underline{.348\pm .002}$ \\
& $\textsc{FT}$
& $.0107\pm .0005$
& $.8214\pm .0266$
& $.0022\pm .0004$
& $.0048\pm .0003$
& $\mathbf{.8942\pm .0013}$
& $\mathbf{.3173\pm .0051}$ \\
& $\textsc{VQ}_{\ice}$
& $\underline{.0063\pm .0003}$
& $\mathbf{.5268\pm .0413}$
& $\mathbf{.0012\pm .0002}$
& $\underline{.0038\pm .0004}$
& $.888\pm .0012$
& $.355\pm .006$ \\

\midrule
\multirow{3}{*}{\rotatebox{90}{\texttt{cifar100}}}
& \cellcolor{blue!10}$\textsc{VQ}$
& \cellcolor{blue!10}$\mathbf{.0017\pm .0001}$
& \cellcolor{blue!10}$\mathbf{.2932\pm .0099}$
& \cellcolor{blue!10}$\mathbf{.0008\pm .0001}$
& \cellcolor{blue!10}$\mathbf{.0018\pm .0001}$
& \cellcolor{blue!10}$\underline{.681\pm .001}$
& \cellcolor{blue!10}$\underline{1.208\pm .007}$ \\
& $\textsc{FT}$
& $.0032\pm .0001$
& $.8124\pm .0118$
& $.0013\pm .0001$
& $\underline{.0023\pm .0001}$
& $\mathbf{.6827\pm .0029}$
& $\mathbf{1.173\pm .007}$ \\
& $\textsc{VQ}_{\ice}$
& $\underline{.0023\pm .0001}$
& $\underline{.3488\pm .0138}$
& $\underline{.0012\pm .0001}$
& $\underline{.0023\pm .0001}$
& $.661\pm .0010$
& $1.380\pm .014$ \\

\midrule
\multirow{3}{*}{\rotatebox{90}{\tissue}}
& \cellcolor{blue!10}$\textsc{VQ}$
& \cellcolor{blue!10}$\mathbf{.0083\pm .0002}$
& \cellcolor{blue!10}$\mathbf{.6984\pm .0368}$
& \cellcolor{blue!10}$\mathbf{.0017\pm .0004}$
& \cellcolor{blue!10}$\mathbf{.0043\pm .0002}$
& \cellcolor{blue!10}$\underline{.624\pm .002}$
& \cellcolor{blue!10}$\underline{1.028\pm .007}$ \\
& $\textsc{FT}$
& $.0282\pm .0011$
& $.9597\pm .0158$
& $.0042\pm .0012$
& $.0074\pm .0013$
& $\mathbf{.6382\pm .0014}$
& $\mathbf{.9900\pm .0049}$ \\
& $\textsc{VQ}_{\ice}$
& $\underline{.0097\pm .0007}$
& $\underline{.7143\pm .0382}$
& $\underline{.0022\pm .0007}$
& $\underline{.0048\pm .0007}$
& $.6132\pm .0028$
& $1.0640\pm .0087$ \\

\bottomrule
\end{tabular} \end{table*}
\end{center}

\subsection{Role of the quantization partition}
\label{sec:quantization_role}
Our product-quantization scheme discretizes contiguous feature slots \(\bar{\mathbf z}\), although there is no requirement that features grouped within the same slot be adjacent in the original representation. To assess the sensitivity of the method to this operational choice, we introduce a random permutation \(\pi\) of the feature indices and apply it to the representation before slotting, i.e.,
\[
\mathbf z_\pi = \big(z_{\pi(1)},\ldots,z_{\pi(m)}\big).
\]
We then partition \(\mathbf z_\pi\) into contiguous slots and apply the same VQ calibration pipeline. We denote this variant by \(\operatorname{VQ}_{\pi}\).
\begin{center}
\begin{table*}[t]
\scriptsize
\caption{Performance after random permutations of inputs $\mathbf z$ with a \texttt{ResNet}.} 
\label{tab:column_shuffle} 
\renewcommand{\arraystretch}{2.6}
\begin{tabular}{c@{\hspace{7pt}} c@{\hspace{7pt}} c@{\hspace{7pt}} c@{\hspace{7pt}} c@{\hspace{7pt}} c@{\hspace{7pt}} c@{\hspace{7pt}} c}
\toprule
Data & Method & $LCE\,\downarrow$ & $MLCE\,\downarrow$ & $ECCE\,\downarrow$ & $ECE\,\downarrow$ & $ACC\,\uparrow$ & $NLL\,\downarrow$ \\
\midrule

\multirow{2}{*}{\rotatebox{90}{\texttt{cifar10}}}
& \cellcolor{blue!10}$\textsc{VQ}$
& \cellcolor{blue!10}$\mathbf{.0059\pm .0002}$
& \cellcolor{blue!10}$\underline{.5595\pm .0276}$
& \cellcolor{blue!10}$\underline{.0013\pm .0002}$
& \cellcolor{blue!10}$\mathbf{.0037\pm .0002}$
& \cellcolor{blue!10}$\underline{.889\pm .001}$
& \cellcolor{blue!10}$\underline{.348\pm .002}$ \\
& $\textsc{VQ}_{\pi}$
& $\underline{.0060\pm .0002}$
& $\mathbf{.5584\pm .0305}$
& $\mathbf{.0012\pm .0002}$
& $\underline{.0039\pm .0002}$
& $\mathbf{.890\pm .0011}$
& $\mathbf{.342\pm .007}$ \\

\midrule

\multirow{2}{*}{\rotatebox{90}{\texttt{cifar100}}}
& \cellcolor{blue!10}$\textsc{VQ}$
& \cellcolor{blue!10}$\mathbf{.0017\pm .0001}$
& \cellcolor{blue!10}$\mathbf{.2932\pm .0099}$
& \cellcolor{blue!10}$\mathbf{.0008\pm .0001}$
& \cellcolor{blue!10}$\mathbf{.0018\pm .0001}$
& \cellcolor{blue!10}$\mathbf{.681\pm .001}$
& \cellcolor{blue!10}$\mathbf{1.208\pm .007}$ \\
& $\textsc{VQ}_{\pi}$
& $\mathbf{.0017\pm .0001}$
& $\underline{.2996\pm .0056}$
& $\mathbf{.0008\pm .0001}$
& $\underline{.0018\pm .0001}$
& $\underline{.6774\pm .0020}$
& $\underline{1.2300\pm .0120}$ \\

\midrule

\multirow{2}{*}{\rotatebox{90}{\tissue{}}}
& \cellcolor{blue!10}$\textsc{VQ}$
& \cellcolor{blue!10}$\underline{.0083\pm .0002}$
& \cellcolor{blue!10}$\underline{.6984\pm .0368}$
& \cellcolor{blue!10}$\underline{.0017\pm .0004}$
& \cellcolor{blue!10}$\underline{.0043\pm .0002}$
& \cellcolor{blue!10}$\mathbf{.624\pm .002}$
& \cellcolor{blue!10}$\underline{1.028\pm .007}$ \\
& $\textsc{VQ}_{\pi}$
& $\mathbf{.0079\pm .0003}$
& $\mathbf{.6943\pm .0409}$
& $\mathbf{.0012\pm .0002}$
& $\mathbf{.0036\pm .0004}$
& $\mathbf{.624\pm .002}$
& $\mathbf{1.027\pm .007}$ \\

\bottomrule
\end{tabular} \end{table*}
\end{center}
Results (\cref{tab:column_shuffle}) indicate that the random feature permutation has little effect on performance overall, with \(\operatorname{VQ}_{\pi}\) closely matching the original VQ formulation and even improving several metrics. On \texttt{CIFAR-10}, \(\operatorname{VQ}_{\pi}\) slightly improves $MLCE$, $ECCE$, $ACC$, and $NLL$, while on \texttt{CIFAR-100}, the two variants remain essentially comparable, with only small differences. On \tissue{}, $\operatorname{VQ}_\pi$ improves all calibration metrics and marginally improves $NLL$ at unchanged accuracy. This suggests that the method is not particularly sensitive to grouping contiguous features within slots. More importantly, the occasional improvements obtained after random permutation indicate that the choice of feature grouping itself may be worth optimizing: relevant directions of miscalibration may involve combinations of non-contiguous representation features, and a learned grouping strategy could therefore provide a more effective partition than the fixed contiguous one.

We further investigate the role of quantization in our pipeline by considering partitions that are not learned. To this end, we evaluate two strategies for initializing the codebook and keeping it fixed throughout calibration. In this setting, the codewords serve only to define the discrete indices used to assemble the compositional calibration maps, while also fixing the model scores within each cell.
The first strategy ($\operatorname{VQ}_{\mathcal N}$) initializes codewords by sampling from a standard normal. Since these codewords may lie far from the representation manifold, the resulting partition can be highly lossy and may poorly reflect the geometry of the learned representation space. We therefore consider a second strategy ($\operatorname{VQ}_{\text{data}}$) that initializes the codebook directly from calibration representations. Specifically, each calibration representation is partitioned into \(w\) contiguous slots, yielding \(w\cdot n_{\mathrm{cal}}\) candidate slot vectors, from which \(|\mathcal C|\) are sampled uniformly at random to form the codebook \(\mathcal C\).
\begin{center}
\begin{table*}[t]
\scriptsize
\caption{Performance when the quantization partition is not learned with a \texttt{ResNet}.} 
\label{tab:partition_not_learned} 
\renewcommand{\arraystretch}{1.85}
\begin{tabular}{c@{\hspace{7pt}} c@{\hspace{7pt}} c@{\hspace{7pt}} c@{\hspace{7pt}} c@{\hspace{7pt}} c@{\hspace{7pt}} c@{\hspace{7pt}} c}
\toprule
Data & Method & $LCE\,\downarrow$ & $MLCE\,\downarrow$ & $ECCE\,\downarrow$ & $ECE\,\downarrow$ & $ACC\,\uparrow$ & $NLL\,\downarrow$ \\
\midrule

\multirow{3}{*}{\rotatebox{90}{\texttt{cifar10}}}
& \cellcolor{blue!10}$\textsc{VQ}$
& \cellcolor{blue!10}$\mathbf{.0059\pm .0002}$
& \cellcolor{blue!10}$\underline{.5595\pm .0276}$
& \cellcolor{blue!10}$\underline{.0013\pm .0002}$
& \cellcolor{blue!10}$\mathbf{.0037\pm .0002}$
& \cellcolor{blue!10}$\mathbf{.889\pm .001}$
& \cellcolor{blue!10}$\mathbf{.348\pm .002}$ \\
& $\textsc{VQ}_{\mathcal N}$
& $.0121\pm .0022$
& $.6997\pm .0999$
& $.0014\pm .0005$
& $.0051\pm .0024$
& $.877\pm .012$
& $.395\pm .054$ \\
& $\textsc{VQ}_{\mathrm{data}}$
& $\underline{.0066\pm .0002}$
& $\mathbf{.4863\pm .0512}$
& $\mathbf{.0012\pm .0003}$
& $\underline{.0038\pm .0003}$
& $\underline{.887\pm .001}$
& $\underline{.353\pm .006}$ \\

\midrule

\multirow{3}{*}{\rotatebox{90}{\texttt{cifar100}}}
& \cellcolor{blue!10}$\textsc{VQ}$
& \cellcolor{blue!10}$\mathbf{.0017\pm .0001}$
& \cellcolor{blue!10}$\mathbf{.2932\pm .0099}$
& \cellcolor{blue!10}$\mathbf{.0008\pm .0001}$
& \cellcolor{blue!10}$\mathbf{.0018\pm .0001}$
& \cellcolor{blue!10}$\mathbf{.681\pm .001}$
& \cellcolor{blue!10}$\mathbf{1.208\pm .007}$ \\
& $\textsc{VQ}_{\mathcal N}$
& $.0026\pm .0001$
& $.5659\pm .0334$
& $\underline{.0009\pm .0001}$
& $.0021\pm .0002$
& $.638\pm .007$
& $1.465\pm .039$ \\
& $\textsc{VQ}_{\mathrm{data}}$
& $\underline{.0022\pm .0001}$
& $\underline{.4011\pm .0327}$
& $\mathbf{.0008\pm .0001}$
& $\underline{.0020\pm .0001}$
& $\underline{.6485\pm .0043}$
& $\underline{1.422\pm .040}$ \\

\midrule

\multirow{3}{*}{\rotatebox{90}{\tissue}}
& \cellcolor{blue!10}$\textsc{VQ}$
& \cellcolor{blue!10}$\mathbf{.0083\pm .0002}$
& \cellcolor{blue!10}$\underline{.6984\pm .0368}$
& \cellcolor{blue!10}$\mathbf{.0017\pm .0004}$
& \cellcolor{blue!10}$\mathbf{.0043\pm .0002}$
& \cellcolor{blue!10}$\mathbf{.624\pm .002}$
& \cellcolor{blue!10}$\mathbf{1.028\pm .007}$ \\
& $\textsc{VQ}_{\mathcal N}$
& $.0342\pm .0101$
& $.8889\pm .0925$
& $.0041\pm .0009$
& $.0120\pm .0014$
& $.577\pm .019$
& $1.185\pm .050$ \\
& $\textsc{VQ}_{\mathrm{data}}$
& $\underline{.0093\pm .0006}$
& $\mathbf{.5330\pm .0577}$
& $\underline{.0021\pm .0007}$
& $\underline{.0043\pm .0008}$
& $\underline{.616\pm .003}$
& $\underline{1.056\pm .008}$ \\

\bottomrule
\end{tabular} \end{table*}
\end{center}
The results (\cref{tab:partition_not_learned}) highlight the importance of learning a partition that captures meaningful local structure. When codewords are sampled from a standard normal distribution, the resulting assignments are largely unrelated to the geometry of the representation space, and performance consistently deteriorates. This provides an informative control: simply learning different calibration maps over an arbitrary partition is not sufficient to obtain the gains of $\operatorname{VQ}$. Intuitively, if the induced cells do not correspond to meaningful differences in miscalibration, the optimal cell-specific corrections should become increasingly similar, effectively approaching the behavior of a global calibrator. In contrast, sampling codewords directly from calibration representations substantially improves performance and often approaches that of learned $\operatorname{VQ}$, indicating that even a fixed partition benefits from being aligned with the representation manifold. Learning the codebook further allows the partition to adapt to the data, identifying local structure that is informative for calibration.

\subsection{Sensitivity to the kernel bandwidth}
\label{sec:ablating_gamma}
We report in~\cref{tab:LCE_gamma_values} and~\cref{tab:MLCE_gamma_values} results for $LCE$ and $MLCE$ for different choices of  the kernel bandwidth $\gamma$. Default for all results in the paper is $10$ and the column is highlighted in blue. 
Across datasets and for both $LCE$ and $MLCE$, we observe a consistent trend. For very small bandwidths, neighborhoods contain few effective samples, placing the metrics in a highly local and high-variance regime in which differences between methods are less informative. $\operatorname{VQ}$ performs particularly well across the broad intermediate range of bandwidths that is most relevant in practice. Conversely, for very large bandwidths, neighborhoods become increasingly global, reducing the sensitivity of the metrics to fine-grained local calibration differences and causing performance gaps to narrow or occasionally reverse. Overall, the results indicate that the observed advantages of $\operatorname{VQ}$ are robust across a broad range of practically meaningful bandwidth choices.
\begin{center}
\begin{table*}[t]
\scriptsize
\caption{$LCE$ with different $\gamma$ values with a \texttt{ResNet}.} 
\label{tab:LCE_gamma_values} 
\resizebox{\textwidth}{!}{\renewcommand{\arraystretch}{1.35}
\begin{tabular}{c l c c c c c c}
\toprule
\textbf{Data} &
\textbf{Method} &
$\gamma = \boldsymbol{0.1}$ &
$\gamma = \boldsymbol{1.0}$ &
$\gamma = \boldsymbol{5.0}$ &
$\gamma = \boldsymbol{7.0}$ &
\cellcolor{blue!10}$\gamma = \boldsymbol{10.0}$ &
$\gamma = \boldsymbol{100.0}$ \\
\midrule

\multirow{9}{*}{\rotatebox{90}{\texttt{cifar10}}}
& \cellcolor{gray!10} VQ
& \cellcolor{gray!10}$0.0316 \pm 0.0004$
& \cellcolor{gray!10}$0.0316 \pm 0.0004$
& \cellcolor{gray!10}$\underline{0.0192 \pm 0.0002}$
& \cellcolor{gray!10}$\mathbf{0.0104 \pm 0.0001}$
& \cellcolor{blue!10}$\mathbf{0.0059 \pm 0.0002}$
& \cellcolor{gray!10}$\underline{0.0039 \pm 0.0002}$ \\

& SM
& $0.0316 \pm 0.0005$
& $0.0316 \pm 0.0005$
& $0.0284 \pm 0.0003$
& $0.0203 \pm 0.0004$
& \cellcolor{blue!10}$0.0102 \pm 0.0003$
& $\mathbf{0.0037 \pm 0.0001}$ \\

& \cellcolor{gray!10} DC
& \cellcolor{gray!10}$\underline{0.0305 \pm 0.0005}$
& \cellcolor{gray!10}$\underline{0.0305 \pm 0.0005}$
& \cellcolor{gray!10}$0.0274 \pm 0.0004$
& \cellcolor{gray!10}$0.0196 \pm 0.0003$
& \cellcolor{blue!10}$0.0099 \pm 0.0003$
& \cellcolor{gray!10}$\mathbf{0.0037 \pm 0.0001}$ \\

& KC
& $0.0339 \pm 0.0002$
& $0.0339 \pm 0.0002$
& $0.0313 \pm 0.0004$
& $0.0236 \pm 0.0005$
& \cellcolor{blue!10}$0.0095 \pm 0.0004$
& $0.0057 \pm 0.0001$ \\

& \cellcolor{gray!10} LN
& \cellcolor{gray!10}$0.0345 \pm 0.0003$
& \cellcolor{gray!10}$0.0344 \pm 0.0003$
& \cellcolor{gray!10}$\mathbf{0.0181 \pm 0.0006}$
& \cellcolor{gray!10}$\underline{0.0117 \pm 0.0004}$
& \cellcolor{blue!10}$\underline{0.0079 \pm 0.0003}$
& \cellcolor{gray!10}$0.0050 \pm 0.0003$ \\

& PS
& $0.0387 \pm 0.0004$
& $0.0387 \pm 0.0004$
& $0.0342 \pm 0.0003$
& $0.0253 \pm 0.0006$
& \cellcolor{blue!10}$0.0180 \pm 0.0005$
& $0.0085 \pm 0.0002$ \\

& \cellcolor{gray!10} IR
& \cellcolor{gray!10}$0.0317 \pm 0.0004$
& \cellcolor{gray!10}$0.0317 \pm 0.0004$
& \cellcolor{gray!10}$0.0286 \pm 0.0002$
& \cellcolor{gray!10}$0.0210 \pm 0.0004$
& \cellcolor{blue!10}$0.0115 \pm 0.0003$
& \cellcolor{gray!10}$0.0047 \pm 0.0001$ \\

& TS
& $0.0317 \pm 0.0023$
& $0.0317 \pm 0.0023$
& $0.0289 \pm 0.0015$
& $0.0219 \pm 0.0004$
& \cellcolor{blue!10}$0.0127 \pm 0.0006$
& $0.0075 \pm 0.0010$ \\

& \cellcolor{gray!10} PC
& \cellcolor{gray!10}$\mathbf{0.0298 \pm 0.0004}$
& \cellcolor{gray!10}$\mathbf{0.0298 \pm 0.0004}$
& \cellcolor{gray!10}$0.0273 \pm 0.0003$
& \cellcolor{gray!10}$0.0213 \pm 0.0003$
& \cellcolor{blue!10}$0.0133 \pm 0.0005$
& \cellcolor{gray!10}$0.0084 \pm 0.0004$ \\

\midrule

\multirow{9}{*}{\rotatebox{90}{\texttt{cifar100}}}
& \cellcolor{gray!10} VQ
& \cellcolor{gray!10}$0.0062 \pm 0.0002$
& \cellcolor{gray!10}$0.0062 \pm 0.0002$
& \cellcolor{gray!10}$\mathbf{0.0034 \pm 0.0001}$
& \cellcolor{gray!10}$\mathbf{0.0022 \pm 0.0001}$
& \cellcolor{blue!10}$\mathbf{0.0017 \pm 0.0001}$
& \cellcolor{gray!10}$\underline{0.0012 \pm 0.0001}$ \\

& SM
& $0.0061 \pm 0.0001$
& $0.0061 \pm 0.0001$
& $0.0059 \pm 0.0001$
& $0.0052 \pm 0.0001$
& \cellcolor{blue!10}$0.0031 \pm 0.0002$
& $\mathbf{0.0010 \pm 0.0000}$ \\

& \cellcolor{gray!10} DC
& \cellcolor{gray!10}$\mathbf{0.0056 \pm 0.0001}$
& \cellcolor{gray!10}$\mathbf{0.0056 \pm 0.0001}$
& \cellcolor{gray!10}$\underline{0.0055 \pm 0.0001}$
& \cellcolor{gray!10}$0.0048 \pm 0.0001$
& \cellcolor{blue!10}$0.0027 \pm 0.0001$
& \cellcolor{gray!10}$0.0014 \pm 0.0001$ \\

& KC
& $0.0085 \pm 0.0001$
& $0.0085 \pm 0.0001$
& $0.0083 \pm 0.0001$
& $0.0072 \pm 0.0001$
& \cellcolor{blue!10}$0.0043 \pm 0.0001$
& $0.0030 \pm 0.0001$ \\

& \cellcolor{gray!10} LN
& \cellcolor{gray!10}$0.0066 \pm 0.0000$
& \cellcolor{gray!10}$0.0066 \pm 0.0000$
& \cellcolor{gray!10}$0.0062 \pm 0.0000$
& \cellcolor{gray!10}$\underline{0.0046 \pm 0.0000}$
& \cellcolor{blue!10}$\underline{0.0024 \pm 0.0001}$
& \cellcolor{gray!10}$\underline{0.0012 \pm 0.0000}$ \\

& PS
& $0.0100 \pm 0.0001$
& $0.0100 \pm 0.0001$
& $0.0096 \pm 0.0001$
& $0.0084 \pm 0.0002$
& \cellcolor{blue!10}$0.0053 \pm 0.0001$
& $0.0028 \pm 0.0000$ \\

& \cellcolor{gray!10} IR
& \cellcolor{gray!10}$\underline{0.0060 \pm 0.0000}$
& \cellcolor{gray!10}$\underline{0.0060 \pm 0.0000}$
& \cellcolor{gray!10}$0.0058 \pm 0.0000$
& \cellcolor{gray!10}$0.0051 \pm 0.0001$
& \cellcolor{blue!10}$0.0031 \pm 0.0001$
& \cellcolor{gray!10}$\underline{0.0012 \pm 0.0000}$ \\

& TS
& $0.0063 \pm 0.0001$
& $0.0063 \pm 0.0001$
& $0.0062 \pm 0.0001$
& $0.0055 \pm 0.0001$
& \cellcolor{blue!10}$0.0034 \pm 0.0001$
& $0.0016 \pm 0.0001$ \\

& \cellcolor{gray!10} PC
& \cellcolor{gray!10}$0.0068 \pm 0.0001$
& \cellcolor{gray!10}$0.0068 \pm 0.0001$
& \cellcolor{gray!10}$0.0066 \pm 0.0001$
& \cellcolor{gray!10}$0.0059 \pm 0.0001$
& \cellcolor{blue!10}$0.0041 \pm 0.0002$
& \cellcolor{gray!10}$0.0021 \pm 0.0001$ \\

\midrule

\multirow{9}{*}{\rotatebox{90}{\tissue{}}}
& \cellcolor{gray!10} VQ
& \cellcolor{gray!10}$\mathbf{0.1242 \pm 0.0002}$
& \cellcolor{gray!10}$\mathbf{0.1242 \pm 0.0002}$
& \cellcolor{gray!10}$\mathbf{0.0488 \pm 0.0004}$
& \cellcolor{gray!10}$\mathbf{0.0183 \pm 0.0003}$
& \cellcolor{blue!10}$\mathbf{0.0083 \pm 0.0002}$
& \cellcolor{gray!10}$\underline{0.0044 \pm 0.0004}$ \\

& SM
& $0.1278 \pm 0.0008$
& $0.1278 \pm 0.0008$
& $0.1236 \pm 0.0005$
& $0.0847 \pm 0.0019$
& \cellcolor{blue!10}$0.0278 \pm 0.0013$
& $\mathbf{0.0037 \pm 0.0002}$ \\

& \cellcolor{gray!10} DC
& \cellcolor{gray!10}$0.1282 \pm 0.0008$
& \cellcolor{gray!10}$0.1282 \pm 0.0008$
& \cellcolor{gray!10}$0.1240 \pm 0.0004$
& \cellcolor{gray!10}$0.0851 \pm 0.0019$
& \cellcolor{blue!10}$0.0322 \pm 0.0015$
& \cellcolor{gray!10}$0.0045 \pm 0.0003$ \\

& KC
& $0.1315 \pm 0.0006$
& $0.1315 \pm 0.0006$
& $0.0766 \pm 0.0030$
& $0.0313 \pm 0.0008$
& \cellcolor{blue!10}$0.0226 \pm 0.0006$
& $0.0195 \pm 0.0004$ \\

& \cellcolor{gray!10} LN
& \cellcolor{gray!10}$\underline{0.1263 \pm 0.0008}$
& \cellcolor{gray!10}$\underline{0.1263 \pm 0.0008}$
& \cellcolor{gray!10}$\underline{0.0597 \pm 0.0009}$
& \cellcolor{gray!10}$\underline{0.0241 \pm 0.0009}$
& \cellcolor{blue!10}$\underline{0.0208 \pm 0.0007}$
& \cellcolor{gray!10}$0.0200 \pm 0.0015$ \\

& PS
& $0.1381 \pm 0.0007$
& $0.1381 \pm 0.0007$
& $0.1322 \pm 0.0005$
& $0.0889 \pm 0.0018$
& \cellcolor{blue!10}$0.0418 \pm 0.0018$
& $0.0133 \pm 0.0012$ \\

& \cellcolor{gray!10} IR
& \cellcolor{gray!10}$0.1286 \pm 0.0016$
& \cellcolor{gray!10}$0.1286 \pm 0.0016$
& \cellcolor{gray!10}$0.1244 \pm 0.0013$
& \cellcolor{gray!10}$0.0871 \pm 0.0022$
& \cellcolor{blue!10}$0.0360 \pm 0.0024$
& \cellcolor{gray!10}$0.0096 \pm 0.0005$ \\

& TS
& $0.1307 \pm 0.0023$
& $0.1307 \pm 0.0023$
& $0.1267 \pm 0.0020$
& $0.0907 \pm 0.0022$
& \cellcolor{blue!10}$0.0506 \pm 0.0026$
& $0.0407 \pm 0.0036$ \\

& \cellcolor{gray!10} PC
& \cellcolor{gray!10}$0.1287 \pm 0.0022$
& \cellcolor{gray!10}$0.1287 \pm 0.0022$
& \cellcolor{gray!10}$0.1256 \pm 0.0019$
& \cellcolor{gray!10}$0.0985 \pm 0.0020$
& \cellcolor{blue!10}$0.0543 \pm 0.0029$
& \cellcolor{gray!10}$0.0390 \pm 0.0023$ \\

\bottomrule
\end{tabular}}
 \end{table*}
\end{center}
\begin{center}
\begin{table*}[t]
\scriptsize
\caption{$MLCE$ with different $\gamma$ values with a \texttt{ResNet}.} 
\label{tab:MLCE_gamma_values} 
\resizebox{\textwidth}{!}{\renewcommand{\arraystretch}{1.35}
\begin{tabular}{c l c c c c c c}
\toprule
\textbf{Data} &
\textbf{Method} &
$\gamma = \boldsymbol{0.1}$ &
$\gamma = \boldsymbol{1.0}$ &
$\gamma = \boldsymbol{5.0}$ &
$\gamma = \boldsymbol{7.0}$ &
\cellcolor{blue!10}$\gamma = \boldsymbol{10.0}$ &
$\gamma = \boldsymbol{100.0}$ \\
\midrule

\multirow{9}{*}{\rotatebox{90}{\texttt{cifar10}}}
& \cellcolor{gray!10} VQ
& \cellcolor{gray!10}$0.9996 \pm 0.0001$
& \cellcolor{gray!10}$0.9996 \pm 0.0001$
& \cellcolor{gray!10}$\mathbf{0.9215 \pm 0.0153}$
& \cellcolor{gray!10}$\mathbf{0.7948 \pm 0.0178}$
& \cellcolor{blue!10}$\mathbf{0.5595 \pm 0.0276}$
& \cellcolor{gray!10}$0.1434 \pm 0.0067$ \\

& SM
& $0.9990 \pm 0.0002$
& $0.9990 \pm 0.0002$
& $0.9886 \pm 0.0042$
& $0.9336 \pm 0.0164$
& \cellcolor{blue!10}$0.8027 \pm 0.0166$
& $\mathbf{0.1381 \pm 0.0138}$ \\

& \cellcolor{gray!10} DC
& \cellcolor{gray!10}$0.9995 \pm 0.0002$
& \cellcolor{gray!10}$0.9995 \pm 0.0002$
& \cellcolor{gray!10}$0.9917 \pm 0.0032$
& \cellcolor{gray!10}$0.9438 \pm 0.0136$
& \cellcolor{blue!10}$0.8127 \pm 0.0194$
& \cellcolor{gray!10}$0.1469 \pm 0.0075$ \\

& KC
& $0.9993 \pm 0.0002$
& $0.9993 \pm 0.0002$
& $0.9929 \pm 0.0025$
& $0.9470 \pm 0.0104$
& \cellcolor{blue!10}$0.7923 \pm 0.0243$
& $\underline{0.1405 \pm 0.0094}$ \\

& \cellcolor{gray!10} LN
& \cellcolor{gray!10}$0.9997 \pm 0.0001$
& \cellcolor{gray!10}$0.9997 \pm 0.0001$
& \cellcolor{gray!10}$\underline{0.9280 \pm 0.0129}$
& \cellcolor{gray!10}$\underline{0.8253 \pm 0.0179}$
& \cellcolor{blue!10}$\underline{0.6279 \pm 0.0157}$
& \cellcolor{gray!10}$0.1457 \pm 0.0138$ \\

& PS
& $\mathbf{0.9970 \pm 0.0002}$
& $\mathbf{0.9970 \pm 0.0002}$
& $0.9933 \pm 0.0008$
& $0.9620 \pm 0.0122$
& \cellcolor{blue!10}$0.8542 \pm 0.0224$
& $0.3557 \pm 0.0201$ \\

& \cellcolor{gray!10} IR
& \cellcolor{gray!10}$0.9996 \pm 0.0002$
& \cellcolor{gray!10}$0.9996 \pm 0.0002$
& \cellcolor{gray!10}$0.9920 \pm 0.0021$
& \cellcolor{gray!10}$0.9406 \pm 0.0163$
& \cellcolor{blue!10}$0.7978 \pm 0.0162$
& \cellcolor{gray!10}$0.1609 \pm 0.0185$ \\

& TS
& $\underline{0.9986 \pm 0.0011}$
& $\underline{0.9986 \pm 0.0011}$
& $0.9878 \pm 0.0040$
& $0.9387 \pm 0.0144$
& \cellcolor{blue!10}$0.8200 \pm 0.0191$
& $0.1914 \pm 0.0227$ \\

& \cellcolor{gray!10} PC
& \cellcolor{gray!10}$0.9999 \pm 0.0001$
& \cellcolor{gray!10}$0.9999 \pm 0.0001$
& \cellcolor{gray!10}$0.9926 \pm 0.0028$
& \cellcolor{gray!10}$0.9439 \pm 0.0171$
& \cellcolor{blue!10}$0.8339 \pm 0.0456$
& \cellcolor{gray!10}$0.1852 \pm 0.0286$ \\

\midrule

\multirow{9}{*}{\rotatebox{90}{\texttt{cifar100}}}
& \cellcolor{gray!10} VQ
& \cellcolor{gray!10}$0.9994 \pm 0.0001$
& \cellcolor{gray!10}$0.9994 \pm 0.0001$
& \cellcolor{gray!10}$\mathbf{0.8465 \pm 0.0063}$
& \cellcolor{gray!10}$\mathbf{0.5972 \pm 0.0104}$
& \cellcolor{blue!10}$\mathbf{0.2932 \pm 0.0099}$
& \cellcolor{gray!10}$\underline{0.1265 \pm 0.0027}$ \\

& SM
& $0.9989 \pm 0.0001$
& $0.9989 \pm 0.0001$
& $0.9964 \pm 0.0003$
& $0.9749 \pm 0.0042$
& \cellcolor{blue!10}$0.8121 \pm 0.0191$
& $0.1277 \pm 0.0072$ \\

& \cellcolor{gray!10} DC
& \cellcolor{gray!10}$0.9999 \pm 0.0000$
& \cellcolor{gray!10}$0.9999 \pm 0.0000$
& \cellcolor{gray!10}$0.9991 \pm 0.0003$
& \cellcolor{gray!10}$0.9831 \pm 0.0027$
& \cellcolor{blue!10}$0.8113 \pm 0.0121$
& \cellcolor{gray!10}$\mathbf{0.1252 \pm 0.0054}$ \\

& KC
& $\mathbf{0.9938 \pm 0.0009}$
& $\mathbf{0.9938 \pm 0.0009}$
& $\underline{0.9899 \pm 0.0006}$
& $0.9525 \pm 0.0026$
& \cellcolor{blue!10}$0.7575 \pm 0.0101$
& $0.1482 \pm 0.0062$ \\

& \cellcolor{gray!10} LN
& \cellcolor{gray!10}$0.9983 \pm 0.0001$
& \cellcolor{gray!10}$0.9983 \pm 0.0001$
& \cellcolor{gray!10}$0.9905 \pm 0.0005$
& \cellcolor{gray!10}$\underline{0.9159 \pm 0.0028}$
& \cellcolor{blue!10}$\underline{0.7022 \pm 0.0070}$
& \cellcolor{gray!10}$0.1319 \pm 0.0068$ \\

& PS
& $\underline{0.9976 \pm 0.0001}$
& $\underline{0.9976 \pm 0.0001}$
& $0.9969 \pm 0.0001$
& $0.9827 \pm 0.0027$
& \cellcolor{blue!10}$0.7031 \pm 0.0409$
& $0.2090 \pm 0.0029$ \\

& \cellcolor{gray!10} IR
& \cellcolor{gray!10}$0.9997 \pm 0.0001$
& \cellcolor{gray!10}$0.9997 \pm 0.0001$
& \cellcolor{gray!10}$0.9981 \pm 0.0004$
& \cellcolor{gray!10}$0.9783 \pm 0.0042$
& \cellcolor{blue!10}$0.8049 \pm 0.0116$
& \cellcolor{gray!10}$0.1357 \pm 0.0068$ \\

& TS
& $0.9990 \pm 0.0001$
& $0.9990 \pm 0.0001$
& $0.9966 \pm 0.0004$
& $0.9764 \pm 0.0037$
& \cellcolor{blue!10}$0.8239 \pm 0.0173$
& $0.1801 \pm 0.0045$ \\

& \cellcolor{gray!10} PC
& \cellcolor{gray!10}$0.9994 \pm 0.0001$
& \cellcolor{gray!10}$0.9994 \pm 0.0001$
& \cellcolor{gray!10}$0.9978 \pm 0.0005$
& \cellcolor{gray!10}$0.9778 \pm 0.0041$
& \cellcolor{blue!10}$0.8187 \pm 0.0127$
& \cellcolor{gray!10}$0.1905 \pm 0.0064$ \\

\midrule

\multirow{9}{*}{\rotatebox{90}{\tissue{}}}
& \cellcolor{gray!10} VQ
& \cellcolor{gray!10}$0.9991 \pm 0.0001$
& \cellcolor{gray!10}$0.9991 \pm 0.0001$
& \cellcolor{gray!10}$\mathbf{0.9727 \pm 0.0042}$
& \cellcolor{gray!10}$\mathbf{0.8825 \pm 0.0076}$
& \cellcolor{blue!10}$\underline{0.6984 \pm 0.0368}$
& \cellcolor{gray!10}$\mathbf{0.0533 \pm 0.0064}$ \\

& SM
& $0.9993 \pm 0.0002$
& $0.9993 \pm 0.0002$
& $0.9989 \pm 0.0004$
& $0.9953 \pm 0.0014$
& \cellcolor{blue!10}$0.9609 \pm 0.0106$
& $\underline{0.0710 \pm 0.0098}$ \\

& \cellcolor{gray!10} DC
& \cellcolor{gray!10}$0.9992 \pm 0.0003$
& \cellcolor{gray!10}$0.9992 \pm 0.0003$
& \cellcolor{gray!10}$0.9988 \pm 0.0004$
& \cellcolor{gray!10}$0.9951 \pm 0.0014$
& \cellcolor{blue!10}$0.9706 \pm 0.0098$
& \cellcolor{gray!10}$0.0826 \pm 0.0085$ \\

& KC
& $0.9967 \pm 0.0005$
& $0.9967 \pm 0.0005$
& $0.9842 \pm 0.0034$
& $0.9327 \pm 0.0125$
& \cellcolor{blue!10}$0.7296 \pm 0.0570$
& $0.1814 \pm 0.0072$ \\

& \cellcolor{gray!10} LN
& \cellcolor{gray!10}$0.9992 \pm 0.0001$
& \cellcolor{gray!10}$0.9992 \pm 0.0001$
& \cellcolor{gray!10}$\underline{0.9812 \pm 0.0021}$
& \cellcolor{gray!10}$\underline{0.9141 \pm 0.0102}$
& \cellcolor{blue!10}$\mathbf{0.6932 \pm 0.0661}$
& \cellcolor{gray!10}$0.1008 \pm 0.0061$ \\

& PS
& $\mathbf{0.9890 \pm 0.0012}$
& $\mathbf{0.9890 \pm 0.0012}$
& $0.9882 \pm 0.0010$
& $0.9803 \pm 0.0015$
& \cellcolor{blue!10}$0.9518 \pm 0.0045$
& $0.1246 \pm 0.0094$ \\

& \cellcolor{gray!10} IR
& \cellcolor{gray!10}$\underline{0.9931 \pm 0.0005}$
& \cellcolor{gray!10}$\underline{0.9931 \pm 0.0005}$
& \cellcolor{gray!10}$0.9929 \pm 0.0004$
& \cellcolor{gray!10}$0.9909 \pm 0.0010$
& \cellcolor{blue!10}$0.9602 \pm 0.0103$
& \cellcolor{gray!10}$0.1472 \pm 0.0115$ \\

& TS
& $0.9993 \pm 0.0001$
& $0.9993 \pm 0.0001$
& $0.9990 \pm 0.0001$
& $0.9959 \pm 0.0011$
& \cellcolor{blue!10}$0.9822 \pm 0.0084$
& $0.1608 \pm 0.0255$ \\

& \cellcolor{gray!10} PC
& \cellcolor{gray!10}$1.0000 \pm 0.0001$ 
& \cellcolor{gray!10}$1.0000 \pm 0.0001$
& \cellcolor{gray!10}$1.0000 \pm 0.0001$
& \cellcolor{gray!10}$0.9996 \pm 0.0001$
& \cellcolor{blue!10}$0.9768 \pm 0.0133$
& \cellcolor{gray!10}$0.2452 \pm 0.0177$ \\

\bottomrule
\end{tabular}}
 \end{table*}
\end{center}

\subsection{Training Time}
\label{sec:time}
We report the average training time for a single epoch for \VQ{} with a \texttt{ResNet} backbone on all three datasets.
For both \cifarTen{} and \cifar{}, the \textit{quantization} head takes on average $\approx1$ second to complete an epoch. The \textit{calibrator} also takes $\approx1$ second to complete a single epoch. For \tissue{}, the \textit{quantization} head takes on average $\approx3$ seconds to complete an epoch. The \textit{calibrator} takes $\approx2.5$ seconds to complete a single epoch.

\section{Implementation Details}
\label{sec:experimental_details}

In this section we provide a comprehensive description of all the procedures underlying our empirical results. We followed the same setting as in \citet{DBLP:journals/corr/Barbera25}. 
\subsection{Training of Classifiers}
Here we describe hyperparameters and the training procedure for the baseline classifiers.

For \texttt{CIFAR-10} and \tissue{} we leverage a \texttt{ResNet-50} architecture with \texttt{IMAGENET1K\_V2} pre-trained weights.
We additionally add a dropout layer to the backbone with $0.2$ rate and a linear classification head.

For \texttt{CIFAR-100} we instead rely on a \texttt{ResNet-152} model, still with \texttt{IMAGENET1K\_V2} pre-trained weights.
We again add a dropout layer to the backbone, but with $0.5$ rate in this case.

All \texttt{ResNet} architectures are trained with early stopping on validation $NLL$, \texttt{Adam} optimizer and a learning rate of $3\!\times\!10^{-4}$.

For all datasets we employ a \texttt{ConvNext-tiny} architecture. We again add a dropout layer to the backbone with $0.2$ rate. 

For all datasets we also employ a \texttt{vit-base-patch16-224} from the timm library \citep{wightman2019pytorch}, still with \texttt{IMAGENET1K\_V2} pre-trained weights.

Both \texttt{ConvNext} and \texttt{ViT} architectures are trained with early stopping on validation $NLL$, \texttt{Adam} optimizer and a learning rate of $5\!\times\!10^{-5}$. 

Finally, for \texttt{Weather} we leverage an \texttt{FTT} classifier. The model is trained with early stopping on validation $NLL$, \texttt{Adam} optimizer and a learning rate of $1\!\times\!10^{-4}$. 

All classifiers are trained with Categorical Cross-Entropy. 

\subsection{Local calibration maps}

\textbf{Quantized classifier}. For all of our experiments we segment latent representations into $w=64$ slots, yielding $d=32$ per slot for \texttt{ResNet} architectures and $d=12$ for \texttt{ConvNext}, and \texttt{ViT} architectures and we instantiate the codebook with $|\mathcal{C}|=64$ prototype vectors. The codebook $\mathcal C$ is initialized with random samples drawn from the segmented representations of the classifier. During learning, updates are performed via EMA with a decay rate of $0.99$ as in \citep{van2017neural}. The new classification head $f_{\mathrm{VQ}}: \mathcal Q \to \Delta^{|\mcY|}$ is a linear layer followed by softmax that maps quantized representations to the $|\mcY|-1$ probability simplex and is trained with \texttt{Adam} and $1\times 10^{-3}$ learning rate and weight decay. 

\textbf{Calibration function.}
The \textit{local} calibrator we propose requires two sets of parameter vectors \(\mathcal{A}=\{\mba_1,\dots,\mba_{|\mathcal C|}\}\subset\mathbb{R}^{|\mcY|},
\mathcal{B}=\{\mathbf{b}_1,\dots,\mathbf b_{|\mathcal C|}\}\subset\mathbb{R}^{|\mcY|}\) and a third parameter vector $\boldsymbol\sigma^2 \in \mathbb{R}^w$. $\mathcal{A}$ and $\mathcal{B}$ are initialised such that for any induced submatrices $\mathbf A$, $\mathbf B$ (as in \cref{sec:local_dir_cal}): \[
\phi(\mathbf A^\top \mathbf B) = \mathbf 1\mathbf 1^\top
\] where $\phi(\cdot)$ is the softplus and $\mathbf 1\in \mathbb R^{|\calY|}$ is the all ones vector. The calibration parameters then are:
\[
\boldsymbol\alpha^{(\mathcal V)}-\mathbf 1 \mathbf 1^\top =\Bigl(\phi\!\left(\mathbf A^\top \mathrm{diag}(\boldsymbol\sigma^2) \mathbf B\right) + \mathbf I \Bigr) -\mathbf 1 \mathbf 1^\top
\] 

This modeling choice initializes each region's calibration map close to the identity, providing a strong starting point for optimization and improving convergence. Lastly, our main results are presented fixing $\boldsymbol{\sigma}^2 = \mathbf 1$ (the all-ones vector) and do not learn it. We refrain from using the full calibrator model, as it introduces additional capacity that empirically increases the risk of overfitting. This is a well-known issue in post-hoc calibration, where simpler models and stronger regularization are often preferred \cite{DBLP:journals/corr/abs-2511-03685}. We present results with a regularized variant of our approach that learns $\boldsymbol{\sigma}^2$ in \cref{sec:reg_sigma}. Training leverages \texttt{Adam}~\citep{DBLP:journals/corr/KingmaB14} with learning rate and weight decay set to $1\times 10^{-3}$. \subsection{Training of Local Methods}
In what follows we illustrate the technical details regarding implementation of local calibration methods in our experiments.

For \cifarTen{} we use a fully connected network with a single hidden layer of size $64$ and a dropout rate of $0.3$.
We operate in a PCA-reduced feature representation space of size $50$. We use \texttt{Adam} for training with learning rate $1\!\times\!10^{-3}$, for $22$ epochs (early stopping) and a batch size of $1024$.
For \cifar \text{ } we leverage the same architecture and hyperparameters except for the hidden layer, which now has size $128$, and the dropout rate, which is set to $0.5$. Training lasts $30$ epochs with early stopping.
The same applies to \tissue{}, but we set the hidden dimension to $256$ and the dropout rate back to $0.3$. Learning instead lasts $60$ epochs with early stopping. \subsection{Metrics}
We follow the implementation of \citep{DBLP:journals/corr/Barbera25}.
More precisely, for both \textit{global} and \textit{local} metrics we partitioned $f_k(\mbx)$ into $15$ bins based on predicted confidence scores. Moreover, we picked a value of $10$ for the kernel-bandwidth hyperparameter consistently with the learning procedure of the local methods. 

\subsection{Calibration Algorithm}
\label{sec:algo}

\begin{algorithm}[t]
\caption{Calibration algorithm}
\label{alg:appendix_calibration}
\KwIn{
EMA decay $\gamma$; learning rates $\eta_{\mathrm{vq}}, \eta_{\mathrm{cal}}$.
}
\KwOut{Learned codebook $\mathcal C$ and calibrated predictor $f_{\mathrm{cal}}$.}

\BlankLine
\textbf{Stage 1: Quantization-aware representation learning}\\
Initialize codebook $\mathcal C$ and quantization-aware head $f_{\mathrm{VQ}}$\;
\Repeat{convergence}{
  Sample minibatch $\mathcal D_{\mathrm{mini}} \subset \mathcal D_{\mathrm{cal}}$\;

  $\mathbf z \leftarrow E(x)$ for $(x,y)\in\mathcal D_{\mathrm{mini}}$\;
  $(z^{(1)},\dots,z^{(w)}) \leftarrow \Phi(\mathbf z)$\;

  $\mathbf s \leftarrow \mathrm{Assign}((z^{(i)})_{i=1}^w,\mathcal C)$\;
  $\overline{\mbq} \leftarrow \mathrm{Select}(\mathbf s,\mathcal C)$\;

  $\hat{\mbp}_{\mathrm{VQ}} \leftarrow f_{\mathrm{VQ}}(\overline{\mbq})$\;

  $\mathcal{L}_{\mathrm{ce}} \leftarrow \mathrm{CE}(\hat{\mbp}_{\mathrm{VQ}}, y)$\;

  $\theta_{\mathrm{VQ}} \leftarrow \theta_{\mathrm{VQ}} - \eta_{\mathrm{vq}}\nabla_{\theta_{\mathrm{VQ}}}\mathcal{L}_{\mathrm{ce}}$\;

  $\mathcal C \leftarrow \mathrm{EMAUpdate}(\mathcal C, (z^{(i)})_{i=1}^w, \mathbf s;\gamma)$\;
}

\BlankLine
\textbf{Stage 2: Region-aware calibration}\\
Freeze $\mathcal C$ and $f_{\mathrm{VQ}}$. Initialize calibration parameters $(\mathcal A,\mathcal B,\boldsymbol\sigma^2)$ and
$f_{\mathrm{cal}}(\hat{\mbp}_{\mathrm{VQ}},\mathbf s)$\; 

\Repeat{convergence}{
  Sample minibatch $\mathcal D_{\mathrm{mini}} \subset \mathcal D_{\mathrm{cal}}$\;

   $\mathbf A^{(\mcV)}$, $\mathbf B^{(\mcV)}$ $\leftarrow$ $\mathrm{Select}(\mathbf s, \mathcal A)$, $\mathrm{Select}(\mathbf s, \mathcal B)$ for $(\hat{\mbp}_{\mathrm{VQ}}, \mathbf s)\in\mathcal D_{\mathrm{mini}}$\;

   $\boldsymbol\alpha^{(\mcV)} \leftarrow \phi\big(\mathbf A^{(\mcV)\top}\mathrm{diag}(\boldsymbol\sigma^2)\mathbf B^{(\mcV)}\big)+\mathbf I$\;
    
  $\hat{\mbp}_{\mathrm{cal}} \leftarrow f_{\mathrm{cal}}(\hat{\mbp}_{\mathrm{VQ}})$\;

  $\mathcal{L}_{\mathrm{cal}} \leftarrow \mathrm{CE}(\hat{\mbp}_{\mathrm{cal}}, y)$\;
  
\If{$\boldsymbol{\sigma}^2$ is learnable}{
  $(\mathcal A,\mathcal B,\boldsymbol\sigma^2) \leftarrow
  (\mathcal A,\mathcal B,\boldsymbol\sigma^2)
  - \eta_{\mathrm{cal}}
  \nabla_{(\mathcal A,\mathcal B,\boldsymbol\sigma^2)}
  \mathcal{L}_{\mathrm{cal}}$
}
\Else{
  $(\mathcal A,\mathcal B) \leftarrow
  (\mathcal A,\mathcal B)
  - \eta_{\mathrm{cal}}
  \nabla_{(\mathcal A,\mathcal B)}
  \mathcal{L}_{\mathrm{cal}}$
}
}

\Return{$\mathcal C,\ f_{\mathrm{cal}}$}.
\end{algorithm}

In this section we provide a comprehensive description of the algorithm underlying our proposal. The method can be summarized in two stages: (i) a discretization step (\textit{divide}) and (ii) a \textit{calibration} (\textit{calibra}) step.
\paragraph{Learning procedure}

\textbf{Quantization-Aware Representation Learning.} The first step discretizes the continuous latent representation of the baseline classifier. Learning proceeds as follows:
\begin{itemize}
    \item \textbf{line 2-3:} A latent encoding $\mbz$ is extracted from the frozen encoder. $\mbz$ is then segmented into $w$ equal-sized slots $\mbz^{(i)}$. 
    \item \textbf{line 4-5:} Each slot $z^{(i)}$ is discretized by nearest neighbor assignment with respect to the vectors in $\mathcal{C}$. $\mathrm{Assign}(\cdot,\cdot)$ yields the corresponding index sequence $\mathbf s$ over all slots. The sequence is used to extract the discrete representation $\mathbf{\overline{q}}$ from $\mathcal C$ via $\mathrm{Select}(\cdot, \cdot)$. This, for the $i$-th slot, takes the value from $\mathcal C$ at index $s^{(i)}$. These values are then concatenated to form $\mathbf{\overline{q}}$.
    \item \textbf{lines 6-9:} A new classification head takes as input $\mathbf{\overline{q}}$ and produces new scores. These values are used to compute the categorical cross entropy between predicted probabilities and ground truth. The weights of $f_{\mathrm{VQ}}$ are updated via gradient descent, whereas the codewords are updated via exponential moving average. \end{itemize}
Training proceeds until convergence.

\textbf{Region-Aware Calibration.} This step takes as input the quantization head scores and the index sequence $\mathbf s$. Its goal is to produce new \textit{locally} calibrated probabilities.
Given scores $\mathbf p_{\mathrm{VQ}}$ and indices $\mathbf s$, the calibration head does the following:
\begin{itemize}
    \item \textbf{line 1:} Produce the \textit{receiver} basis $\mathbf A^{(\mcV)}$ from $\mathcal A$ and the \textit{sender} basis $\mathbf B^{(\mcV)}$ from $\mathcal B$ via $\mathrm{Select}(\mathbf s, \mathcal A)$ and $\mathrm{Select}(\mathbf s, \mathcal B)$; 
    \item \textbf{line 2-3:} Compute $\boldsymbol\alpha^{(\mcV)}$ and produce new calibrated scores (\cref{eqn:log-linear-form}). \item \textbf{line 4-8:} Updated calibration parameters (potentially including $\text{diag}(\boldsymbol{\sigma}^2)$) via gradient descent on the categorical cross entropy until convergence.
\end{itemize}

The output is a codebook \(\mathcal C\) defining the tessellation of the latent space and a calibrated predictor \(f_{\mathrm{cal}}\).

\begin{figure*}[t]
        \includegraphics[width=\linewidth]{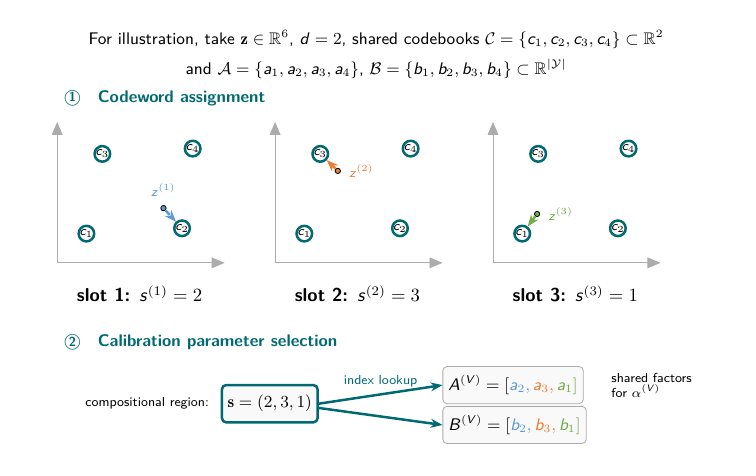}
\caption{Illustration of the indexed parametrization trick. An instance $\mathbf z \in \mathbb R^6$ is first segmented in $2$-dimensional slots $\mathbf z^{(1)},\mathbf  z^{(2)},\mathbf  z^{(3)}$. Each slot is assigned to the closest codeword $\mathbf c_j\in \calC$. The resulting index sequence $\mathbf s$, identifies a unique Voronoi cell in $\mathbb R^6$ and is used to select parameters from $\calA$ and $\calB$. The matrices $A^{(\calV)}$, $B^{(\calV)}$ yield Dirichlet concentration parameters via \cref{eq:bilinear_alpha_quant}.}
    \label{fig:pract_examp}
\end{figure*}

\subsection{Hardware}
\label{sec:hardware}

To perform the experiments, we used two machines: $(i)$ a 16-core machine with an AMD Ryzen 9 7950X CPU and 2 NVIDIA GeForce RTX 4090 GDDR6X with 24GB of memory, OS Ubuntu 22.04.4 LTS; $(ii)$ a small cluster with 1 single master node and 2 nodes with 8 x GPU A100 80GB, OS Rocky Linux 9.6 

\newpage

 \section*{NeurIPS Paper Checklist}

\begin{enumerate}

\item {\bf Claims}
    \item[] Question: Do the main claims made in the abstract and introduction accurately reflect the paper's contributions and scope?
    \item[] Answer: \answerYes{}
    \item[] Justification:
The paper's main claims are well reflected by our theoretical (see \cref{sec:methodology}) and empirical experiments (\cref{sec:experimental-evaluation}).
    \item[] Guidelines:
    \begin{itemize}
        \item The answer \answerNA{} means that the abstract and introduction do not include the claims made in the paper.
        \item The abstract and/or introduction should clearly state the claims made, including the contributions made in the paper and important assumptions and limitations. A \answerNo{} or \answerNA{} answer to this question will not be perceived well by the reviewers. 
        \item The claims made should match theoretical and experimental results, and reflect how much the results can be expected to generalize to other settings. 
        \item It is fine to include aspirational goals as motivation as long as it is clear that these goals are not attained by the paper. 
    \end{itemize}

\item {\bf Limitations}
    \item[] Question: Does the paper discuss the limitations of the work performed by the authors?
    \item[] Answer: \answerYes{} \item[] Justification: Please refer to \cref{sec:limitations} for an extended description of the paper limitations.
    \item[] Guidelines:
    \begin{itemize}
        \item The answer \answerNA{} means that the paper has no limitation while the answer \answerNo{} means that the paper has limitations, but those are not discussed in the paper. 
        \item The authors are encouraged to create a separate ``Limitations'' section in their paper.
        \item The paper should point out any strong assumptions and how robust the results are to violations of these assumptions (e.g., independence assumptions, noiseless settings, model well-specification, asymptotic approximations only holding locally). The authors should reflect on how these assumptions might be violated in practice and what the implications would be.
        \item The authors should reflect on the scope of the claims made, e.g., if the approach was only tested on a few datasets or with a few runs. In general, empirical results often depend on implicit assumptions, which should be articulated.
        \item The authors should reflect on the factors that influence the performance of the approach. For example, a facial recognition algorithm may perform poorly when image resolution is low or images are taken in low lighting. Or a speech-to-text system might not be used reliably to provide closed captions for online lectures because it fails to handle technical jargon.
        \item The authors should discuss the computational efficiency of the proposed algorithms and how they scale with dataset size.
        \item If applicable, the authors should discuss possible limitations of their approach to address problems of privacy and fairness.
        \item While the authors might fear that complete honesty about limitations might be used by reviewers as grounds for rejection, a worse outcome might be that reviewers discover limitations that aren't acknowledged in the paper. The authors should use their best judgment and recognize that individual actions in favor of transparency play an important role in developing norms that preserve the integrity of the community. Reviewers will be specifically instructed to not penalize honesty concerning limitations.
    \end{itemize}

\item {\bf Theory assumptions and proofs}
    \item[] Question: For each theoretical result, does the paper provide the full set of assumptions and a complete (and correct) proof?
    \item[] Answer: \answerYes{} \item[] Justification: All the proofs for our theory in \cref{sec:methodology} are provided in \cref{app:proofs}.
    \item[] Guidelines:
    \begin{itemize}
        \item The answer \answerNA{} means that the paper does not include theoretical results. 
        \item All the theorems, formulas, and proofs in the paper should be numbered and cross-referenced.
        \item All assumptions should be clearly stated or referenced in the statement of any theorems.
        \item The proofs can either appear in the main paper or the supplemental material, but if they appear in the supplemental material, the authors are encouraged to provide a short proof sketch to provide intuition. 
        \item Inversely, any informal proof provided in the core of the paper should be complemented by formal proofs provided in appendix or supplemental material.
        \item Theorems and Lemmas that the proof relies upon should be properly referenced. 
    \end{itemize}

    \item {\bf Experimental result reproducibility}
    \item[] Question: Does the paper fully disclose all the information needed to reproduce the main experimental results of the paper to the extent that it affects the main claims and/or conclusions of the paper (regardless of whether the code and data are provided or not)?
    \item[] Answer: \answerYes{} \item[] Justification: We detail thoroughly all implementation details for base-models, methods and metrics in \cref{sec:experimental_details}. 
    \item[] Guidelines:
    \begin{itemize}
        \item The answer \answerNA{} means that the paper does not include experiments.
        \item If the paper includes experiments, a \answerNo{} answer to this question will not be perceived well by the reviewers: Making the paper reproducible is important, regardless of whether the code and data are provided or not.
        \item If the contribution is a dataset and\slash or model, the authors should describe the steps taken to make their results reproducible or verifiable. 
        \item Depending on the contribution, reproducibility can be accomplished in various ways. For example, if the contribution is a novel architecture, describing the architecture fully might suffice, or if the contribution is a specific model and empirical evaluation, it may be necessary to either make it possible for others to replicate the model with the same dataset, or provide access to the model. In general. releasing code and data is often one good way to accomplish this, but reproducibility can also be provided via detailed instructions for how to replicate the results, access to a hosted model (e.g., in the case of a large language model), releasing of a model checkpoint, or other means that are appropriate to the research performed.
        \item While NeurIPS does not require releasing code, the conference does require all submissions to provide some reasonable avenue for reproducibility, which may depend on the nature of the contribution. For example
        \begin{enumerate}
            \item If the contribution is primarily a new algorithm, the paper should make it clear how to reproduce that algorithm.
            \item If the contribution is primarily a new model architecture, the paper should describe the architecture clearly and fully.
            \item If the contribution is a new model (e.g., a large language model), then there should either be a way to access this model for reproducing the results or a way to reproduce the model (e.g., with an open-source dataset or instructions for how to construct the dataset).
            \item We recognize that reproducibility may be tricky in some cases, in which case authors are welcome to describe the particular way they provide for reproducibility. In the case of closed-source models, it may be that access to the model is limited in some way (e.g., to registered users), but it should be possible for other researchers to have some path to reproducing or verifying the results.
        \end{enumerate}
    \end{itemize}

\item {\bf Open access to data and code}
    \item[] Question: Does the paper provide open access to the data and code, with sufficient instructions to faithfully reproduce the main experimental results, as described in supplemental material?
    \item[] Answer: \answerYes{} \item[] Justification: 
The code is attached as a zip file and will be released upon acceptance. 
    \item[] Guidelines:
    \begin{itemize}
        \item The answer \answerNA{} means that paper does not include experiments requiring code.
        \item Please see the NeurIPS code and data submission guidelines (\url{https://neurips.cc/public/guides/CodeSubmissionPolicy}) for more details.
        \item While we encourage the release of code and data, we understand that this might not be possible, so \answerNo{} is an acceptable answer. Papers cannot be rejected simply for not including code, unless this is central to the contribution (e.g., for a new open-source benchmark).
        \item The instructions should contain the exact command and environment needed to run to reproduce the results. See the NeurIPS code and data submission guidelines (\url{https://neurips.cc/public/guides/CodeSubmissionPolicy}) for more details.
        \item The authors should provide instructions on data access and preparation, including how to access the raw data, preprocessed data, intermediate data, and generated data, etc.
        \item The authors should provide scripts to reproduce all experimental results for the new proposed method and baselines. If only a subset of experiments are reproducible, they should state which ones are omitted from the script and why.
        \item At submission time, to preserve anonymity, the authors should release anonymized versions (if applicable).
        \item Providing as much information as possible in supplemental material (appended to the paper) is recommended, but including URLs to data and code is permitted.
    \end{itemize}

\item {\bf Experimental setting/details}
    \item[] Question: Does the paper specify all the training and test details (e.g., data splits, hyperparameters, how they were chosen, type of optimizer) necessary to understand the results?
    \item[] Answer: \answerYes{} \item[] Justification: The details can be found in \cref{sec:experimental_details}. 
    \item[] Guidelines:
    \begin{itemize}
        \item The answer \answerNA{} means that the paper does not include experiments.
        \item The experimental setting should be presented in the core of the paper to a level of detail that is necessary to appreciate the results and make sense of them.
        \item The full details can be provided either with the code, in appendix, or as supplemental material.
    \end{itemize}

\item {\bf Experiment statistical significance}
    \item[] Question: Does the paper report error bars suitably and correctly defined or other appropriate information about the statistical significance of the experiments?
    \item[] Answer: \answerYes{} \item[] Justification: We provide all standard deviations in Tables and report either boxplots or confidence intervals at 95\% in all the Figures.
    \item[] Guidelines:
    \begin{itemize}
        \item The answer \answerNA{} means that the paper does not include experiments.
        \item The authors should answer \answerYes{} if the results are accompanied by error bars, confidence intervals, or statistical significance tests, at least for the experiments that support the main claims of the paper.
        \item The factors of variability that the error bars are capturing should be clearly stated (for example, train/test split, initialization, random drawing of some parameter, or overall run with given experimental conditions).
        \item The method for calculating the error bars should be explained (closed form formula, call to a library function, bootstrap, etc.)
        \item The assumptions made should be given (e.g., Normally distributed errors).
        \item It should be clear whether the error bar is the standard deviation or the standard error of the mean.
        \item It is OK to report 1-sigma error bars, but one should state it. The authors should preferably report a 2-sigma error bar than state that they have a 96\% CI, if the hypothesis of Normality of errors is not verified.
        \item For asymmetric distributions, the authors should be careful not to show in tables or figures symmetric error bars that would yield results that are out of range (e.g., negative error rates).
        \item If error bars are reported in tables or plots, the authors should explain in the text how they were calculated and reference the corresponding figures or tables in the text.
    \end{itemize}

\item {\bf Experiments compute resources}
    \item[] Question: For each experiment, does the paper provide sufficient information on the computer resources (type of compute workers, memory, time of execution) needed to reproduce the experiments?
    \item[] Answer: \answerYes{} \item[] Justification: We provide hardware details in \cref{sec:hardware}.
    \item[] Guidelines:
    \begin{itemize}
        \item The answer \answerNA{} means that the paper does not include experiments.
        \item The paper should indicate the type of compute workers CPU or GPU, internal cluster, or cloud provider, including relevant memory and storage.
        \item The paper should provide the amount of compute required for each of the individual experimental runs as well as estimate the total compute. 
        \item The paper should disclose whether the full research project required more compute than the experiments reported in the paper (e.g., preliminary or failed experiments that didn't make it into the paper). 
    \end{itemize}
    
\item {\bf Code of ethics}
    \item[] Question: Does the research conducted in the paper conform, in every respect, with the NeurIPS Code of Ethics \url{https://neurips.cc/public/EthicsGuidelines}?
    \item[] Answer: \answerYes{} \item[] Justification: We follow the NeurIPS Code of Ethics.
    \item[] Guidelines:
    \begin{itemize}
        \item The answer \answerNA{} means that the authors have not reviewed the NeurIPS Code of Ethics.
        \item If the authors answer \answerNo, they should explain the special circumstances that require a deviation from the Code of Ethics.
        \item The authors should make sure to preserve anonymity (e.g., if there is a special consideration due to laws or regulations in their jurisdiction).
    \end{itemize}

\item {\bf Broader impacts}
    \item[] Question: Does the paper discuss both potential positive societal impacts and negative societal impacts of the work performed?
    \item[] Answer: \answerYes{} \item[] Justification: We provide a dedicated Broader Impact section.
    \item[] Guidelines:
    \begin{itemize}
        \item The answer \answerNA{} means that there is no societal impact of the work performed.
        \item If the authors answer \answerNA{} or \answerNo, they should explain why their work has no societal impact or why the paper does not address societal impact.
        \item Examples of negative societal impacts include potential malicious or unintended uses (e.g., disinformation, generating fake profiles, surveillance), fairness considerations (e.g., deployment of technologies that could make decisions that unfairly impact specific groups), privacy considerations, and security considerations.
        \item The conference expects that many papers will be foundational research and not tied to particular applications, let alone deployments. However, if there is a direct path to any negative applications, the authors should point it out. For example, it is legitimate to point out that an improvement in the quality of generative models could be used to generate Deepfakes for disinformation. On the other hand, it is not needed to point out that a generic algorithm for optimizing neural networks could enable people to train models that generate Deepfakes faster.
        \item The authors should consider possible harms that could arise when the technology is being used as intended and functioning correctly, harms that could arise when the technology is being used as intended but gives incorrect results, and harms following from (intentional or unintentional) misuse of the technology.
        \item If there are negative societal impacts, the authors could also discuss possible mitigation strategies (e.g., gated release of models, providing defenses in addition to attacks, mechanisms for monitoring misuse, mechanisms to monitor how a system learns from feedback over time, improving the efficiency and accessibility of ML).
    \end{itemize}
    
\item {\bf Safeguards}
    \item[] Question: Does the paper describe safeguards that have been put in place for responsible release of data or models that have a high risk for misuse (e.g., pre-trained language models, image generators, or scraped datasets)?
    \item[] Answer: 
\answerNA{}.
    \item[] Justification: Our proposal does not feature models that present risk of misuse. 
    \item[] Guidelines:
    \begin{itemize}
        \item The answer \answerNA{} means that the paper poses no such risks.
        \item Released models that have a high risk for misuse or dual-use should be released with necessary safeguards to allow for controlled use of the model, for example by requiring that users adhere to usage guidelines or restrictions to access the model or implementing safety filters. 
        \item Datasets that have been scraped from the Internet could pose safety risks. The authors should describe how they avoided releasing unsafe images.
        \item We recognize that providing effective safeguards is challenging, and many papers do not require this, but we encourage authors to take this into account and make a best faith effort.
    \end{itemize}

\item {\bf Licenses for existing assets}
    \item[] Question: Are the creators or original owners of assets (e.g., code, data, models), used in the paper, properly credited and are the license and terms of use explicitly mentioned and properly respected?
    \item[] Answer: \answerYes{} \item[] Justification: We cite all the relevant information.
    \item[] Guidelines:
    \begin{itemize}
        \item The answer \answerNA{} means that the paper does not use existing assets.
        \item The authors should cite the original paper that produced the code package or dataset.
        \item The authors should state which version of the asset is used and, if possible, include a URL.
        \item The name of the license (e.g., CC-BY 4.0) should be included for each asset.
        \item For scraped data from a particular source (e.g., website), the copyright and terms of service of that source should be provided.
        \item If assets are released, the license, copyright information, and terms of use in the package should be provided. For popular datasets, \url{paperswithcode.com/datasets} has curated licenses for some datasets. Their licensing guide can help determine the license of a dataset.
        \item For existing datasets that are re-packaged, both the original license and the license of the derived asset (if it has changed) should be provided.
        \item If this information is not available online, the authors are encouraged to reach out to the asset's creators.
    \end{itemize}

\item {\bf New assets}
    \item[] Question: Are new assets introduced in the paper well documented and is the documentation provided alongside the assets?
    \item[] Answer: \answerNA{} \item[] Justification: We do not provide new assets.
    \item[] Guidelines:
    \begin{itemize}
        \item The answer \answerNA{} means that the paper does not release new assets.
        \item Researchers should communicate the details of the dataset\slash code\slash model as part of their submissions via structured templates. This includes details about training, license, limitations, etc. 
        \item The paper should discuss whether and how consent was obtained from people whose asset is used.
        \item At submission time, remember to anonymize your assets (if applicable). You can either create an anonymized URL or include an anonymized zip file.
    \end{itemize}

\item {\bf Crowdsourcing and research with human subjects}
    \item[] Question: For crowdsourcing experiments and research with human subjects, does the paper include the full text of instructions given to participants and screenshots, if applicable, as well as details about compensation (if any)? 
    \item[] Answer: \answerNA{} \item[] Justification: We do not involve humans.
    \item[] Guidelines:
    \begin{itemize}
        \item The answer \answerNA{} means that the paper does not involve crowdsourcing nor research with human subjects.
        \item Including this information in the supplemental material is fine, but if the main contribution of the paper involves human subjects, then as much detail as possible should be included in the main paper. 
        \item According to the NeurIPS Code of Ethics, workers involved in data collection, curation, or other labor should be paid at least the minimum wage in the country of the data collector. 
    \end{itemize}

\item {\bf Institutional review board (IRB) approvals or equivalent for research with human subjects}
    \item[] Question: Does the paper describe potential risks incurred by study participants, whether such risks were disclosed to the subjects, and whether Institutional Review Board (IRB) approvals (or an equivalent approval/review based on the requirements of your country or institution) were obtained?
    \item[] Answer: \answerNA{} \item[] Justification: We do not involve humans.
    \item[] Guidelines:
    \begin{itemize}
        \item The answer \answerNA{} means that the paper does not involve crowdsourcing nor research with human subjects.
        \item Depending on the country in which research is conducted, IRB approval (or equivalent) may be required for any human subjects research. If you obtained IRB approval, you should clearly state this in the paper. 
        \item We recognize that the procedures for this may vary significantly between institutions and locations, and we expect authors to adhere to the NeurIPS Code of Ethics and the guidelines for their institution. 
        \item For initial submissions, do not include any information that would break anonymity (if applicable), such as the institution conducting the review.
    \end{itemize}

\item {\bf Declaration of LLM usage}
    \item[] Question: Does the paper describe the usage of LLMs if it is an important, original, or non-standard component of the core methods in this research? Note that if the LLM is used only for writing, editing, or formatting purposes and does \emph{not} impact the core methodology, scientific rigor, or originality of the research, declaration is not required.
\item[] Answer: \answerNA{} \item[] Justification: We used LLMs only for polishing the text.
    \item[] Guidelines:
    \begin{itemize}
        \item The answer \answerNA{} means that the core method development in this research does not involve LLMs as any important, original, or non-standard components.
        \item Please refer to our LLM policy in the NeurIPS handbook for what should or should not be described.
    \end{itemize}

\end{enumerate} \end{document}